\documentclass{article}

\usepackage[main, final]{neurips_2025}

\usepackage{authblk}

\usepackage{include/custom}
\usepackage{include/paralist}

\usepackage[utf8]{inputenc} % allow utf-8 input
\usepackage[T1]{fontenc}    % use 8-bit T1 fonts
\usepackage{hyperref}       % hyperlinks
\usepackage{url}            % simple URL typesetting
\usepackage{booktabs}       % professional-quality tables
\usepackage{amsfonts}       % blackboard math symbols
\usepackage{amsthm}
\usepackage{nicefrac}       % compact symbols for 1/2, etc.
\usepackage{microtype}      % microtypography
\usepackage{xcolor}         % colors

\usepackage{multirow}
\usepackage{booktabs}

\usepackage[table]{xcolor}
\usepackage{makecell}
\usepackage{graphicx}
\usepackage{subfig}

\usepackage{tikz}
\usepackage{tikz-3dplot}
\usepackage{xcolor}
\usetikzlibrary{3d,calc}
\colorlet{c1}{blue}
\colorlet{c2}{green!50!black}
\colorlet{c3}{red!70!magenta}
\colorlet{c12}{cyan}
\colorlet{c23}{orange}
\colorlet{c13}{violet}
\colorlet{tri}{brown}
\colorlet{best}{red!90!black}
\colorlet{second_best}{blue!90!black}
\pgfplotsset{compat=1.18}  % or any recent version like 1.17, 1.16, etc.
\usepgfplotslibrary{fillbetween}
\usepackage{wrapfig}
\usepackage{caption}
\usetikzlibrary{arrows.meta, decorations.markings, calc, fadings, positioning, decorations.pathreplacing}

\usepackage{xcolor}

\newtheorem*{remark}{Remark}
\newtheorem{exmp}{Example}[section]

\title{Composing Linear Layers from Irreducibles}

\author{%
  Travis Pence $\;\;$ Daisuke Yamada $\;\;$ Vikas Singh\\
  University of Wisconsin-Madison\\
  \texttt{\{tnpence, dyamada2\}@wisc.edu}, \texttt{vsingh@biostat.wisc.edu} \\
}
\begin{document}
\maketitle

\begin{abstract}
\label{sec:abstract}
Contemporary large models often exhibit behaviors suggesting the presence of low-level primitives that compose into modules with richer functionality, but these fundamental building blocks remain poorly understood. 
We investigate this compositional structure in linear layers by asking: \textit{can we identify/synthesize linear transformations from a minimal set of geometric primitives?} 
Using Clifford algebra, we show that linear layers can be expressed as compositions of bivectors---geometric objects encoding oriented planes---and introduce a differentiable algorithm that decomposes them into products of rotors. 
This construction uses only $\mathcal{O}\left(\log^2 d\right)$ parameters, versus $\mathcal{O}(d^2)$ required by dense matrices. Applied to the key, query, and value projections in LLM attention layers, rotor-based layers match the performance of strong baselines such as block-Hadamard and low-rank approximations. Our findings provide an algebraic perspective on how these geometric primitives can compose into higher-level functions within deep models.
\end{abstract}
\section{Introduction}
\label{sec:introduction}

There is growing consensus \citep{doi:10.1073/pnas.2219150120} that, like biological systems, modern models may internally rely on \textit{low-level primitives} that \textit{compose} to form modules with more complex functionality \citep{weiss2021thinkingliketransformers}. This compositional perspective was one motivation behind capsule networks \citep{sabour2017dynamicroutingcapsules}, which explicitly modeled part-whole relationships through vector-based capsules. But the task of localizing and characterizing such primitives remains challenging albeit interesting.  
%Achieving this goal will offer deeper insights into the internal mechanisms of models widely used in scientific and industrial domains. But there are also practical motivations, ranging from guardrails to interpretability. 
A general capability to compose pre-trained modules in a prescribed way could, in principle, lead to a larger model with predictable and more controllable functionality \citep{zou2025representationengineeringtopdownapproach, schug2024discoveringmodularsolutionsgeneralize, pmlr-v97-ghazi19a, abnar2023adaptivitymodularityefficientgeneralization, press2023measuringnarrowingcompositionalitygap}, with potential applications 
ranging from safety guardrails to interpretability. 

Some recent results suggest some progress in this direction. In mixture-of-experts architectures \citep{masoudnia2014mixture, riquelme2021scaling}, specialized sub-networks are conditionally activated and composed via routing \citep{buchel2025efficient}. Model merging has evolved from simple parameter averaging to more sophisticated alignment of different model latent representations \citep{lahner2024direct}. Fine-tuning methods like LoRA \citep{hu2021loralowrankadaptationlarge} implicitly assume that low-dimensional adjustment of a base network should suffice---essentially a two-level composition. Each approach offers a distinct perspective on composition (e.g., see \cite{chytas2024understanding})
%many of these techniques operate at the level of high-level transformations rather than 
but are not focused on addressing how the mechanistic composition of low-level primitives gives more complex behavior. To this end, mechanistic interpretability \citep{rai2024practical} and neurosymbolic methods \citep{yang2022safeneurosymboliclearningdifferentiable} explore this space, but are still in a nascent stage of development. 

\paragraph{Scope of this paper.} 
%Despite these perspectives on composition, there is still a gap in our understanding of whether the main modules of contemporary models themselves can be composed from an even more elementary set of operations/objects. We can ask an even more basic question. 
Consider a core module---with millions of parameters---in a large model. \textit{Can we synthesize its functionality from its most basic primitives? How many such objects would we need?}  This casts our broader interest in composition into a concrete problem:  \textit{identifying a minimal set of irreducibles that combine in specific ways to realize the full functionality of the module.} We study this problem for \textit{linear layers}---noting that linear layers make up a large portion of parameters in large language models (LLMs). While prior work suggests various parameter-efficient approximations, our interest is not only function approximation, rather to build up the functionality by characterizing its \textit{algebraic structure}. 
%We hypothesize that the true degrees of freedom is an \textit{order of magnitude smaller} than the raw parameter count, if we can understand the core generator elements and their compositional symmetries.
%
To formalize the above intuition, we use the language of \textit{Clifford algebra}, where linear transformations naturally decompose into simple \textit{bivectors}---geometric objects representing oriented planes. This view reveals how the functionality of a linear layer can be synthesized as a structured composition of a few hundred geometric objects (parameters). 

Our key {\bf contributions} are:
\begin{inparaenum}[\bfseries (a)]
    \item We express linear transformations as compositions of geometric primitives---specifically, bivectors in Clifford algebra---using rotor sandwich products acting on local subspaces of an input multivector. This requires $\mathcal{O}\left(\log^2 d\right)$ scalar parameters, compared to $\mathcal{O}\left(d^2\right)$ for dense layers, where $d$ is the input/output dimension.
    \item We propose a differentiable invariant decomposition algorithm that maps bivectors to their corresponding rotors in closed-form, which enables integration with {\tt autograd} and gradient-based optimization.
    %efficient gradient-based methods (with specialized libraries). 
    %\item Our construction requires only $\mathcal{O}(\log^2 d)$ scalar parameters, whereas dense layers need $\mathcal{O}(d^2)$, where $d$ is the dimension of input/output.
    \item Empirically, we replace the key, query, and value projections in LLM attention layers and show comparable downstream performance in accuracy and perplexity across various datasets. 
\end{inparaenum}

Our goal is to show the \textit{feasibility} of this algebraic decomposition approach. It is not a drop-in replacement yet, since practical benefits will require additional system-level integration beyond the scope of this work. The focus here is on the underlying algorithmic and mathematical foundations.

\section{Preliminaries}
\label{sec:prelim}
\vspace{-2pt}
We review some relevant concepts from Clifford algebra (also see \citet{hestenes2012clifford}).
\vspace{-3pt}
\paragraph{Clifford algebra.} A \textit{Clifford algebra}, $\Cl_{p,q}(\rr)$, is an associative algebra over $\rr^n$ equipped with a quadratic form of signature $(p,q)$, where $n=p+q$. The algebra admits two basic products: the \textit{inner product} and \textit{outer (wedge) product}, denoted by $\cdot$ and $\wedge$ respectively. Their sum defines the \textit{geometric product}, which for vectors $u, v \in \rr^n$ takes the form:
\[
    uv \triangleq u\cdot v + u\wedge v.
\]
The algebra is generated by orthogonal basis vectors $e_1,\ldots,e_n$, called \textit{generators}, satisfying:
\[
    e_i^2 = +1 \quad \text{for } 1 \le i \le p; \quad
    e_j^2 = -1 \quad \text{for } p < j \le n; \quad
    e_i \wedge e_j = -e_j\wedge e_i \quad \text{for } i\neq j.
\]

\begin{wrapfigure}{r}{0.3\textwidth}
\vspace{-10pt}
\centering
\tdplotsetmaincoords{60}{120} % view angle

\scalebox{0.5}{
\begin{tikzpicture}[tdplot_main_coords, scale=3]

\begin{scope}[shift={(0,0,2.5)}]
  % Axes
  \coordinate (O) at (0,0,0);
  \coordinate (E1) at (1.4,0,0);
  \coordinate (E2) at (0,1.4,0);
  \coordinate (E3) at (0,0,1.4);

  \draw[->, thick, color=c1] (O) -- (E1) node[anchor=north west] {\LARGE \( e_1 \)};
  \draw[->, thick, color=c2] (O) -- (E2) node[anchor=north east] {\LARGE \( e_2 \)};
  \draw[->, thick, color=c3] (O) -- (E3) node[anchor=north west] {\LARGE \( e_3 \)};

  % e1 ^ e2 (xy-plane)
  \filldraw[fill=c12!20, draw=c12, opacity=0.6]
    (O) -- (1,0,0) -- (1,1,0) -- (0,1,0) -- cycle;
  \node at (1.2,1.1,0)  {\LARGE \( e_1 \wedge e_2 \)};
    \coordinate (C12) at (0.8,0.5,0);
  \begin{scope}[shift={(C12)}, canvas is xy plane at z=0]
    \draw[->, thick, c12] (0,0) arc (30:180:0.3);
  \end{scope}

  % e1 ^ e3 (xz-plane)
  \filldraw[fill=c13!20, draw=c13, opacity=0.6]
    (O) -- (1,0,0) -- (1,0,1) -- (0,0,1) -- cycle;
  \node[rotate=45] at (1.1,.25,1.48) {\LARGE \( e_1 \wedge e_3 \)};
  \coordinate (C13) at (0.7,0,0.4);
  \begin{scope}[shift={(C13)}, canvas is xz plane at y=0]
    \draw[->, thick, c13] (0,0) arc (0:180:0.2);
  \end{scope}

  % e2 ^ e3 (yz-plane)
  \filldraw[fill=c23!20, draw=c23, opacity=1]
    (O) -- (0,1,0) -- (0,1,1) -- (0,0,1) -- cycle;
  \node[rotate=-15] at (0.1,.7,1.22) {\LARGE \( e_2 \wedge e_3 \)};
  \coordinate (C23) at (0,0.65,0.4);
  \begin{scope}[shift={(C23)}, canvas is yz plane at x=0]
    \draw[->, thick, c23] (0,0) arc (0:180:0.2);
  \end{scope}
  
\end{scope}
\begin{scope}
  \node at (0.5,0.7,-0.45) {\LARGE \(e_1 \wedge e_2 \wedge e_3\)};

  %--- Cube corners
  \coordinate (O) at (0,0,0);
  \coordinate (A) at (1,0,0);
  \coordinate (B) at (1,1,0);
  \coordinate (C) at (0,1,0);
  \coordinate (D) at (0,0,1);
  \coordinate (E) at (1,0,1);
  \coordinate (F) at (1,1,1);
  \coordinate (G) at (0,1,1);

  %--- Cube edges
  \draw[thin] (A) -- (B) -- (C);
  \draw[thin] (D) -- (E) -- (F) -- (G) -- cycle;
  \draw[thin] (A) -- (E);
  \draw[thin] (B) -- (F);
  \draw[thin] (C) -- (G);

  %--- Axis segments
  \coordinate (E1cut) at (1,0,0); % front x-face
  \coordinate (E2cut) at (0,1,0); % front y-face
  \coordinate (E3cut) at (0,0,1); % top face
  \coordinate (E1end) at (1.4,0,0);
  \coordinate (E2end) at (0,1.4,0);
  \coordinate (E3end) at (0,0,1.4);

  %--- e1 axis
  \draw[dashed, thick, color=c1] (O) -- (E1cut);
  \draw[->, thick, color=c1] (E1cut) -- (E1end) node[anchor=north west] {\LARGE \(e_1\)};
  \filldraw[fill=tri, opacity=0.3, draw=none]
  (0,0,0) -- (1,0,0) -- (1,1,0) -- (0,1,0) -- cycle;

  %--- e2 axis
  \draw[dashed, thick, color=c2] (O) -- (E2cut);
  \draw[->, thick, color=c2] (E2cut) -- (E2end) node[anchor=north east] {\LARGE \(e_2\)};
  \filldraw[fill=tri, opacity=0.3, draw=none]
  (0,0,0) -- (1,0,0) -- (1,0,1) -- (0,0,1) -- cycle;
  
  %--- e3 axis
  \draw[dashed, thick, color=c3] (O) -- (E3cut);
  \draw[->, thick, color=c3] (E3cut) -- (E3end) node[anchor=north west] {\LARGE \(e_3\)};
  \filldraw[fill=tri, opacity=0.3, draw=none]
  (0,0,0) -- (0,1,0) -- (0,1,1) -- (0,0,1) -- cycle;
\end{scope}
\end{tikzpicture}
}
\captionof{figure}{The basis vectors, bivectors, and trivector for $\Cl(3)$}
\label{fig:wedge-cube}
\vspace{-20pt}
\end{wrapfigure}
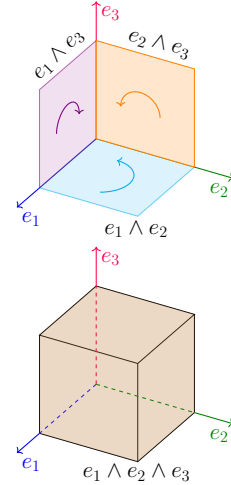

These relations encode the metric and orientation of the underlying space.
The algebra has a canonical basis of $2^n$ elements: the scalar $1$ and all distinct products of the basis vectors $e_1, \dots, e_n$. In particular, \textit{basis bivectors} are wedge products of two distinct basis vectors, i.e., $e_i \wedge e_j = e_i e_j$ for $i < j$. More generally, \textit{basis $k$-vectors} are wedge products of $k$ distinct basis vectors (see Fig. \ref{fig:wedge-cube}), and there are $\binom{n}{k}$ such elements for each $k$. The total number of such basis elements is the \textit{dimension} of the algebra. The \textit{reversion},  given by $\dagger$, reverses the order of basis vectors. For example, $(e_1 e_2)^\dagger = e_2 e_1$ and $(e_1 e_2 e_3)^\dagger = e_3 e_2 e_1$.

The algebra decomposes into a direct sum of subspaces indexed by \textit{grade}: scalars (grade $0$), vectors (grade $1$), \textit{bivectors} (grade $2$), and general \textit{$k$-vectors}, which are linear combinations of corresponding basis elements. A \textit{multivector} is a general element of $\Cl_{p,q}(\rr)$, expressed as a linear combination of components of multiple grades. For example, $e_1e_2$ is a basis bivector; $e_1e_3 + 2e_2e_3$ is a bivector; and $1 + (2e_2+e_4) - e_1e_3$ is a multivector composed of elements of grades $0,1$, and $2$. We denote the subspaces of vectors and bivectors by $\Cl^1(\rr)$ and $\Cl^2(\rr)$, and more generally, $\Cl^k(\mathbb{R})$ for grade-$k$ elements. The algebra naturally splits into \textit{even} and \textit{odd} subalgebras based on grade parity. The \textit{even subalgebra} $\Cl^+(n) \subset \Cl(n)$ consists of elements of even grade (scalars, bivectors, 4-vectors, etc), while odd elements include vectors, trivectors, and so on.
Some well-known algebraic systems arise as special cases of Clifford algebras: the real numbers $\rr \cong \Cl_{0,0}(\rr)$, the complex numbers $\mathbb{C} \cong \Cl_{0,1}(\rr)$, the quaternions $\mathbb{H}\cong\Cl_{0,2}(\rr)$, and the hyperbolic numbers $\Cl_{1,0}(\rr)$. This way, Clifford algebras naturally generalize familiar algebraic systems by incorporating geometric structure. To keep notations short, we write $\Cl_{p,q}(\rr)$ as $\Cl(p,q)$ and only consider $\Cl(n,0)$, denoted as $\Cl(n)$.
\section{Algebraic Structure of Rotor-based Transformations}
\label{sec:method}
We aim to describe standard linear transformations in terms of the algebraic structure of Clifford algebra. We begin by noting how any linear map can be expressed
as a sum of multivector products, and then show how restricting to rotors in the Spin group is helpful.

\paragraph{Clifford form of linear transformations.} A linear map between two vector spaces is one that preserves additivity and homogeneity. Traditionally, such transformations are represented as dense matrices with independent parameters.
% The standard matrix representation treats these transformations as a set of independent parameters.  
In contrast, we consider $\Cl(n)$ and write these transformations in terms of the geometric product between multivectors---algebraic objects that encode both magnitude and orientation. We first restate a textbook result. 
%This richer structure will play 
%a key role throughout. 

%\tp{Expand this section a bit}
\begin{lemma}
    \label{thm:multivector_transfromation}
    (\citet{hestenes2012clifford}) Let 
    $a_t$ and $b_t$ denote multivectors in $\Cl(n)$. Any linear function $F$ from $\Cl^k(n)$ to $\Cl(n)$ can be written as the finite sum for some width $w < \infty$,
    \begin{equation}
        F(x) = \sum_{t=1}^w a_t x b_t.
    \end{equation}
\end{lemma}
This result shows that linear transformations in Clifford algebra can be expressed as products involving multivectors acting from both the left and right. Thus, Clifford algebra gives us a way to represent a general, arbitrary linear map. But there is a cost: arbitrary multivectors have too much freedom and require all $2^n$ parameters of the full Clifford algebra, which makes the representation inefficient. We will constrain $a_t$ and $b_t$ to preserve \textit{rotational symmetries}, leading to the Spin group
% \[\Spin(n) \triangleq \left\{r\in \Cl^+(n) | rr^\dagger = 1\right\}\]
\[\Spin(n) \triangleq \left\{r\in \Cl^+(n) \mid rr^\dagger = 1 \text{ and } rvr^\dagger \in \Cl^1(n) \text{ for all } v \in \Cl^1(n)\right\}\]
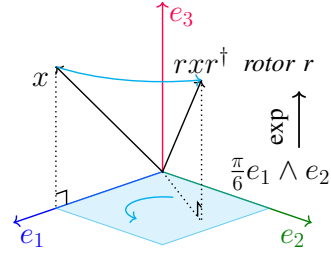
\begin{wrapfigure}[12]{r}{0.35\textwidth}
\vspace{-15pt}
% \raggedright
% \centering
\begin{center}
\tdplotsetmaincoords{70}{135} % view angle

\scalebox{.665}{
\begin{tikzpicture}[tdplot_main_coords, scale=3]

% \begin{scope}[shift={(0,0,2.5)}]
  % Axes
  \coordinate (O) at (0,0,0);
  \coordinate (E1) at (1.4,0,0);
  \coordinate (E2) at (0,1.4,0);
  \coordinate (E3) at (0,0,1.2);

  \draw[->, thick, color=c1] (O) -- (E1) node[anchor=north west] {\LARGE \( e_1 \)};
  \draw[->, thick, color=c2] (O) -- (E2) node[anchor=north east] {\LARGE \( e_2 \)};
  \draw[->, thick, color=c3] (O) -- (E3) node[anchor=north west] {\LARGE \( e_3 \)};

  % e1 ^ e2 (xy-plane)
  \filldraw[fill=c12!20, draw=c12, opacity=0.6]
    (O) -- (1,0,0) -- (1,1,0) -- (0,1,0) -- cycle;

  \begin{scope}[shift={(.3,0.4,0)}, canvas is xy plane at z=0]
    \draw[->, thick, c12] (0,0) arc (210:360:0.25);
  \end{scope}

  % \pgfmathsetmacro{\zshift}{1}
  % \filldraw[fill=c12!20, draw=c12, opacity=0.6]
  % ($(O)+(0,0,\zshift)$) --
  % ($(1,0,0)+(0,0,\zshift)$) --
  % ($(1,1,0)+(0,0,\zshift)$) --
  % ($(0,1,0)+(0,0,\zshift)$) -- cycle;

  % Vector x from origin to (1, 0, 1)
  \draw[->, thick, color=black] (O) -- (1,0,1) node[anchor=north east] {\LARGE \( x \)};

  \draw[->, thick, black]
  (O) -- (0.5,0.866,1)
    node[anchor=south] {\LARGE \(r x r^\dagger\)};

  \draw[dotted, thick, black] (0.5, 0.866, 1) -- (0.5, 0.866, 0);
  \draw[dotted, thick, black] (0,0,0) -- (0.5, 0.866, 0);
  \draw[thick, black] 
  (0.45, 0.779, 0) -- ++(0, 0, 0.1) -- ++(.05, .0866, 0);

\begin{scope}[canvas is xy plane at z=1]
  \draw[->, thick, c12] (1,0) arc[start angle=0, end angle=60, radius=1];
\end{scope}

\draw[dotted, thick, black] (1,0,1) -- (1,0,0);

\begin{scope}[shift={(.9,0,0)}, canvas is xz plane at y=0]
  \draw[thick] (0,0) -- ++(0,0.1) -- ++(0.1,0);
\end{scope}

\node at (-.58,.52,-0.03)  {\LARGE \(\frac{\pi}{6}e_1 \wedge e_2 \)};

\node[
  anchor=west,
  text width=2cm,
  align=left,
  font=\itshape\Large
] (rot) at (1,1.7,1.45) { 
  rotor r
};

\draw[->, line width=1pt, black] (-.58, .7, .2) -- (-.58, .7, .6);

\node[
  anchor=west,
  rotate=90,
  text width=2cm,
  align=left,
  font=\Large
] (rot) at (-.25,.85,.3) { 
  exp
};

% \end{scope}
\end{tikzpicture}
}
\end{center}
\vspace{-10pt}
\captionof{figure}{The sandwich product rotating a vector $60^\circ$ in the $e_1\wedge e_2$ plane.}
\label{fig:rotor}
% \vspace{40pt}
\end{wrapfigure}
where $\Cl^+(n)\subset\Cl(n)$ is the even subalgebra and $\dagger$ denotes grade-wise reversion. The Spin group captures the set of orientation-preserving rotations within $\Cl(n)$. The elements $r\in\Spin(n)$, called \textit{rotors}, act on multivectors $x\in\Cl(n)$ through the sandwich product as shown in Fig. \ref{fig:rotor},
\begin{equation}
    x \mapsto rxr^\dagger.
    \label{eq:rotor-sandwich}
\end{equation}
Applying multiple rotors in {\em parallel} instantiates Lem. $\ref{thm:multivector_transfromation}$, where $a_t=r_t$ and $b_t=r_t^\dagger$. We clarify that while the {\em general} Clifford algebra can represent any linear map, restricting to $\Spin(n)$ limits transformations to orthogonal (rotation-preserving) ones. Consequently, our construction does not capture arbitrary linear maps. In practice, expressivity is recovered by combining multiple rotor modules acting on different subspaces and aggregating their outputs (see Sec.~\ref{sec:experiment}). While this sandwich form is useful, we must efficiently parametrize these rotors. To do so, we first observe the relationship between $\Spin(n)$ and the more familiar rotation group, $\SO(n)$.

\begin{fact}[\citet{gallier2020differential}]
    $\Spin(n)$ is a double cover of the special orthogonal group $\SO(n)$ for $n\geq 3$.
\end{fact}

This relationship is important in that while $\Spin(n)$ and $\SO(n)$ are topologically different (hence the ``double cover''), they share the same infinitesimal structure---namely, their Lie algebra. This allows using parametrization techniques for $\SO(n)$ to represent elements of $\Spin(n)$. In particular, we will see that rotors can be generated from bivectors via the exponential map, just as rotation matrices arise from skew-symmetric matrices.

\begin{remark}
    For vector inputs $x$, the sandwich product $rxr^\dagger$ performs the same rotation as the corresponding matrix in $\SO(n)$. However, the full utility of rotors becomes clear when $x$ is multivector: the transformation extends to higher-grade components within the Clifford algebra, going beyond the vector subspace. This makes rotors especially suitable for operating on the richer representations. 
\end{remark}

We return to the question: can we efficiently represent rotors so that $a_t$, $b_t$ in Lem. \ref{thm:multivector_transfromation} belong to $\Spin(n)$?

% \begin{figure}[!bt]
%   \centering
%   \input{figures/process}
%   \caption{The [bivector $\rightarrow$ invariant decomposition $\rightarrow$ rotor decomposition $\rightarrow$ rotor] process that enables exact parametrization. Note a \textit{pure} rotor is one that corresponds to a \textit{simple} bivector.}
%   \label{fig:process}
% \end{figure}

\paragraph{Constructing rotors from bivectors.}
\textit{Rotors in $\Spin(n)$ can be parametrized via bivectors}: grade-2 elements of $\Cl(n)$ that encode oriented planes of rotation. To understand this parametrization, we examine the connection between rotors and rotation matrices via Lie groups and their Lie algebras. 
%We can parametrize rotors via the grade $2$ elements, bivectors, of $\Cl(n)$. To understand how, we must characterize the algebraic structures that house these elements.
\begin{definition}
    A Lie group $G$ is a smooth manifold with the usual group properties along with smooth (infinitely differentiable) group operations. Its associated Lie algebra is a vector space equipped with an antisymmetric, bilinear operation $[X,Y]$, called the Lie bracket, satisfying the Jacobi identity.
\end{definition}

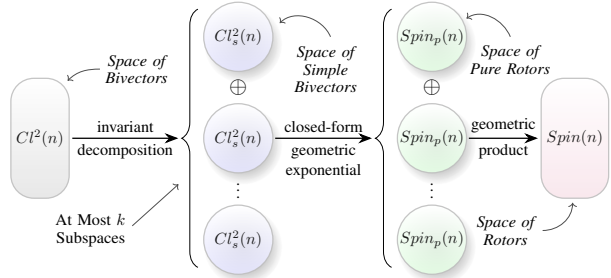
\begin{wrapfigure}[12]{r}{0.58\textwidth}
% \vspace{-20pt}
% \raggedright
% \centering
\begin{center}
\scalebox{.61}{
\begin{tikzpicture}[
    vector/.style={-{Stealth[length=3mm]}, thick, black},
    oval/.style={top color=white, bottom color=gray!20, draw=gray!50, thick, opacity=0.9, minimum width=2.2cm, minimum height=1cm, rounded corners=0.5cm},
    circlenode/.style={top color=white, bottom color=gray!20, draw=gray!50, thick, opacity=0.9, circle, minimum width=1.5cm, minimum height=1.5cm, inner sep=0pt},
    arrow/.style={-{Stealth[length=3mm]}, thick, black},
    oplus/.style={black, minimum size=0.6cm}
]

% Create shadows for main oval (left)
% \begin{scope}[opacity=0.2, transform canvas={xshift=0.15cm, yshift=-0.2cm}]
%     \shade[inner color=black, outer color=white, opacity=0.5] (0,0) ellipse (1.3cm and 0.6cm);
%     \shade[inner color=black!80, outer color=black!10] (0,0) ellipse (1.1cm and 0.5cm);
% \end{scope}

% % Main flow with much better spacing
% --- vertical shadow for main oval ---
\begin{scope}[xshift=1.2cm]
% SHADOW HERE
% \begin{scope}[opacity=0.2, transform canvas={xshift=0.2cm, yshift=-0.25cm}]
%     % swap (x‐radius and y‐radius)
%     \shade[inner color=black, outer color=white, opacity=0.5] (0,0) ellipse (0.6cm and 1.3cm);
%     \shade[inner color=black!80, outer color=black!10]  (0,0) ellipse (0.5cm and 1.1cm);
% \end{scope}

% --- vertical oval node ---
\node[oval, minimum width=1.2cm, minimum height=2.6cm] (cl2) {$Cl^2(n)$};
\end{scope}

\node[
  text width=1.5cm,
  align=center,
  font=\itshape
] (rot) at (3.3,1.5) { 
  Space of Bivectors
};

% Invariant decomposition arrow with much more space
\draw[arrow] (1.9,0) -- ++(2.3,0) node[midway, above] {invariant} node[midway, below] {decomposition};

% Create shadows for left circles
\begin{scope}[opacity=0.2, transform canvas={xshift=0.1cm, yshift=-0.15cm}]
    \shade[inner color=black, outer color=white, opacity=0.5] (5.5,2.2) circle (0.85cm);
    \shade[inner color=black!80, outer color=black!10] (5.5,2.2) circle (0.75cm);
\end{scope}

\begin{scope}[opacity=0.2, transform canvas={xshift=0.1cm, yshift=-0.15cm}]
    \shade[inner color=black, outer color=white, opacity=0.5] (5.5,0) circle (0.85cm);
    \shade[inner color=black!80, outer color=black!10] (5.5,0) circle (0.75cm);
\end{scope}

% Middle left block - Cl_s^2(n) circles with much more vertical spacing
\node[circlenode, bottom color=blue!10] (cls1) at (5.5,2.2) {$Cl_s^2(n)$};
\node[circlenode, bottom color=blue!10] (cls2) at (5.5,0) {$Cl_s^2(n)$};
% \node at (7,3.0) {\scriptsize space of simple bivectors};
\node[
  text width=1.5cm,
  align=center,
  font=\itshape
] (rot) at (7.4,1.5) { 
  Space of Simple Bivectors
};

% Vertical dots
\node at (5.5,-1.025) {$\vdots$};

% Add bottom Cl_s^2(n) circle below the dots
\begin{scope}[opacity=0.2, transform canvas={xshift=0.1cm, yshift=-0.15cm}]
    \shade[inner color=black, outer color=white, opacity=0.5] (5.5,-2.2) circle (0.85cm);
    \shade[inner color=black!80, outer color=black!10] (5.5,-2.2) circle (0.75cm);
\end{scope}
\node[circlenode, bottom color=blue!10] (cls3) at (5.5,-2.2) {$Cl_s^2(n)$};

% Superposition symbols between circles - aligned with circles
\node[oplus] (oplus1) at (5.5,1.1) {\Large$\oplus$};
% \node[oplus] (oplus2) at (5.5,-1.1) {$\oplus$};
%\node[oplus] (oplus3) at (5.5,-1.65) {$\oplus$};

% Arrow with floor(n/2) label - from right edge of circles
% \draw[arrow] (6.4,0) -- ++(2.0,0);
\draw[arrow] (6.3,0) -- ++(2.1,0) node[midway, above] {closed-form} node[midway, below] {\shortstack{geometric\\exponential}};

% Create shadows for right circles
\begin{scope}[opacity=0.2, transform canvas={xshift=0.1cm, yshift=-0.15cm}]
    \shade[inner color=black, outer color=white, opacity=0.5] (9.7,2.2) circle (0.85cm);
    \shade[inner color=black!80, outer color=black!10] (9.7,2.2) circle (0.75cm);
\end{scope}

\begin{scope}[opacity=0.2, transform canvas={xshift=0.1cm, yshift=-0.15cm}]
    \shade[inner color=black, outer color=white, opacity=0.5] (9.7,0) circle (0.85cm);
    \shade[inner color=black!80, outer color=black!10] (9.7,0) circle (0.75cm);
\end{scope}

% Middle right block - Spin_p(n) circles with same spacing as left AND plus symbols
\node[circlenode, bottom color=green!10] (spin1) at (9.7,2.2) {$Spin_p(n)$};
\node[circlenode, bottom color=green!10] (spin2) at (9.7,0) {$Spin_p(n)$};
\node[
  text width=1.72cm,
  align=center,
  font=\itshape
] (rot) at (11.4,1.7) { 
  Space of Pure Rotors
};

% Vertical dots
\node at (9.7,-1.025) {$\vdots$};

% Add bottom Spin_p(n) circle below the dots
\begin{scope}[opacity=0.2, transform canvas={xshift=0.1cm, yshift=-0.15cm}]
    \shade[inner color=black, outer color=white, opacity=0.5] (9.7,-2.2) circle (0.85cm);
    \shade[inner color=black!80, outer color=black!10] (9.7,-2.2) circle (0.75cm);
\end{scope}
\node[circlenode, bottom color=green!10] (spin3) at (9.7,-2.2) {$Spin_p(n)$};

% Add plus symbols for Spin_p(n) block too
\node[oplus] (spin_oplus1) at (9.7,1.1) {\Large$\oplus$};
% \node[oplus] (spin_oplus2) at (10.0,-1.1) {$\oplus$};
%\node[oplus] (spin_oplus3) at (10.0,-1.65) {$\oplus$};

% Final arrow to Spin(n) - from right edge of circles
\draw[arrow] (10.5,0) -- ++(1.5,0) node[midway, above] {geometric} node[midway, below] {product};

% Create shadow for final oval (right): SHADOW HERE
\begin{scope}[opacity=0.2, transform canvas={xshift=0.1cm, yshift=-0.2cm}]
    % \shade[inner color=black, outer color=white, opacity=0.5] (14.1,0) ellipse (0.6cm and 1.3cm);
    % \shade[inner color=black!80, outer color=black!10] (14.1,0) ellipse (0.5cm and 1.1cm);
    % \shade[inner color=black, outer color=white, opacity=0.5] (14.1,0) rectangle (16, 3);
\end{scope}

% Final result — vertical oval
\node[oval, bottom color=purple!10, minimum width=1.2cm, minimum height=2.6cm] (spinn) at (12.8,0) {$Spin(n)$};

\node[
  text width=1.5cm,
  align=center,
  font=\itshape
] (rot) at (11.3,-2) { 
  Space of Rotors
};

% Add curly brace for Cl_s^2(n) column
\draw [decorate,decoration={brace,amplitude=10pt,raise=5pt}] 
  (4.8,-2.8) -- (4.8,2.8);

% Add curly brace for Spin_p(n) column  
\draw [decorate,decoration={brace,amplitude=10pt,raise=5pt}] 
  (9,-2.8) -- (9,2.8);

\node[
  text width=1.5cm,
  align=center,
] (rot) at (2.3,-2) { 
  At Most $k$ Subspaces
};

\draw[->, thin, gray!60!black]
  (3.2, -1.9) -- (4.2, -1);

% Bottom equation - align each element with its corresponding column
% \node at (0,-4.0) {$b =$};
% \node at (5.5,-4.0) {$\displaystyle\sum_{i=1}^k b_i$};
% \draw[arrow] (6.5,-4.0) -- ++(2.5,0);
% \node at (10.0,-4.0) {$\displaystyle\prod_{i=1}^{\left\lfloor\frac{n}{2}\right\rfloor} r_i$};
% \draw[arrow] (12.0,-4.0) -- ++(1.5,0);
% \node at (14.5,-4.0) {$r$};

\draw[->, thin, gray!60!black]
  (2.5, 1.5) 
    to[out=180,in=45] 
      (1.8,1.25);

\draw[->, thin, gray!60!black]
  (7.5, 2.2) 
    to[out=112.5,in=45] 
      (6.4, 2.5);

\draw[->, thin, gray!60!black]
  (11.4, 2.2) 
    to[out=112.5,in=35] 
      (10.45, 2.6);

\draw[->, thin, gray!60!black]
  (12.1, -2) 
    to[out=0,in=-90] 
      (12.75, -1.4);

\end{tikzpicture}
}
\end{center}
\vspace{-5pt}
\captionof{figure}{The [bivector $\rightarrow$ invariant decomposition $\rightarrow$ rotor decomposition $\rightarrow$ rotor] process that enables exact parametrization. Note that a \textit{pure} rotor is one that corresponds to a \textit{simple} bivector.}
\label{fig:process}
% \vspace{-20pt}
\end{wrapfigure}

Let us check examples. $\Spin(n)$ and $\SO(n)$ are different Lie groups with the same Lie algebra. Specifically, the Lie algebra of skew-symmetric matrices,
        $\so(n) \triangleq \left\{B\in\rr^{n\times n} | B = -B^T\right\},$
underlies both $\Spin(n)$ and $\SO(n)$, despite their topological differences. This shared structure is important. The exponential map gives a surjective correspondence between $\so(n)$ to $\SO(n)$ as
\begin{equation}
    \exp(B) = \sum_{i=0}^\infty \frac{B^i}{i!},
    \label{eq:mat_exp}
\end{equation}
showing that every rotation in $\SO(n)$ can be realized as $\exp(B)$ for some $B\in\so(n)$ using only $\dim(\so(n))={n \choose 2}$ independent parameters---fewer than the $n^2$ entries of a matrix \citep{lezcanocasado2019cheaporthogonalconstraintsneural}. The bridge to the Clifford algebra setting is from another key isomorphism:
\begin{fact}[\citet{doran2003geometric}]
    The space of skew-symmetric matrices is isomorphic to that of bivectors, i.e., $\so(n)\cong\Cl^2(n)$.
\end{fact}

Given a skew-symmetric matrix $B \in \so(n)$, the corresponding bivector $b \in \Cl^2(n)$ is constructed as 
$$
b = \sum_{1 \le i < j \le n} B_{i,j} \, e_i \wedge e_j.$$ 
Similar to the exponential map relating $\so(n)$ and $\SO(n)$ (i.e., $\exp(B)$ generating rotations in $\SO(n)$), every rotor $r \in \Spin(n)$ is the exponential of some bivector $b\in\Cl^2(n)$ given explicitly by the series
\begin{equation}
    \label{eq:bad_exp}
    r = \clexp(b)=\sum_{i=0}^\infty\frac{b^i}{i!},
\end{equation}
where $b^k$ is the $k$-fold geometric product of $b$ with itself. This means that we can encode a rotor with only $\dim\left(\Cl^2(n)\right) = {n \choose 2}$ parameters---the dimension of the bivector space. 

\paragraph{Main advantage.} At first glance, our parametrization may seem to offer no advantage. We have parametrized both $\Spin(n)$ and $\SO(n)$ from the same ${n \choose 2}$ parameters from $\so(n)$. %But this view misses an 
%important difference. 
Notice that while $\SO(n)$ acts only on $n$-dimensional vectors, rotors in $\Spin(n)$ act on a subset of the full $2^n$-dimensional space of multivectors. 
%whereas $\Spin(n)$ acts on multivector valued input. 
By utilizing \textit{rotors as the action on multivector input and bivectors as our irreducible primitives}, 
%we can construct a subset of the possible linear maps on $2^n$ dimensional input that 
we will use exponentially fewer parameters. 
%than a dense linear layer. 

\begin{exmp}
A $d\times d$ dense matrix with $d=2048$, common in self-attention blocks and projection layers (e.g., in LLaMa), uses more than $4$M parameters. If 
$w$ (e.g., $\simeq 3$) is the width hyperparameter in Lem. \ref{thm:multivector_transfromation}, 
approximating with bivector irreducibles requires only $w~{\log_2 d \choose 2} = 55~ w.$
\end{exmp}
This reduction (to $55 \times 3$) comes from identifying bivectors as the primitive that generate rich linear maps through composition. 
%\tp{Tie back to irreducible message more}
Note that the infinite series in $\eqref{eq:bad_exp}$ is problematic due to the need for approximation. We avoid this by presenting a closed-form solution that preserves differentiability.

% to act on a multivector to realize local rotations in local subspaces. 

% using a composition of lower-level geometric primitives: namely, \textit{bivectors} and their corresponding \textit{rotors}. These rotors act on input data represented as multivectors in a Clifford algebra. In this section, we describe the core construction and algorithms.
% \input{section/method}
\section{Algorithmic Implementation and Analysis of our Rotor-gadget}
\label{sec:how_we_do_it}
We discussed above how rotors can be parametrized through bivectors via the exponential map in \eqref{eq:bad_exp}. A remaining challenge is the infinite series. Of course, we can truncate to a finite length and incur approximation errors. However, for a special class of bivectors, the exponential map admits an exact closed-form solution, which prevents any approximation errors. Moreover, this form remains fully differentiable. We describe the details of this alternative here. 

\subsection{Closed-form differentiable computations of rotors}\label{sec:compute-rotors}
A key observation is that when a skew-symmetric matrix $B\in\so(n)$ generates a rotation restricted to a single $2$-dimensional plane, the matrix exponential in \eqref{eq:mat_exp} reduces to a finite closed-form expression---mirroring the simplicity of the classic Rodrigues formula for axis-angle rotations \citep{goldstein2002classical}. An analogous simplification holds for bivectors. The exponential map in \eqref{eq:bad_exp} admits a closed form when $b \in \Cl^2(n)$ is \textit{simple}, meaning it can be written as $b=u\wedge v$ for some vectors $u, v\in\Cl^1(n)$, or equivalently $b\wedge b = 0$. This ensures that $b$ represents a single-plane rotation rather than a composition of rotations. The resulting closed-form expression is:
\begin{equation}
    \clexp(b)=\cos(\|b\|) + \frac{\sin(\|b\|)}{\|b\|}b.
    \label{eq:good_exp}
\end{equation}
This form is exact \citep{doran2003geometric}. However, restricting to simple bivectors is limiting, as simple bivectors span only a \textit{subset} of $\Cl^2(n)$, and thus generate only a \textit{subset} of $\Spin(n)$. To capture the full expressivity of rotor-based transformations, we wish to extend to general bivectors. 

\paragraph{Building bivectors from simple bivectors.}
To utilize the closed-form exponential in \eqref{eq:good_exp} while retaining the full richness of $\Cl^2(n)$, we need a way to express arbitrary bivectors in terms of simple ones. Fortunately, \citet{roelfs2021gradedsymmetrygroupsplane} show that any bivector $b \in \Cl^2(n)$ admits an \textit{invariant decomposition} as a sum of at most $k\triangleq\lfloor n/2\rfloor$ mutually commuting, orthogonal, simple bivectors $\{b_1, b_2, \ldots, b_k\}$. This decomposition has two key advantages: (1) each component $b_i$ admits an efficient closed-form solution $\exp(b_i)$ in \eqref{eq:good_exp} since $b_i$ is simple, and (2) mutual commutativity ensures $\exp\left(\sum_{i=1}^kb_i\right) = \prod_{i=1}^k\exp(b_i)$, via the standard Lie algebra identity $\exp(X+Y)=\exp(X)\exp(Y)$ when their commutator $[X,Y]=0$. In $\so(n)$, the Lie bracket is $[X,Y]=XY - YX$, which vanishes when $X$ and $Y$ commute. With these benefits in place, the remaining question is how to construct the decomposition? A recent result provides a spectral formulation in terms of the eigenvectors and eigenvalues of $b$, stated below.
\begin{lemma}[Thm. 4.8 in \citet{eelbode2024outereigentangentconcepts}]
    Let $b\in\Cl^2(n)$ have as many eigenvectors as its effective pseudo-dimension. Then, we have
    $$
        b = \sum_{j=1}^{k} \mu_{j} \frac{v_{\mu_{j}^{+}} \wedge v_{\mu_{j}^{-}}}{v_{\mu_{j}^{+}} \cdot v_{\mu_{j}^{-}}},
    $$
    where $\sigma(b)=\{\pm \mu_1, \ldots, \mu_k\}$ is the spectrum of $b$ and $v_{\mu_j^+}$ and $v_{\mu_j^-}$ are partner eigenvectors.
    \label{lem:invariantDecomp}
\end{lemma}

\begin{figure}[bt]
\begin{minipage}[t]{0.48\textwidth}
\begin{algorithm}[H]
    \caption{Differentiable Inv. Decomp.}
    \label{algo:inv_decomp}
    {\footnotesize
    \begin{algorithmic}[1]
    \REQUIRE $b\in\Cl^2(n)$, $v\in\Cl^1(n)^{k-1}$
    \ENSURE Inv. Decomp. $\{b_1,\dots b_k\}$,
    
    singular vectors $\{v_1,\ldots,v_{k-1}\}$
    \STATE Initialize $\texttt{decomp}, \texttt{vectors} \gets \emptyset, \emptyset$
    \FOR{$i = 1$ to $k-1$}
        \STATE $b_s, v_i \gets \texttt{Proj}_{simple}(b, v_i)$
        \STATE $b \gets b - b_s$
        \STATE $\texttt{decomp} \gets \texttt{decomp} \cup \{b_s\}$
        \STATE $\texttt{vectors} \gets \texttt{vectors} \cup \{v_i\}$
    \ENDFOR
    \STATE $\texttt{decomp} \gets \texttt{decomp} \cup \{b\}$
    \STATE \textbf{return} \texttt{decomp}, \texttt{vectors}
    \end{algorithmic}
    }
\end{algorithm}
\end{minipage}
\hfill
\begin{minipage}[t]{0.48\textwidth}
\begin{algorithm}[H]
    \caption{GA Power Iteration}
    \label{algo:ga_poweriter}
    {\footnotesize
    \begin{algorithmic}[1]
    \REQUIRE $b\in\Cl^2(n)$, $v\in\Cl^1(n)$, 
    
    threshold $\epsilon \in \rr$
    \ENSURE Approximate $\texttt{Proj}_{simple}(b)$
    \STATE $v_{\text{prev}} \gets v$
    \WHILE{$\lVert v + v_{\text{prev}}\rVert > \epsilon$}
        \STATE $v_{\text{prev}} \gets v$
        \STATE $v \gets b \llcorner (b \llcorner v)$
        \STATE $v \gets v / \|v\|_2$
    \ENDWHILE
    \STATE $\sigma u = b \llcorner v$
    \STATE $b_s \gets \sigma u \wedge v$
    \RETURN $b_s$, $v$
    \end{algorithmic}
    }
\end{algorithm}
\end{minipage}
\end{figure}

\paragraph{Differentiable invariant decomposition.} 
While the invariant decomposition can be computed using eigendecomposition, standard  algorithms pose challenges. The eigenvalues of a bivector come in conjugate pairs, and singular values come in positive pairs, making differentiation (and backpropagation) not as straightforward due to the numerical instability of eigen-decomposition for near-degenerate singular values. To address this, we introduce a Krylov subspace-inspired algorithm that iteratively extracts the simple bivectors without requiring explicit eigendecomposition. The procedure is shown in Alg. $\ref{algo:inv_decomp}$. It includes a subroutine for projecting a bivector onto the manifold of simple bivectors, which we accomplish with a Clifford algebraic adaptation of the power iteration method shown in Alg. $\ref{algo:ga_poweriter}$. This uses right contraction $b \llcorner v$, which extracts the components of $b$ that lie in the direction of $v$, avoiding the need to construct explicit matrix representations.

The projection has the closed-form expression $\Projsimple(b) = \sigma(u \wedge v)$ where $\sigma$ is the top singular value of $b$, and $u,v$ are the corresponding left and right singular vectors. Note that because of sign symmetry of the paired singular vectors, we detect convergence by the sum, not their difference, and threshold $\epsilon$. Further discussion and proofs of Alg.$\ref{algo:inv_decomp}$ --$\ref{algo:ga_poweriter}$ are in Appendix B. We provide a full visualization of bivector to rotor in Fig. $\ref{fig:process}$.

\paragraph{Computation graph size.} To control the size of the computation graph, we adapt DEQ \citep{bai2019deepequilibriummodels}: we run the fixed-point iteration without gradient tracking, and then perform a single final forward pass with tracking enabled. This ensures that the graph scales only with $k$ (number of components in decomposition), and not the number of iterations in Alg. $\ref{algo:ga_poweriter}$, which depends on the spectral gap and threshold $\epsilon$. Also, while one could initialize $v$ in Alg. $\ref{algo:ga_poweriter}$ using a non-differentiable SVD, this is unnecessary. Under small perturbations of $b$, the singular vectors vary smoothly. This follows from matrix perturbation results, which guarantees that in non-degenerate cases, eigenspaces vary analytically with the matrix entries (Ch. 2 of \citet{kato1980perturbation}). Thus, warm-starting with the previous singular vectors yields fast convergence, if gradient steps are not too large.

\begin{wrapfigure}{r}{0.51\textwidth}
\vspace{-25pt}
% \raggedright
% \centering
\begin{center}
% \begin{figure}[t]
% \vspace{-5pt}
% \centering
\scalebox{.875}{

\begin{tikzpicture}[every node/.style={font=\small}, scale=1, >=latex]
\tikzset{every node/.style={font=\small, line width=0.03mm}}

% Create shadows for leftmost node (x) - fixed alignment and opacity
\begin{scope}[opacity=0.1, transform canvas={xshift=0.2cm, yshift=-0.1cm}]
    \shade[inner color=black, outer color=white, opacity=0.25] (-0.2,-2) rectangle (0.2,2);
    \shade[inner color=black!80, outer color=black!10] (-0.15,-1.95) rectangle (0.15,1.95);
\end{scope}

% Full input vector x with gradient and divisions
\node[draw=gray!50, fill=none, minimum height=4cm, minimum width=0.4cm] (x) at (0,0) {};

% Add divisions to the input vector
\begin{scope}
    \clip (-0.2,-2) rectangle (0.2,2);
    \fill[top color=white, bottom color=gray!15] (-0.2,1.3) rectangle (0.2,2);
    \fill[top color=white, bottom color=gray!15] (-0.2,0.6) rectangle (0.2,1.3);
    \fill[top color=white, bottom color=gray!15] (-0.2,-0.1) rectangle (0.2,0.6);
    \fill[top color=white, bottom color=gray!15] (-0.2,-0.8) rectangle (0.2,-0.1);
    \fill[top color=white, bottom color=gray!15] (-0.2,-1.5) rectangle (0.2,-0.8);
    \fill[top color=white, bottom color=gray!15] (-0.2,-2) rectangle (0.2,-1.5);
\end{scope}

% Create shadows for input blocks
\begin{scope}[opacity=0.2, transform canvas={xshift=0.1cm, yshift=-0.15cm}]
    \shade[inner color=black, outer color=white, opacity=0.25] (2,1) rectangle (2.4,2);
    \shade[inner color=black!80, outer color=black!10] (2,1.05) rectangle (2.35,1.95);
\end{scope}

\begin{scope}[opacity=0.2, transform canvas={xshift=0.1cm, yshift=-0.15cm}]
    \shade[inner color=black, outer color=white, opacity=0.25] (2,-0.5) rectangle (2.4,0.5);
    \shade[inner color=black!80, outer color=black!10] (2,-0.45) rectangle (2.35,0.45);
\end{scope}

\begin{scope}[opacity=0.2, transform canvas={xshift=0.1cm, yshift=-0.15cm}]
    \shade[inner color=black, outer color=white, opacity=0.25] (2,-2) rectangle (2.4,-1);
    \shade[inner color=black!80, outer color=black!10] (2,-1.95) rectangle (2.35,-1.05);
\end{scope}

% Input blocks (red, blue, green) with gradients
\node[draw=gray!50, fill=none, minimum height=1.2cm, minimum width=0.3cm] (r) at (2,1.5) {};
\fill[top color=red!20, bottom color=red!40, opacity=0.8] (1.85,0.9) rectangle (2.15,2.1);

\node[draw=gray!50, fill=none, minimum height=1.2cm, minimum width=0.3cm] (b) at (2,0) {};
\fill[top color=blue!20, bottom color=blue!40, opacity=0.8] (1.85,-0.6) rectangle (2.15,0.6);

\node[draw=gray!50, fill=none, minimum height=1.2cm, minimum width=0.3cm] (g) at (2,-1.5) {};
\fill[top color=green!15, bottom color=green!30, opacity=0.8] (1.85,-2.1) rectangle (2.15,-0.9);

% Straight arrows from full input to input blocks with enhanced styling
\foreach \i/\src in {0/1.6, 2/0.3} {
  \draw[->, thick, black!80] (x.east) ++(0,\src) -- ([yshift={0.3cm - \i*0.3cm}]r.west);
}
\foreach \i/\src in {0/1, 1/-0.6, 2/-1.2} {
  \draw[->, thick, black!80] (x.east) ++(0,\src) -- ([yshift={0.3cm - \i*0.3cm}]b.west);
}
\foreach \i/\src in {1/-0.2, 2/-1.7} {
  \draw[->, thick, black!80] (x.east) ++(0,\src) -- ([yshift={0.3cm - \i*0.3cm}]g.west);
}

% Create shadows for output blocks - ONLY ONE SET, WITH PROPER OPACITY
\begin{scope}[opacity=0.2, transform canvas={xshift=0.2cm, yshift=-0.15cm}]
    \shade[inner color=black, outer color=white, opacity=0.5] (3.7,0.5) rectangle (4.6,1.7);
    \shade[inner color=black!80, outer color=black!10] (3.75,0.55) rectangle (4.55,1.65);
\end{scope}

\begin{scope}[opacity=0.2, transform canvas={xshift=0.2cm, yshift=-0.15cm}]
    \shade[inner color=black, outer color=white, opacity=0.5] (3.7,-1.7) rectangle (4.6,-0.5);
    \shade[inner color=black!80, outer color=black!10] (3.75,-1.65) rectangle (4.55,-0.55);
\end{scope}

% Output block 1 (top) with enhanced gradients
\begin{scope}
  \clip (3.7,0.5) rectangle (4.6,1.7);
  \fill[top color=red!20, bottom color=red!35]   (3.7,0.5) rectangle (4,1.7);
  \fill[top color=blue!20, bottom color=blue!35]  (4,0.5) rectangle (4.3,1.7);
  \fill[top color=green!15, bottom color=green!25] (4.3,0.5) rectangle (4.6,1.7);
\end{scope}
\node[draw=gray!50, fill=none, minimum height=1.2cm, minimum width=0.9cm] (topout) at (4.15,1.1) {};

% Output block 2 (bottom) with enhanced gradients
\begin{scope}
  \clip (3.7,-1.7) rectangle (4.6,-0.5);
  \fill[top color=red!20, bottom color=red!35]   (3.7,-1.7) rectangle (4,-0.5);
  \fill[top color=blue!20, bottom color=blue!35]  (4,-1.7) rectangle (4.3,-0.5);
  \fill[top color=green!15, bottom color=green!25] (4.3,-1.7) rectangle (4.6,-0.5);
\end{scope}
\node[draw=gray!50, fill=none, minimum height=1.2cm, minimum width=0.9cm] (botout) at (4.15,-1.1) {};

% Create shadow for final output
\begin{scope}[opacity=0.2, transform canvas={xshift=0.1cm, yshift=-0.15cm}]
    \shade[inner color=black, outer color=white, opacity=0.25] (6.15,-1.75) rectangle (6.55,1.75);
    \shade[inner color=black!80, outer color=black!10] (6.18,-1.7) rectangle (6.52,1.7);
\end{scope}

% Final output block with divisions
\node[draw=gray!50, fill=none, minimum height=3.5cm, minimum width=0.4cm] (y) at (6.15,0) {};

% Add divisions to the output vector
\begin{scope}
    \clip (5.95,-1.75) rectangle (6.35,1.75);
    \fill[top color=white, bottom color=gray!15] (5.95,1.2) rectangle (6.35,1.75);
    \fill[top color=white, bottom color=gray!15] (5.95,0.65) rectangle (6.35,1.2);
    \fill[top color=white, bottom color=gray!15] (5.95,0.1) rectangle (6.35,0.65);
    \fill[top color=white, bottom color=gray!15] (5.95,-0.45) rectangle (6.35,0.1);
    \fill[top color=white, bottom color=gray!15] (5.95,-1) rectangle (6.35,-0.45);
    \fill[top color=white, bottom color=gray!15] (5.95,-1.55) rectangle (6.35,-1);
    \fill[top color=white, bottom color=gray!15] (5.95,-1.75) rectangle (6.35,-1.55);
\end{scope}

% Enhanced arrows from input blocks to output blocks
% Top output
\draw[->, thick, black!80] (r.east) -- ([yshift=0.1cm]topout.west);
\draw[->, thick, black!80] (b.east) -- (topout.west);
\draw[->, thick, black!80] (g.east) -- ([yshift=-0.1cm]topout.west);

% Bottom output
\draw[->, thick, black!80] (r.east) -- ([yshift=0.1cm]botout.west);
\draw[->, thick, black!80] (b.east) -- (botout.west);
\draw[->, thick, black!80] (g.east) -- ([yshift=-0.1cm]botout.west);

% Enhanced arrows from topout to y
\draw[->, thick, black!80] ([yshift=0.1cm]topout.east) -- ([yshift=1.3cm]y.west);
\draw[->, thick, black!80] (topout.east)               -- ([yshift=0.4cm]y.west);
\draw[->, thick, black!80] ([yshift=-0.1cm]topout.east)-- ([yshift=-0.5cm]y.west);

% Enhanced arrows from botout to y
\draw[->, thick, black!80] ([yshift=0.1cm]botout.east) -- ([yshift=1cm]y.west);
\draw[->, thick, black!80] (botout.east)               -- ([yshift=0.2cm]y.west);
\draw[->, thick, black!80] ([yshift=-0.1cm]botout.east)-- ([yshift=-1.1cm]y.west);

% Labels below each major component
\node[align=center] at (0,-2.5)     {$x \in \rr^{d_\text{in}}$}; % Full input
\node[align=center] at (2,-2.6)     {$x^{I_{i}} \in \Cl(n)$ \\ \tiny{$i \in \{1,2,3\}$}}; % Input blocks
\node[align=center] at (4.15,-2.6)  {$y^{O_{j}} \in \Cl(n)$ \\ \tiny{$j \in \{1,2\}$}}; % Output blocks
\node[align=center] at (6.15,-2.5)  {$y \in \rr^{d_{\text{out}}}$}; % Full output

% Label above the center between input blocks and output blocks
\node[align=center] at (3,1.95) {\footnotesize{Six rotors} \\$\psi_{r_{ij},s_{ij}}$};

% Label above the center between input blocks and output blocks
\node[align=center] at (5.3,1.9) { \footnotesize{pooling}};

\end{tikzpicture}

}
% \vspace{-5pt}
% \caption{\small
% Rotor architecture with $c_1 = 3$ and $c_2 = 2$. 
% An input $x$ is split into $\{x^{I_i}\}_{i \in [c_1]}$, each mapped to $y^{O_j}$ via rotor maps $\psi_{r_{ij}, s_{ij}}$, for each $j \in [c_2]$. 
% The outputs $\{y^{O_j}\}$ are pooled and assembled into the final output $y$.}
% \label{fig:rotor-arch}
% \vspace{-5pt}
% \end{figure}
\end{center}
\vspace{-10pt}
\captionof{figure}{Rotor architecture with $c_1 = 3$ and $c_2 = 2$. An input $x$ is split into $\left\{x^{I_i}\right\}_{i \in [c_1]}$, each mapped to $y^{O_j}$ via rotor maps $\psi_{r_{ij}, s_{ij}}$, for each $j \in [c_2]$. The outputs $\left\{y^{O_j}\right\}$ are pooled and assembled into the final output $y$.}
\label{fig:rotor_map}
\vspace{3pt}
\end{wrapfigure}
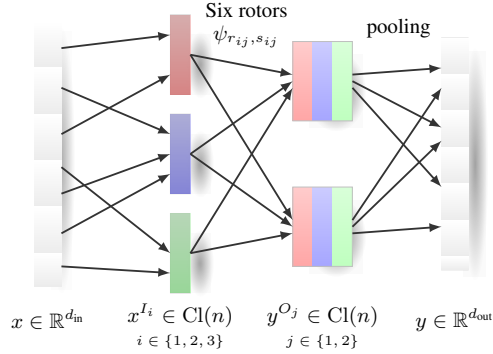

\subsection{A Generalized Rotor Gadget}
\label{subsc:genRotorGadget}
Occasionally, we may need mappings between arbitrary dimensional spaces. To do so, we now describe a generalized rotor gadget. Instead of the standard sandwich product in \eqref{eq:rotor-sandwich}, we can allow two different rotors on left/right for more expressiveness. Define the rotor-based transformation as:
\begin{equation}
    \psi_{r,s}(x) \triangleq r x s^\dagger,
    \label{eq:rotor-transformation}
\end{equation}
where $r,s\in\Spin(n)$. This construction, however, assumes the linear map acts on data with dimension of the Clifford algebra, a power of $2$. To address this, we use multiple rotor-sandwich modules operating on different subspaces. For arbitrary input and output dimensions $d_\text{in}$ and $d_\text{out}$, we utilize the rotor-sandwich modules $\psi_{r,s}(x)$ as building blocks to construct a map $\psi:\rr^{d_\text{in}} \rightarrow \rr^{d_\text{out}}$. Fix positive integers $c_1, c_2,$ and $n$ with $2^n \le \min(d_\text{in},d_\text{out})$ and let $[h]\triangleq\{1,2,\ldots,h\}$. For each $i\in [c_1]$ and $j\in [c_2]$, define $I_i \subseteq [d_\text{in}]$ and $O_{j} \subseteq [d_\text{out}]$ as $2^n$ subsets of input and output coordinates, respectively. We associate each pair with a rotor map $\psi_{r_{ij},s_{ij}}$ defined by rotors $r_{ij}, s_{ij} \in \Spin(n)$, parametrized by their corresponding bivectors $a_{ij}, b_{ij}\in\Cl^2(n)$. Then, each sub-module operates on $\Cl(n)$ and computes the rotor-sandwich action $\psi_{r_{ij},s_{ij}} : \rr^{I_i}\rightarrow\rr^{O_j}$.
The full output is defined by aggregating all $c_1c_2$ rotor maps:
\begin{equation}
    \psi(x) \triangleq \sigma\left(\left\{\psi_{r_{ij},s_{ij}}\left(x^{I_i}\right) | \; i\in[c_1], j\in[c_2]\right\}\right),
    \label{eq:rotor}
\end{equation}
where $\sigma$ is a pooling operator on the outputs of $\psi_{r_{ij},s_{ij}}$. Note that $\cup_{i} I_i = [d_{\text{in}}]$ and $\cup_{j} O_j = [d_{\text{out}}]$ are needed to fully cover the input and output dimensions. This gives us a general rotor-based transformation from arbitrary input and output dimensions parametrized by bivectors which is visualized in Fig. \ref{fig:rotor_map}. We now analyze the number of learnable parameters $\psi$ requires.

\begin{theorem}[$\psi$ Parameter Count]\label{thm:global-map}
Let \( \psi: \mathbb{R}^{d_{\text{in}}} \to \mathbb{R}^{d_{\text{out}}} \) be the mapping defined above, composed of rotor modules 
\(\psi_{r_{ij}, s_{ij}}\) with $i\in[c_1]$ and $j\in[c_2]$, each acting in $\Cl(n)$ with $2^n \le \min(d_{\text{in}}, d_{\text{out}}) \triangleq d$. The total number of learnable parameters is upper bounded by
\[
    2 c_1 c_2 \binom{n}{2} = \mathcal{O}\left(\log^2 d\right).
\]
\end{theorem}

Thm. \ref{thm:global-map} shows that rotor maps use only \( \mathcal{O}\left(\log^2 \min(d_{\text{in}}, d_{\text{out}})\right) \) parameters. In comparison, a standard dense layer and rank-$r$ factorization require \( \mathcal{O}(d_{\text{in}} d_{\text{out}}) \) and  \( \mathcal{O}(r(d_{\text{in}} + d_{\text{out}})) \), respectively.

\section{Experiments}
\label{sec:experiment}

\paragraph{Goals.}
We empirically evaluate rotors by \textit{replacing key, query, and value linear layers in pre-trained LLMs} and measuring downstream performance on perplexity (PPL) and accuracy. Our experiments span multiple models and datasets. The main goals are to:
\begin{inparaenum}[\bfseries (G-1)]
    \item Demonstrate the \textit{feasibility of composing linear layers from bivector primitives} by assessing whether rotors match baseline performance of Low-Rank and Block-Hadamard approximations across diverse settings.
    \item Quantify rotor parameter efficiency compared to dense and approximate alternatives.
    \item  Analyze how rotor architectural choices---such as width and depth---affect performance.
\end{inparaenum}

We focus on linear layers in smaller pre-trained language models (up to 1.5B parameters), where reduced redundancy compared to larger models makes preserving performance harder. We do not attempt full model conversion---which would require larger calibration datasets and a layerwise optimization scheme like GPTQ \citep{frantar2022gptq}. Instead, we selectively replace 1-3 attention layers (key, query, and value projections) to isolate the effect, and assess whether rotor decomposition offers competitive PPL and accuracy relative to baselines. 

\paragraph{End-to-end training with rotors.}
While the main LLM experiments train rotor layers to \textit{mimic individual dense layers in isolation} (see Sec.~\ref{sec:exp_setup}), we also assess their behavior when used throughout a network trained jointly from scratch. To keep the setup lightweight, we replace \textit{all dense layers} in a simple MLP (except the classification head) with rotor layers and train the model \textit{end-to-end} on \texttt{FMNIST} \citep{xiao2017fashionmnistnovelimagedataset} under identical conditions as the dense baseline. Details and results are provided in Sec.~\ref{sec:fmnist-results}, with experimental settings in Appendix~\ref{apdx:hyper-parameters}. With this experiment, we aim to complement the layerwise LLM analysis by testing full-network training with rotor layers.

\subsection{Experimental Setup}\label{sec:exp_setup}
\paragraph{Substituting attention layers.}
Given $x \in \mathbb{R}^d$ to a transformer block, we have:
\[
    \text{Attn}(x) = \left[\text{softmax}\left(\text{mask}\left(\frac{QK^\top}{\sqrt{d}}\right)\right)V\right]W_o,
\]
where the query, key, and value linear projections are defined as $Q = W_q x$, $K = W_k x$, and $V = W_v x$, with $W_q$, $W_k$, and $W_v$ as dense learnable matrices. We \textit{jointly} replace these linear layers in $1$–$3$ selected attention layers with rotors or baseline approximations, \textit{keeping other parameters fixed} except, for consistency, retraining the corresponding output linear layer $W_o$ after each substitution.

%\begin{table*}[t]
\begin{wraptable}{r}{0.65\textwidth}
%\vspace{-10pt}
\scriptsize
\centering
\setlength{\tabcolsep}{1.4pt}
\renewcommand{\arraystretch}{0.9}
\begin{tabular}{
>{\centering\arraybackslash}m{0.3cm}
>{\centering\arraybackslash}m{1.1cm}
>{\centering\arraybackslash}m{1.0cm}
*{3}{>{\centering\arraybackslash}m{0.85cm}} @{\hspace{2pt}}
*{3}{>{\centering\arraybackslash}m{0.85cm}} 
% @{\hspace{2pt}}
% *{3}{>{\centering\arraybackslash}m{0.85cm}}
}
\toprule
& \textbf{Dataset} & \textbf{Method} 
& \multicolumn{3}{c}{\textbf{LLaMa-3.2 1B}} 
& \multicolumn{3}{c}{\textbf{Qwen-2.5 1.5B}} \\
% & \multicolumn{3}{c}{\textbf{Qwen-2.5 0.5B}} \\
& & 
& one & two & three 
& one & two & three  \\
% & one & two & three  \\
\midrule
\multirow{15}{*}{\makecell{\rotatebox{90}{Log-PPL}}}
& \multirow{5}{*}{\makecell{\rotatebox{90}{Wikitext2}}}   
& Original & \multicolumn{3}{c}{----------- 2.575 -----------} & \multicolumn{3}{c}{----------- 2.287 -----------} \\
& & LR1    & 2.688 & 3.455 & 4.956 & 2.350 & 2.402 & 2.591 \\
& & LR4    & 2.658 & 2.729 & 2.880 & 2.342 & \textcolor{second_best}{2.372} & \textcolor{second_best}{2.548} \\
& & BH1    & \textcolor{second_best}{2.636} & \textcolor{best}{2.700} & \textcolor{best}{2.779} & \textcolor{second_best}{2.323} & 2.388 & 2.558 \\
& & \cellcolor{gray!30}Rotor & \cellcolor{gray!30} \textcolor{best}{2.629} & \cellcolor{gray!30} \textcolor{second_best}{2.717} & \cellcolor{gray!30} \textcolor{second_best}{2.818} 
& \cellcolor{gray!30}\textcolor{best}{2.307} & \cellcolor{gray!30}\textcolor{best}{2.369} & \cellcolor{gray!30}\textcolor{best}{2.515} \\
& \multirow{5}{*}{\makecell{\rotatebox{90}{C4}}}
& Original & \multicolumn{3}{c}{----------- 3.151 -----------} & \multicolumn{3}{c}{----------- 2.834 -----------}\\
& & LR1    & 3.414 & 4.071 & 5.001 & 2.884 & 2.910 & 2.985 \\
& & LR4    & 3.390 & 3.315 & 3.504 & 2.874 & 2.905 & 2.980 \\
& & BH1    & \textcolor{second_best}{3.343} & \textcolor{best}{3.262} & \textcolor{best}{3.404} & \textcolor{second_best}{2.865} & \textcolor{best}{2.897} & \textcolor{best}{2.975} \\
& & \cellcolor{gray!30}Rotor & \cellcolor{gray!30}\textcolor{best}{3.261} & \cellcolor{gray!30}\textcolor{second_best}{3.285} & \cellcolor{gray!30}\textcolor{second_best}{3.428} 
& \cellcolor{gray!30}\textcolor{best}{2.854} & \cellcolor{gray!30}\textcolor{second_best}{2.900} & \cellcolor{gray!30}\textcolor{second_best}{2.977} 
\\
& \multirow{5}{*}{\makecell{\rotatebox{90}{PTB}}}
& Original & \multicolumn{3}{c}{----------- 3.260 -----------} & \multicolumn{3}{c}{----------- 2.985 -----------} \\
& & LR1    & 3.358 & 4.684 & 6.904 & 3.046 & 3.151 & 3.225 \\
& & LR4    & \textcolor{second_best}{3.316} & 3.400 & 3.466 & 3.034 & 3.127 & \textcolor{second_best}{3.192} \\
& & BH1    & \textcolor{best}{3.293} & \textcolor{best}{3.355} & \textcolor{best}{3.395} & \textcolor{second_best}{3.025} & \textcolor{best}{3.101} & \textcolor{best}{3.168} \\
& & \cellcolor{gray!30}Rotor & \cellcolor{gray!30}3.327 & \cellcolor{gray!30}\textcolor{second_best}{3.392} & \cellcolor{gray!30}\textcolor{second_best}{3.442} 
& \cellcolor{gray!30}\textcolor{best}{3.011} & \cellcolor{gray!30}\textcolor{second_best}{3.109} & \cellcolor{gray!30}3.202 
\\
\midrule
\multirow{10}{*}{\makecell{\rotatebox{90}{Accuracy (\%)}}}
& \multirow{5}{*}{\makecell{\rotatebox{90}{\makecell{Arc \\ Challenge}}}} 
& Original & \multicolumn{3}{c}{----------- 58.37 -----------} & \multicolumn{3}{c}{----------- 66.09 -----------} \\
& & LR1    & 50.78 & 50.44 & 44.26 & 55.06 & 50.97 & 44.55 \\
& & LR4    & 53.84 & 53.39 & 45.95 & 57.48 & \textcolor{best}{54.51} & \textcolor{best}{60.77} \\
& & BH1    & \textcolor{second_best}{54.83} & \textcolor{second_best}{54.25} & \textcolor{second_best}{49.61} & \textcolor{second_best}{60.11} & 49.27 & \textcolor{second_best}{60.68} \\
& & \cellcolor{gray!30}Rotor & \cellcolor{gray!30}\textcolor{best}{55.31} & \cellcolor{gray!30}\textcolor{best}{54.50} & \cellcolor{gray!30}\textcolor{best}{49.64} 
& \cellcolor{gray!30}\textcolor{best}{61.34} & \cellcolor{gray!30}\textcolor{second_best}{52.27} & \cellcolor{gray!30}47.28 
\\
& \multirow{5}{*}{\makecell{\rotatebox{90}{Hellaswag}}} 
& Original & \multicolumn{3}{c}{----------- 41.00 -----------} & \multicolumn{3}{c}{----------- 55.00 -----------} \\
& & LR1    & 36.17 & 28.93 & 14.47 & 42.53 & 32.33 & \textcolor{best}{13.93} \\
& & LR4    & 38.02 & 33.79 & 33.87 & 44.41 & 40.53 & 11.27 \\
& & BH1    & \textcolor{second_best}{39.10} & \textcolor{best}{35.27} & \textcolor{second_best}{35.87} & \textcolor{second_best}{45.96} & \textcolor{best}{42.73} & \textcolor{second_best}{13.06} \\
& & \cellcolor{gray!30}Rotor & \cellcolor{gray!30}\textcolor{best}{39.33} & \cellcolor{gray!30}\textcolor{second_best}{34.94} & \cellcolor{gray!30}\textcolor{best}{37.52} 
& \cellcolor{gray!30}\textcolor{best}{50.20} & \cellcolor{gray!30}\textcolor{second_best}{40.60} & \cellcolor{gray!30}6.868 
\\
\bottomrule
\end{tabular}
\caption{\footnotesize{Log-PPL $(\downarrow)$ and accuracy $(\uparrow)$ using original, Low-Rank ($r=1$ or $4$), BH1, and Rotor (ours) for $1$–$3$ layer replacements. One-layer results are averaged over all layers; two/three-layer results are averaged over five random selections. Red indicates best, blue second-best per setting.}}
\label{tab:merged-performance-by-task}
%\end{table*}
\vspace{-30pt}
\end{wraptable}

\paragraph{Training protocol and architectural choices.}
To fit each substitute layer (rotor, LR, or BH), we extract hidden states from the pre-trained model and minimize MSE between the projected outputs of the original and approximated layers. Each variant is trained independently using the Adam optimizer \citep{kingma2017adammethodstochasticoptimization}. In our rotor architecture, \textit{depth} refers to the number of stacked rotor maps $\psi$, while \textit{width} denotes the number of parallel rotor maps within each layer. For example, the rotor map in Fig. \ref{fig:rotor_map} has both width and depth equal to $1$ (i.e., one rotor map). We also insert \textit{fixed permutations} between rotors to enable grade mixing and add \textit{normalization} layers to stabilize training; both are parameter-free. All architectural details and hyperparameters are provided in Appendix~\ref{apdx:hyper-parameters}.

\begin{table*}[t]
\scriptsize
\centering
\setlength{\tabcolsep}{1.8pt}
\renewcommand{\arraystretch}{0.95}
\begin{tabular}{
>{\centering\arraybackslash}m{1.5cm}
>{\centering\arraybackslash}m{1.0cm}
*{15}{>{\centering\arraybackslash}m{0.616cm}<{\fontsize{4}{2.5}\selectfont}}
}
\toprule
\textbf{Dataset} & \textbf{Method} & \multicolumn{15}{c}{\textbf{One layer replaced (Layer index)}} \\
& & \textbf{1} & \textbf{2} & \textbf{3} & \textbf{4} & \textbf{5} & \textbf{6} & \textbf{7}
  & \textbf{8} & \textbf{9} & \textbf{10} & \textbf{11} & \textbf{12} & \textbf{13} & \textbf{14} & \textbf{15} \\
\midrule
\multirow{4}{*}{Wikitext2 $(\downarrow)$}
& LR1   & 3.620 & 2.750 & 2.754 & 2.781 & 2.742 & 2.705 & 2.703 & 2.714 & 2.695 & 2.622 & 2.632 & 2.628 & 2.612 & 2.630 & 2.647 \\
& LR4   & 3.851 & 2.723 & 2.752 & 2.673 & 2.673 & 2.650 & 2.671 & 2.667 & 2.674 & 2.620 & 2.628 & 2.614 & 2.602 & 2.620 & 2.634 \\
& BH1   & 3.615 & 2.686 & 2.675 & 2.645 & 2.657 & 2.637 & 2.645 & 2.653 & 2.647 & 2.612 & 2.617 & 2.612 & 2.592 & 2.606 & 2.614 \\
& \cellcolor{gray!30}Rotor 
        & \cellcolor{gray!30}2.924 & \cellcolor{gray!30}2.665 & \cellcolor{gray!30}2.664
        & \cellcolor{gray!30}2.645 & \cellcolor{gray!30}2.664 & \cellcolor{gray!30}2.635 & \cellcolor{gray!30}2.642
        & \cellcolor{gray!30}2.640 & \cellcolor{gray!30}2.640 & \cellcolor{gray!30}2.607 & \cellcolor{gray!30}2.616
        & \cellcolor{gray!30}2.613 & \cellcolor{gray!30}2.593 & \cellcolor{gray!30}2.611 & \cellcolor{gray!30}2.566 \\
\noalign{\vskip 1.5pt}
\multirow{4}{*}{C4 $(\downarrow)$}
& LR1   & 4.433 & 3.297 & 3.440 & 3.274 & 3.300 & 3.276 & 3.309 & 3.282 & 3.292 & 3.205 & 3.284 & 3.271 & 3.192 & 3.236 & 3.214 \\
& LR4   & 4.432 & 3.292 & 3.404 & 3.250 & 3.266 & 3.250 & 3.264 & 3.260 & 3.281 & 3.203 & 3.204 & 3.191 & 3.187 & 3.219 & 3.208 \\
& BH1   & 4.293 & 3.288 & 3.276 & 3.212 & 3.230 & 3.215 & 3.241 & 3.230 & 3.276 & 3.196 & 3.193 & 3.184 & 3.174 & 3.194 & 3.194 \\
& \cellcolor{gray!30}Rotor 
        & \cellcolor{gray!30}3.660 & \cellcolor{gray!30}3.258 & \cellcolor{gray!30}3.292
        & \cellcolor{gray!30}3.232 & \cellcolor{gray!30}3.242 & \cellcolor{gray!30}3.228 & \cellcolor{gray!30}3.249
        & \cellcolor{gray!30}3.245 & \cellcolor{gray!30}3.248 & \cellcolor{gray!30}3.197 & \cellcolor{gray!30}3.196
        & \cellcolor{gray!30}3.187 & \cellcolor{gray!30}3.176 & \cellcolor{gray!30}3.202 & \cellcolor{gray!30}3.203 \\
\noalign{\vskip 1.5pt}
\multirow{4}{*}{PTB $(\downarrow)$}
& LR1   & 5.401 & 3.468 & 3.435 & 3.406 & 3.419 & 3.346 & 3.358 & 3.401 & 3.363 & 3.322 & 3.281 & 3.308 & 3.271 & 3.322 & 3.292 \\
& LR4   & 5.183 & 3.394 & 3.395 & 3.316 & 3.347 & 3.304 & 3.315 & 3.334 & 3.328 & 3.281 & 3.276 & 3.265 & 3.264 & 3.308 & 3.287 \\
& BH1   & 4.835 & 3.352 & 3.336 & 3.293 & 3.324 & 3.288 & 3.292 & 3.307 & 3.302 & 3.273 & 3.266 & 3.268 & 3.255 & 3.265 & 3.270 \\
& \cellcolor{gray!30}Rotor 
        & \cellcolor{gray!30}4.194 & \cellcolor{gray!30}3.412 & \cellcolor{gray!30}3.403
        & \cellcolor{gray!30}3.356 & \cellcolor{gray!30}3.369 & \cellcolor{gray!30}3.320 & \cellcolor{gray!30}3.338
        & \cellcolor{gray!30}3.336 & \cellcolor{gray!30}3.326 & \cellcolor{gray!30}3.278 & \cellcolor{gray!30}3.271
        & \cellcolor{gray!30}3.300 & \cellcolor{gray!30}3.266 & \cellcolor{gray!30}3.300 & \cellcolor{gray!30}3.284 \\
\midrule
\multirow{4}{*}{\makecell{Arc \\ Challenge $(\uparrow)$}}
& LR1   & 50.64 & 53.22 & 46.78 & 46.35 & 45.49 & 51.07 & 52.79 & 51.93 & 33.05 & 50.21 & 56.65 & 55.80 & 56.65 & 55.08 & 55.79 \\
& LR4   & 50.64 & 53.22 & 50.64 & 51.93 & 54.51 & 55.79 & 54.08 & 51.07 & 46.35 & 51.07 & 58.80 & 59.23 & 57.51 & 56.22 & 57.51 \\
& BH1   & 53.22 & 54.51 & 52.79 & 54.94 & 57.08 & 54.08 & 53.65 & 52.36 & 51.93 & 50.21 & 57.94 & 57.94 & 57.08 & 58.80 & 57.51 \\
& \cellcolor{gray!30}Rotor 
        & \cellcolor{gray!30}54.51 & \cellcolor{gray!30}55.36 & \cellcolor{gray!30}53.65
        & \cellcolor{gray!30}55.36 & \cellcolor{gray!30}54.51 & \cellcolor{gray!30}55.79 & \cellcolor{gray!30}54.51
        & \cellcolor{gray!30}53.22 & \cellcolor{gray!30}52.36 & \cellcolor{gray!30}50.64 & \cellcolor{gray!30}58.37
        & \cellcolor{gray!30}60.09 & \cellcolor{gray!30}56.65 & \cellcolor{gray!30}57.94 & \cellcolor{gray!30}57.08 \\
\noalign{\vskip 1.5pt}
\multirow{4}{*}{HellaSwag $(\uparrow)$}
& LR1   & 29.00 & 32.00 & 34.33 & 39.67 & 34.00 & 33.00 & 34.00 & 33.00 & 34.67 & 37.00 & 41.00 & 38.33 & 40.33 & 40.00 & 39.67 \\
& LR4   & 32.67 & 39.67 & 34.67 & 39.00 & 37.67 & 37.00 & 35.00 & 38.67 & 33.67 & 37.67 & 41.33 & 40.00 & 40.33 & 40.33 & 40.00 \\
& BH1   & 37.33 & 37.67 & 37.67 & 41.00 & 37.00 & 40.00 & 40.33 & 37.00 & 36.33 & 36.33 & 42.67 & 40.00 & 41.33 & 42.00 & 41.33 \\
& \cellcolor{gray!30}Rotor 
        & \cellcolor{gray!30}40.00 & \cellcolor{gray!30}39.67 & \cellcolor{gray!30}35.67
        & \cellcolor{gray!30}42.33 & \cellcolor{gray!30}41.00 & \cellcolor{gray!30}38.67 & \cellcolor{gray!30}36.67
        & \cellcolor{gray!30}38.67 & \cellcolor{gray!30}36.00 & \cellcolor{gray!30}36.67 & \cellcolor{gray!30}42.00
        & \cellcolor{gray!30}39.67 & \cellcolor{gray!30}40.33 & \cellcolor{gray!30}42.00 & \cellcolor{gray!30}40.67 \\
\bottomrule
\end{tabular}

\vspace{0.3cm}

\begin{tabular}{
>{\centering\arraybackslash}m{1.5cm}
>{\centering\arraybackslash}m{1.0cm}
*{15}{>{\centering\arraybackslash}m{0.618cm}<{\fontsize{4}{2.5}\selectfont}}
}
\toprule
\textbf{Dataset} & \textbf{Method} & \multicolumn{15}{c}{\textbf{Two layers replaced (Layer pairs)}} \\
& &
\textbf{10,11} & \textbf{10,12} & \textbf{10,13} & \textbf{10,14} & \textbf{10,15} &
\textbf{11,12} & \textbf{11,13} & \textbf{11,14} & \textbf{11,15} &
\textbf{12,13} & \textbf{12,14} & \textbf{12,15} &
\textbf{13,14} & \textbf{13,15} & \textbf{14,15} \\
\midrule
\multirow{4}{*}{Wikitext2 $(\downarrow)$}
& LR1   & 2.724 & 2.702 & 2.695 & 2.697 & 2.718 & 2.757 & 2.731 & 2.756 & 2.757 & 2.768 & 2.723 & 2.741 & 2.729 & 2.863 & 2.862 \\
& LR4   & 2.703 & 2.674 & 2.662 & 2.669 & 2.694 & 2.691 & 2.662 & 2.669 & 2.694 & 2.660 & 2.665 & 2.676 & 2.686 & 2.708 & 2.758 \\
& BH1   & 2.677 & 2.660 & 2.645 & 2.650 & 2.656 & 2.670 & 2.643 & 2.662 & 2.660 & 2.645 & 2.656 & 2.657 & 2.655 & 2.680 & 2.700 \\
& \cellcolor{gray!30}Rotor 
        & \cellcolor{gray!30}2.679 & \cellcolor{gray!30}2.670 & \cellcolor{gray!30}2.654 & \cellcolor{gray!30}2.660 & \cellcolor{gray!30}2.676 
        & \cellcolor{gray!30}2.662 & \cellcolor{gray!30}2.652 & \cellcolor{gray!30}2.667 & \cellcolor{gray!30}2.677 
        & \cellcolor{gray!30}2.667 & \cellcolor{gray!30}2.667 & \cellcolor{gray!30}2.675 & \cellcolor{gray!30}2.662 
        & \cellcolor{gray!30}2.688 & \cellcolor{gray!30}2.745 \\
\noalign{\vskip 1.5pt}
\multirow{4}{*}{C4 $(\downarrow)$}
& LR1   & 3.298 & 3.273 & 3.258 & 3.296 & 3.293 & 3.315 & 3.276 & 3.315 & 3.303 & 3.263 & 3.300 & 3.281 & 3.285 & 3.278 & 3.386 \\
& LR4   & 3.290 & 3.258 & 3.249 & 3.279 & 3.269 & 3.265 & 3.255 & 3.287 & 3.271 & 3.241 & 3.278 & 3.246 & 3.281 & 3.268 & 3.353 \\
& BH1   & 3.281 & 3.241 & 3.228 & 3.246 & 3.244 & 3.242 & 3.227 & 3.246 & 3.244 & 3.216 & 3.237 & 3.226 & 3.234 & 3.232 & 3.296 \\
& \cellcolor{gray!30}Rotor
        & \cellcolor{gray!30}3.267 & \cellcolor{gray!30}3.246 & \cellcolor{gray!30}3.231 & \cellcolor{gray!30}3.256 & \cellcolor{gray!30}3.254
        & \cellcolor{gray!30}3.248 & \cellcolor{gray!30}3.232 & \cellcolor{gray!30}3.260 & \cellcolor{gray!30}3.257
        & \cellcolor{gray!30}3.223 & \cellcolor{gray!30}3.252 & \cellcolor{gray!30}3.241
        & \cellcolor{gray!30}3.249 & \cellcolor{gray!30}3.247 & \cellcolor{gray!30}3.327 \\
\noalign{\vskip 1.5pt}
\multirow{4}{*}{PTB $(\downarrow)$}
& LR1   & 3.385 & 3.386 & 3.367 & 3.410 & 3.402 & 3.418 & 3.370 & 3.418 & 3.367 & 3.486 & 3.423 & 3.371 & 3.431 & 3.455 & 3.619 \\
& LR4   & 3.345 & 3.319 & 3.320 & 3.348 & 3.336 & 3.326 & 3.315 & 3.353 & 3.332 & 3.300 & 3.317 & 3.313 & 3.367 & 3.335 & 3.419 \\
& BH1   & 3.315 & 3.307 & 3.297 & 3.304 & 3.310 & 3.309 & 3.287 & 3.332 & 3.301 & 3.281 & 3.293 & 3.298 & 3.294 & 3.300 & 3.345 \\
& \cellcolor{gray!30}Rotor
        & \cellcolor{gray!30}3.336 & \cellcolor{gray!30}3.343 & \cellcolor{gray!30}3.311 & \cellcolor{gray!30}3.349 & \cellcolor{gray!30}3.334
        & \cellcolor{gray!30}3.339 & \cellcolor{gray!30}3.303 & \cellcolor{gray!30}3.342 & \cellcolor{gray!30}3.299
        & \cellcolor{gray!30}3.311 & \cellcolor{gray!30}3.359 & \cellcolor{gray!30}3.336
        & \cellcolor{gray!30}3.352 & \cellcolor{gray!30}3.342 & \cellcolor{gray!30}3.441 \\
\midrule
\multirow{4}{*}{\makecell{Arc \\ Challenge $(\uparrow)$}}
& LR1   & 43.06 & 46.07 & 46.50 & 48.21 & 43.92 & 46.93 & 55.51 & 57.23 & 55.94 & 52.94 & 54.65 & 54.65 & 54.22 & 55.51 & 54.65 \\
& LR4   & 42.92 & 48.50 & 48.07 & 51.07 & 49.79 & 59.23 & 56.65 & 59.23 & 57.94 & 58.80 & 56.65 & 57.51 & 56.22 & 55.36 & 57.94 \\
& BH1   & 43.35 & 48.50 & 51.07 & 50.21 & 50.64 & 59.66 & 58.37 & 58.80 & 56.65 & 57.94 & 59.23 & 57.51 & 57.94 & 55.65 &  58.37 \\
& \cellcolor{gray!30}Rotor 
        & \cellcolor{gray!30}45.49 & \cellcolor{gray!30}48.93 & \cellcolor{gray!30}53.65 & \cellcolor{gray!30}52.79 & \cellcolor{gray!30}51.50
        & \cellcolor{gray!30}57.94 & \cellcolor{gray!30}57.94 & 
        \cellcolor{gray!30}58.80 &
        \cellcolor{gray!30}56.65 & \cellcolor{gray!30}59.66
        & \cellcolor{gray!30}56.65 & \cellcolor{gray!30}57.08 & \cellcolor{gray!30}57.03 & \cellcolor{gray!30}55.79
        & \cellcolor{gray!30}57.05 \\
\noalign{\vskip 1.5pt}
\multirow{4}{*}{HellaSwag $(\uparrow)$}
& LR1   & 32.00 & 35.00 & 36.00 & 40.33 & 34.00 & 37.67 & 41.67 & 39.33 & 41.33 & 34.00 & 40.33 & 39.33 & 40.33 & 43.00 & 40.67 \\
& LR4   & 32.33 & 35.33 & 38.33 & 38.33 & 36.67 & 39.00 & 41.67 & 41.00 & 42.33 & 40.00 & 38.00 & 39.00 & 40.00 & 39.00 & 37.67 \\
& BH1   & 37.33 & 37.67 & 35.67 & 40.33 & 39.33 & 39.33 & 41.00 & 41.67 & 41.33 & 40.33 & 43.67 & 40.67 & 42.67 & 40.67 & 40.67 \\
& \cellcolor{gray!30}Rotor 
        & \cellcolor{gray!30}36.00 & \cellcolor{gray!30}37.67 & \cellcolor{gray!30}38.67 & \cellcolor{gray!30}38.67 & \cellcolor{gray!30}37.67 
        & \cellcolor{gray!30}39.33 & \cellcolor{gray!30}41.67 & \cellcolor{gray!30}41.33 & \cellcolor{gray!30}41.33
        & \cellcolor{gray!30}39.33 & \cellcolor{gray!30}41.67 & \cellcolor{gray!30}39.00 
        & \cellcolor{gray!30}41.00 & \cellcolor{gray!30}40.33 & \cellcolor{gray!30}40.33 \\
\bottomrule
\end{tabular}
\caption{\small{Performance on log-PPL (↓) and accuracy (↑) when replacing \textbf{one attention layer (top)} for layer indices $1$–$15$ and \textbf{two attention layers (bottom)} for pairs of indices from $10$–$15$ of \texttt{LLaMa-3.2 1B}. Methods are Low-Rank ($r=1$ and $4$), BH1, and Rotor.}}
\label{tab:two-layer-replace}
\end{table*}

\vspace{-1pt}

\paragraph{Models, datasets, and baselines.}
We evaluate on two pre-trained LLMs: \texttt{LLaMa-3.2 1B} \citep{touvron2023llamaopenefficientfoundation} and \texttt{Qwen-2.5 1.5B} \citep{qwen2025qwen25technicalreport}. Metrics include log perplexity $(\downarrow)$ on three language modeling datasets---\texttt{Wikitext2}, \texttt{C4} \citep{dodge2021documentinglargewebtextcorpora}, and \texttt{PTB} \citep{marcus1993building}---and accuracy $(\uparrow)$ on two multiple-choice benchmarks---\texttt{Arc Challenge} \citep{clark2018thinksolvedquestionanswering} and \texttt{HellaSwag} \citep{zellers2019hellaswagmachinereallyfinish}. We compare rotors to:
\begin{inparaenum}[\bfseries (a)]
    \item \textit{LR1 and LR4}: Low-rank projections with rank $r = 1$ or $4$, where a dense matrix $W \in \mathbb{R}^{d_{\text{out}} \times d_{\text{in}}}$ is approximated as $XY$, with $X \in \mathbb{R}^{d_{\text{out}} \times r}$ and $Y \in \mathbb{R}^{r \times d_{\text{in}}}$.
    \item \textit{BH1}: Block-Hadamard \citep{zeng2023lookupffn} of depth $1$, which approximate $W$ by $BH$, where $H$ is a fixed Hadamard matrix and $B$ is block-diagonal and learnable.
\end{inparaenum}
Parameter counts for all the methods are detailed in Tab.~\ref{tab:proj-param-counts}; reference LLM performance (with dense matrices) are shown in Tab.~\ref{tab:merged-performance-by-task}. Additional experiments are provided in Appendix~\ref{apdx:additional_exp}.

% \vspace{-10pt}
\subsection{Results and Discussion}

\paragraph{Parameter efficiency.}
As shown in Tab.~\ref{tab:proj-param-counts}, rotors require significantly fewer parameters than both dense and approximate baselines. For example, in \texttt{LLaMa-3.2 1B}, the query projection uses over $4.19$M parameters in its dense form, compared to $16.4$K in LR4, $32.7$K in BH1, and just \(\leq 896\) in rotors---\textit{$4700\times$ reduction} over dense and $18\times$ over LR4. Similar savings apply across key and value layers, directly supporting \textbf{G-2}.

{\paragraph{Compute cost.}
Rotor layers are currently slower in wall-clock runtime. Depending on rotor width $(w)$ and depth $(d)$, inference is roughly $w\times d$ times slower than dense layers in our experiments. In terms of FLOPs, rotor layers require $35.8$M $\times\,(w d)$ and $82.4$M $\times \, (w d)$ on \texttt{LLaMa-3.2 1B} and \texttt{LLaMa-3.2 3B}, compared to $206.1$M and $515.3$M for dense layers—owing to their block-diagonal structure. These reductions arise naturally from the geometric formulation rather than from explicit sparsity constraints. During training, gradient storage and backpropagation through the decomposition (Fig.~\ref{fig:process}) add minor additional cost. Optimized kernels and hardware-aware implementations could exploit this structure much better, as well as offer memory-bandwidth benefits discussed in Sec~\ref{sec:conclusion_furture_work}.

\paragraph{Rotor performance vs. baselines.}
Despite their compact size, rotors match or outperform LR1, LR4, and BH1 across most datasets and models. In Tab.~\ref{tab:merged-performance-by-task}, on \texttt{Wikitext2}, rotors achieve $2.629$ log-PPL in \texttt{LLaMa-3.2 1B} (vs.\ $2.636$ for second-best BH1), and $2.307$ in \texttt{Qwen-2.5 1.5B} (vs.\ $2.323$ for second-best BH1) when a single attention layer is replaced. On \texttt{Arc Challenge}, rotor layers yield $55.31$\% accuracy in \texttt{LLaMa-3.2 1B} (vs.\ $54.83$\% for BH1) and $61.34$\% in \texttt{Qwen-2.5 1.5B} (vs.\ $57.48$\% for LR4). Rotor projections are consistently either the best or second-best across most settings.

Tab.~\ref{tab:two-layer-replace} confirms rotor robustness across layer combinations. For example, replacing layers $12$ and $13$ in \texttt{Qwen-2.5 1.5B} yields $59.66$\% accuracy on \texttt{Arc Challenge}, outperforming LR4 ($58.80$\%) and BH1 ($57.94$\%). In contrast, LR1, despite having $2$--$5\times$ more parameters than rotors, shows significant performance drops (e.g., log-PPL on \texttt{Wikitext2} for \texttt{LLaMa-3.2 1B}). These results confirm \textbf{G-1} and show that rotors match  downstream performance of other baselines only using far fewer parameters.

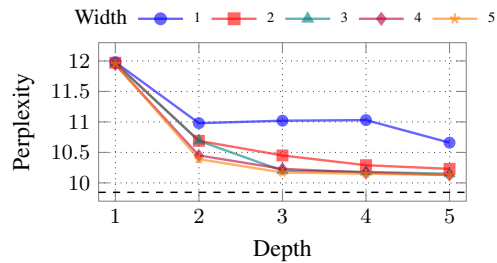
\begin{wrapfigure}[13]{r}{0.48\textwidth}
\vspace{-20pt}
% \raggedright
% \centering
\begin{center}
\scalebox{0.9}{
\begin{tikzpicture}
  \begin{axis}[
    width=0.5\textwidth,
    height=0.28\textwidth,
    xlabel={Depth},
    ylabel={Perplexity},
    xmin=0.8, xmax=5.2,
    ymin=9.7, ymax=12.3,
    xtick={1,2,3,4,5},
    ytick={10,10.5,11,11.5,12},
    grid=both,
    grid style={line width=0.3pt, black, dotted},
    tick style={line width=0.3pt, gray},
    tick label style={font=\footnotesize},
    label style={font=\footnotesize},
    axis lines=box,
    axis line style={line width=0.3pt, gray},
    mark size=2pt,
    legend style={
      at={(0.5,1.03)},
      anchor=south,
      font=\tiny,
      draw=none,
      fill=none,
      legend columns=6,
      column sep=2pt,
    },
    legend image post style={scale=1.1},
  ]

    % Convergence line (no legend)
    \addplot[dashed, black, line width=0.6pt, forget plot] coordinates {(0.8,9.845) (5.2,9.845)};

    % Data lines
    \addlegendimage{empty legend}
    \addlegendentry{\footnotesize{Width}}
    \addplot[blue, mark=*, line width=1pt, opacity=0.6] coordinates {(1,11.98) (2,10.98) (3,11.02) (4,11.03) (5,10.66)}; \addlegendentry{1}
    \addplot[red, mark=square*, line width=1pt, opacity=0.6] coordinates {(1,11.97) (2,10.69) (3,10.45) (4,10.29) (5,10.23)}; \addlegendentry{2}
    \addplot[teal, mark=triangle*, line width=1pt, opacity=0.6] coordinates {(1,11.96) (2,10.69) (3,10.20) (4,10.18) (5,10.15)}; \addlegendentry{3}
    \addplot[purple, mark=diamond*, line width=1pt, opacity=0.6] coordinates {(1,11.94) (2,10.45) (3,10.23) (4,10.17) (5,10.13)}; \addlegendentry{4}
    \addplot[orange, mark=star, line width=1pt, opacity=0.6] coordinates {(1,11.95) (2,10.39) (3,10.17) (4,10.15) (5,10.13)}; \addlegendentry{5}
  \end{axis}
\end{tikzpicture}
}
\label{fig:width_depth}
\end{center}
\vspace{-18pt}
\captionof{figure}{\small{Effect of \textbf{rotor width and depth}. Replacing Layer-13 in \texttt{Qwen-2.5 1.5B} with rotors of varying depth and width. The dashed line (9.845) indicates convergence to the base model’s perplexity.}}
\label{fig:width_depth}
\vspace{40pt}
\end{wrapfigure}

In Fig. \ref{fig:width_depth}, PPL consistently decreases with both increasing rotor \textit{width} and \textit{depth}. The strongest improvements occur from depth $1$ to $2$, after which gains taper off. This trend suggests that both stacking more layers (depth) and using more parallel rotor maps (width) contribute to better approximation of linear layers---showing \textbf{G-3}.

\begin{table}[t]
    \centering
    \setlength{\tabcolsep}{3.5pt}
    \renewcommand{\arraystretch}{0.95}
    \begin{tabular}{lcccccccccc}
    \toprule
    \#Epoch & 1 & 2 & 3 & 4 & 5 & 6 & 7 & 8 & 9 & 10 \\
    \midrule
    Dense & 85.05 & 86.23 & 87.17 & 88.07 & 87.90 & 88.54 & 89.07 & 89.36 & 89.53 & \textbf{89.67} \\
    Rotor & 80.80 & 82.05 & 82.60 & 82.38 & 86.14 & 86.75 & 86.74 & 86.52 & 87.94 & \textbf{88.36} \\
    \bottomrule
    \end{tabular}
    \vspace{4pt}
    \caption{Accuracy (\%) on \texttt{FMNIST} over 10 epochs for dense and rotor-based MLPs.}
    \label{tab:fmnist-acc}
\end{table}

\paragraph{Results of end-to-end training on \texttt{FMNIST}.} \label{sec:fmnist-results}
As shown in Tab.~\ref{tab:fmnist-acc}, the dense network improves faster in early epochs, while the rotor-based model learns more gradually but reaches a comparable final accuracy ($89.67$\% vs.\ $88.36$\%). This suggests that end-to-end optimization with rotors is slightly slower---potentially due to the high compression---but ultimately competitive with dense layers.

Overall, the LLM results support our claim: rotor layers---constructed from bivector primitives---offer insight into how linear layers can be synthesized from a compact set of building blocks. Moreover, the \texttt{FMNIST} experiment shows that rotor layers are feasible when trained jointly with the rest of the network end-to-end from scratch. Our full codebase, including datasets and hyperparameters for all experiments, is available at \url{https://github.com/vsingh-group/ComposingLinearLayers}. Key architectural and training details are also provided in Appendix~\ref{apdx:hyper-parameters}.

% Furthermore, \textit{adding fixed permutations} between layers further help by encouraging grade mixing (e.g., interactions between different grade components within input multivector), which enhances expressivity.

% \input{figures/width-depth}
\section{Related Work}
\label{sec:related_work}

\paragraph{Compositions in machine learning.}
Building complex model behavior from composition of simpler compute primitives is an active research topic. Capsule networks \citep{sabour2017dynamicroutingcapsules} modeled part-whole relations, while mixture-of-experts \citep{riquelme2021scaling} and model merging \citep{lahner2024direct} compose specialized sub-modules in a conditional manner. Recent studies show how internal circuits in LLMs organize around functional units \citep{weiss2021thinkingliketransformers}. In NLP, 
%compositional approaches also have roots in linguistic theory. T
Tree-structured models \citep{tai2015improved} have been used to encode syntactic composition, while works on compositional semantics \citep{mitchell2008vector} investigate how word meanings compose to form sentence meanings. %Compositional generalization, i.e., t
The ability to understand new combinations of familiar components remains a challenge \citep{zhou2024compositional,lake2018generalization} and a few strategies are being studied \citep{li2022neural, chen2020compositional, chytas2024understanding}. Other directions include 
%compositional program synthesis for language understanding \citep{raza2015compositional} and 
modular approaches \citep{das2018neural,pfeiffer2023modular}. %There is also an extensive body of work in computer vision on visual understanding via parts-based models \cite{felzenszwalb2009object} and these ideas are also being investigated in visual generation now \cite{liu2022compositional}. Our work is different in that we focus on composition at the level of low-level geometric primitives, in terms of its algebraic structure.
%

% \begin{table*}[t]
\begin{wraptable}[11]{r}{0.4\textwidth}
\vspace{-15pt}
% \scriptsize
% \centering
\centering
\footnotesize
\setlength{\tabcolsep}{2.5pt}
\renewcommand{\arraystretch}{1.1}
\begin{tabular}{lccc}
  \toprule
  \textbf{Method} & \textbf{Key} & \textbf{Query} & \textbf{Value} \\
  \midrule
  Dense & 1048576 & 4194304 & 1048576 \\
  LR1   & 2560 & 4096 & 2560 \\
  LR4   & 10240 & 16384 & 10240 \\
  BH1   & 8192 & 32768 & 8192 \\
  \cellcolor{gray!30}Rotor 
    & \cellcolor{gray!30}$\le 1080$ & \cellcolor{gray!30}$\le 896$ & \cellcolor{gray!30}$\le 1080$ \\
  \bottomrule
\end{tabular}
\caption{\small{Summary of \textbf{the number of parameters} used for key, query, and value projections in a single attention layer of \texttt{LLaMa-3.2 1B}.}}
\label{tab:proj-param-counts}
% \end{table*}
\vspace{10pt}
\end{wraptable}

\paragraph{Approximating linear layers.}
To reduce the cost of dense layers, various structured approximations are common. Low-rank factorization such as LoRA \citep{hu2021loralowrankadaptationlarge} constrains weights to rank-$r$ matrices. Other matrices such as circulant \citep{yu2014circulant,cheng2015exploration}, Toeplitz \citep{sindhwani2015structured}, Walsh-Hadamard \citep{alam2024walsh}, and Block-Hadamard \citep{zeng2023lookupffn}, among others, have been shown to be resource efficient.
%for efficiency gains and low-resource deployments. 
%Instead of enforcing the external criteria that a structured transform must approximate the linear layer, we instead seek to 
Our goal is to understand the compositional structure of the linear transformation itself. Resource efficiency is the {\em result}, not the key motivation. 
\paragraph{Clifford/Geometric algebra in ML.}
Recent works have explored Clifford algebra for encoding geometric structure in ML models. GCANs \citep{ruhe2023geometriccliffordalgebranetworks} and CGENNs \citep{ruhe2023cliffordgroupequivariantneural} construct equivariant layers by combining GA-based transformations. Clifford neural layers \citep{brandstetter2023cliffordneurallayerspde} was applied to physical modeling tasks. GA has been connected to convex training of ReLU networks \citep{pilanci2024complexityclarityanalyticalexpressions} and randomized methods for  multivector representations \citep{wang2024randomizedgeometricalgebramethods} are available. GA has also been used for time-series models \citep{chen2025simpletm}. 
%Our method contributes a novel use of Clifford algebra: not as a replacement of the full network, but as a principled and differentiable way to decompose standard linear layers into a compact set of irreducible/primitives.
\section{Conclusions and Future Work}
\label{sec:conclusion_furture_work}
We show that the functionality of standard linear layers can be expressed with exponentially fewer parameters, $\mathcal{O}\left(\log^2 d\right)$ versus $\mathcal{O}\left(d^2\right)$, while maintaining competitive performance when applied to attention mechanisms in modern LLMs. In particular, our experiments on 1--1.5B parameter models demonstrate that rotor-based modules can effectively reproduce the behavior of dense linear layers when trained in isolation, while our end-to-end experiment on \texttt{FMNIST} shows their feasibility when trained jointly with the rest of the network from scratch. The underlying algebraic structure reveals a rich compositional hierarchy, where geometric primitives combine through rotor operations to form expressive yet compact transformations. This is achieved by mapping bivector parameters to their corresponding rotor operations. These insights open several promising directions, including the development of interpretable architectures and parameter-efficient models that leverage this compositional structure. Beyond that, our framework connects to statistical models involving interaction decomposition, such as ANOVA \citep{st1989analysis}, multivariate analysis \citep{mardia2024multivariate}, and mechanisms to approximate Shapley features \citep{chen2023harsanyinet}. For example, modeling high-order dependencies in statistical learning often suffers from exponential parameter growth. While ideas such as Tucker tensor decomposition \citep{li2018tucker} factorize the full coefficient tensor, our framework offers an alternative by mapping $k$ predictor variables to a multivector $X$ and modeling the response variable via a rotor transformation as $\hat{y} = \langle \exp(a) X \exp(b)^\dagger \rangle_0$ learned from $\mathcal{O}(k^2)$, where ($a, b$) are learned bivectors from our algorithms. This constructs interaction spaces through the composition of 2-way primitives.

\paragraph{Practical considerations.} Our current implementation remains computationally demanding, since it does not yet exploit the inherent sparsity or algebraic structure of rotor decompositions (see Appendix~\ref{apdx:limitations} for a discussion of limitations). In parallel, a promising direction is to leverage these ideas to mitigate memory bandwidth bottlenecks that dominate large-model inference. As noted in \citet{davies2025efficientllminferencebandwidth}, (i) memory capacity requirements for frontier LLMs exceed hundreds of gigabytes, and (ii) memory bandwidth, rather than compute, is the primary bottleneck for inference throughput. Our approach directly addresses (ii): instead of loading millions of dense-layer weights from memory, rotor layers can, in principle, synthesize them on-chip from a small set of geometric parameters, effectively trading compute for memory bandwidth. Future work will focus on developing sparse and hardware-aware Clifford kernels to realize these gains in practice and on extending the framework to full-model training and dynamic rotor composition for scalable deployment.

\newpage
\section*{Acknowledgments}
We thank the anonymous reviewers for their constructive feedback. We are very grateful to Prof.~Karthikeyan Sankaralingam for numerous valuable discussions and working out the link to \cite{davies2025efficientllminferencebandwidth}. TP, DY and VS were all partly supported by NIH R01AG092220.

\bibliographystyle{plainnat}
\bibliography{refs}

\newpage
\section*{NeurIPS Paper Checklist}

\begin{enumerate}

\item {\bf Claims}
    \item[] Question: Do the main claims made in the abstract and introduction accurately reflect the paper's contributions and scope?
    \item[] Answer: \answerYes{}
    \item[] Justification: The main claim made in the abstract is that we can express a linear layer in $\mathcal{O}(\log^2 d)$ parameters with algebraic primitives. The framework is detailed in Sections $\ref{sec:method}$ and $\ref{sec:how_we_do_it}$ with supporting experimental results in Section $\ref{sec:experiment}$ and Appendix D.
    \item[] Guidelines:
    \begin{itemize}
        \item The answer NA means that the abstract and introduction do not include the claims made in the paper.
        \item The abstract and/or introduction should clearly state the claims made, including the contributions made in the paper and important assumptions and limitations. A No or NA answer to this question will not be perceived well by the reviewers. 
        \item The claims made should match theoretical and experimental results, and reflect how much the results can be expected to generalize to other settings. 
        \item It is fine to include aspirational goals as motivation as long as it is clear that these goals are not attained by the paper. 
    \end{itemize}

\item {\bf Limitations}
    \item[] Question: Does the paper discuss the limitations of the work performed by the authors?
    \item[] Answer: \answerYes{}
    \item[] Justification: The limitations are discussed in Appendix A.
    \item[] Guidelines:
    \begin{itemize}
        \item The answer NA means that the paper has no limitation while the answer No means that the paper has limitations, but those are not discussed in the paper. 
        \item The authors are encouraged to create a separate "Limitations" section in their paper.
        \item The paper should point out any strong assumptions and how robust the results are to violations of these assumptions (e.g., independence assumptions, noiseless settings, model well-specification, asymptotic approximations only holding locally). The authors should reflect on how these assumptions might be violated in practice and what the implications would be.
        \item The authors should reflect on the scope of the claims made, e.g., if the approach was only tested on a few datasets or with a few runs. In general, empirical results often depend on implicit assumptions, which should be articulated.
        \item The authors should reflect on the factors that influence the performance of the approach. For example, a facial recognition algorithm may perform poorly when image resolution is low or images are taken in low lighting. Or a speech-to-text system might not be used reliably to provide closed captions for online lectures because it fails to handle technical jargon.
        \item The authors should discuss the computational efficiency of the proposed algorithms and how they scale with dataset size.
        \item If applicable, the authors should discuss possible limitations of their approach to address problems of privacy and fairness.
        \item While the authors might fear that complete honesty about limitations might be used by reviewers as grounds for rejection, a worse outcome might be that reviewers discover limitations that aren't acknowledged in the paper. The authors should use their best judgment and recognize that individual actions in favor of transparency play an important role in developing norms that preserve the integrity of the community. Reviewers will be specifically instructed to not penalize honesty concerning limitations.
    \end{itemize}

\item {\bf Theory assumptions and proofs}
    \item[] Question: For each theoretical result, does the paper provide the full set of assumptions and a complete (and correct) proof?
    \item[] Answer: \answerYes{}
    \item[] Justification: The proofs of Alg. $\ref{algo:inv_decomp}$ and $\ref{algo:ga_poweriter}$ are in Appendix B.
    \item[] Guidelines:
    \begin{itemize}
        \item The answer NA means that the paper does not include theoretical results. 
        \item All the theorems, formulas, and proofs in the paper should be numbered and cross-referenced.
        \item All assumptions should be clearly stated or referenced in the statement of any theorems.
        \item The proofs can either appear in the main paper or the supplemental material, but if they appear in the supplemental material, the authors are encouraged to provide a short proof sketch to provide intuition. 
        \item Inversely, any informal proof provided in the core of the paper should be complemented by formal proofs provided in appendix or supplemental material.
        \item Theorems and Lemmas that the proof relies upon should be properly referenced. 
    \end{itemize}

    \item {\bf Experimental result reproducibility}
    \item[] Question: Does the paper fully disclose all the information needed to reproduce the main experimental results of the paper to the extent that it affects the main claims and/or conclusions of the paper (regardless of whether the code and data are provided or not)?
    \item[] Answer: \answerYes{}
    \item[] Justification: We detail the experimental setup in Section $\ref{sec:experiment}$ and give hyperparameters and training details in Appendix C.
    \item[] Guidelines:
    \begin{itemize}
        \item The answer NA means that the paper does not include experiments.
        \item If the paper includes experiments, a No answer to this question will not be perceived well by the reviewers: Making the paper reproducible is important, regardless of whether the code and data are provided or not.
        \item If the contribution is a dataset and/or model, the authors should describe the steps taken to make their results reproducible or verifiable. 
        \item Depending on the contribution, reproducibility can be accomplished in various ways. For example, if the contribution is a novel architecture, describing the architecture fully might suffice, or if the contribution is a specific model and empirical evaluation, it may be necessary to either make it possible for others to replicate the model with the same dataset, or provide access to the model. In general. releasing code and data is often one good way to accomplish this, but reproducibility can also be provided via detailed instructions for how to replicate the results, access to a hosted model (e.g., in the case of a large language model), releasing of a model checkpoint, or other means that are appropriate to the research performed.
        \item While NeurIPS does not require releasing code, the conference does require all submissions to provide some reasonable avenue for reproducibility, which may depend on the nature of the contribution. For example
        \begin{enumerate}
            \item If the contribution is primarily a new algorithm, the paper should make it clear how to reproduce that algorithm.
            \item If the contribution is primarily a new model architecture, the paper should describe the architecture clearly and fully.
            \item If the contribution is a new model (e.g., a large language model), then there should either be a way to access this model for reproducing the results or a way to reproduce the model (e.g., with an open-source dataset or instructions for how to construct the dataset).
            \item We recognize that reproducibility may be tricky in some cases, in which case authors are welcome to describe the particular way they provide for reproducibility. In the case of closed-source models, it may be that access to the model is limited in some way (e.g., to registered users), but it should be possible for other researchers to have some path to reproducing or verifying the results.
        \end{enumerate}
    \end{itemize}

\item {\bf Open access to data and code}
    \item[] Question: Does the paper provide open access to the data and code, with sufficient instructions to faithfully reproduce the main experimental results, as described in supplemental material?
    \item[] Answer: \answerYes{}
    \item[] Justification: We include our codebase and instructions to run the code in supplementary material.
    \item[] Guidelines:
    \begin{itemize}
        \item The answer NA means that paper does not include experiments requiring code.
        \item Please see the NeurIPS code and data submission guidelines (\url{https://nips.cc/public/guides/CodeSubmissionPolicy}) for more details.
        \item While we encourage the release of code and data, we understand that this might not be possible, so “No” is an acceptable answer. Papers cannot be rejected simply for not including code, unless this is central to the contribution (e.g., for a new open-source benchmark).
        \item The instructions should contain the exact command and environment needed to run to reproduce the results. See the NeurIPS code and data submission guidelines (\url{https://nips.cc/public/guides/CodeSubmissionPolicy}) for more details.
        \item The authors should provide instructions on data access and preparation, including how to access the raw data, preprocessed data, intermediate data, and generated data, etc.
        \item The authors should provide scripts to reproduce all experimental results for the new proposed method and baselines. If only a subset of experiments are reproducible, they should state which ones are omitted from the script and why.
        \item At submission time, to preserve anonymity, the authors should release anonymized versions (if applicable).
        \item Providing as much information as possible in supplemental material (appended to the paper) is recommended, but including URLs to data and code is permitted.
    \end{itemize}

\item {\bf Experimental setting/details}
    \item[] Question: Does the paper specify all the training and test details (e.g., data splits, hyperparameters, how they were chosen, type of optimizer, etc.) necessary to understand the results?
    \item[] Answer: \answerYes{}
    \item[] Justification: The details are given in Section $\ref{sec:experiment}$ with hyperparameters and additional details on data splits and experiments  given in Appendix C.
    \item[] Guidelines:
    \begin{itemize}
        \item The answer NA means that the paper does not include experiments.
        \item The experimental setting should be presented in the core of the paper to a level of detail that is necessary to appreciate the results and make sense of them.
        \item The full details can be provided either with the code, in appendix, or as supplemental material.
    \end{itemize}

\item {\bf Experiment statistical significance}
    \item[] Question: Does the paper report error bars suitably and correctly defined or other appropriate information about the statistical significance of the experiments?
    \item[] Answer: \answerNo{}
    \item[] Justification: Error bars are not reported for accuracy because it would be too computationally expensive. See Appendix C for GPU hours and resources required.
    \item[] Guidelines:
    \begin{itemize}
        \item The answer NA means that the paper does not include experiments.
        \item The authors should answer "Yes" if the results are accompanied by error bars, confidence intervals, or statistical significance tests, at least for the experiments that support the main claims of the paper.
        \item The factors of variability that the error bars are capturing should be clearly stated (for example, train/test split, initialization, random drawing of some parameter, or overall run with given experimental conditions).
        \item The method for calculating the error bars should be explained (closed form formula, call to a library function, bootstrap, etc.)
        \item The assumptions made should be given (e.g., Normally distributed errors).
        \item It should be clear whether the error bar is the standard deviation or the standard error of the mean.
        \item It is OK to report 1-sigma error bars, but one should state it. The authors should preferably report a 2-sigma error bar than state that they have a 96\% CI, if the hypothesis of Normality of errors is not verified.
        \item For asymmetric distributions, the authors should be careful not to show in tables or figures symmetric error bars that would yield results that are out of range (e.g. negative error rates).
        \item If error bars are reported in tables or plots, The authors should explain in the text how they were calculated and reference the corresponding figures or tables in the text.
    \end{itemize}

\item {\bf Experiments compute resources}
    \item[] Question: For each experiment, does the paper provide sufficient information on the computer resources (type of compute workers, memory, time of execution) needed to reproduce the experiments?
    \item[] Answer: \answerYes{}
    \item[] Justification: See Appendix C for GPU hours and memory requirements.
    \item[] Guidelines:
    \begin{itemize}
        \item The answer NA means that the paper does not include experiments.
        \item The paper should indicate the type of compute workers CPU or GPU, internal cluster, or cloud provider, including relevant memory and storage.
        \item The paper should provide the amount of compute required for each of the individual experimental runs as well as estimate the total compute. 
        \item The paper should disclose whether the full research project required more compute than the experiments reported in the paper (e.g., preliminary or failed experiments that didn't make it into the paper). 
    \end{itemize}
    
\item {\bf Code of ethics}
    \item[] Question: Does the research conducted in the paper conform, in every respect, with the NeurIPS Code of Ethics \url{https://neurips.cc/public/EthicsGuidelines}?
    \item[] Answer: \answerYes{}
    \item[] Justification: This paper conforms, in every respect, with the NeurIPS Code of Ethics.
    \item[] Guidelines:
    \begin{itemize}
        \item The answer NA means that the authors have not reviewed the NeurIPS Code of Ethics.
        \item If the authors answer No, they should explain the special circumstances that require a deviation from the Code of Ethics.
        \item The authors should make sure to preserve anonymity (e.g., if there is a special consideration due to laws or regulations in their jurisdiction).
    \end{itemize}

\item {\bf Broader impacts}
    \item[] Question: Does the paper discuss both potential positive societal impacts and negative societal impacts of the work performed?
    \item[] Answer: \answerNA{}
    \item[] Justification: This is a technical paper that has no societal impact.
    \item[] Guidelines:
    \begin{itemize}
        \item The answer NA means that there is no societal impact of the work performed.
        \item If the authors answer NA or No, they should explain why their work has no societal impact or why the paper does not address societal impact.
        \item Examples of negative societal impacts include potential malicious or unintended uses (e.g., disinformation, generating fake profiles, surveillance), fairness considerations (e.g., deployment of technologies that could make decisions that unfairly impact specific groups), privacy considerations, and security considerations.
        \item The conference expects that many papers will be foundational research and not tied to particular applications, let alone deployments. However, if there is a direct path to any negative applications, the authors should point it out. For example, it is legitimate to point out that an improvement in the quality of generative models could be used to generate deepfakes for disinformation. On the other hand, it is not needed to point out that a generic algorithm for optimizing neural networks could enable people to train models that generate Deepfakes faster.
        \item The authors should consider possible harms that could arise when the technology is being used as intended and functioning correctly, harms that could arise when the technology is being used as intended but gives incorrect results, and harms following from (intentional or unintentional) misuse of the technology.
        \item If there are negative societal impacts, the authors could also discuss possible mitigation strategies (e.g., gated release of models, providing defenses in addition to attacks, mechanisms for monitoring misuse, mechanisms to monitor how a system learns from feedback over time, improving the efficiency and accessibility of ML).
    \end{itemize}
    
\item {\bf Safeguards}
    \item[] Question: Does the paper describe safeguards that have been put in place for responsible release of data or models that have a high risk for misuse (e.g., pretrained language models, image generators, or scraped datasets)?
    \item[] Answer: \answerNA{}
    \item[] Justification: The paper poses no such risk.
    \item[] Guidelines:
    \begin{itemize}
        \item The answer NA means that the paper poses no such risks.
        \item Released models that have a high risk for misuse or dual-use should be released with necessary safeguards to allow for controlled use of the model, for example by requiring that users adhere to usage guidelines or restrictions to access the model or implementing safety filters. 
        \item Datasets that have been scraped from the Internet could pose safety risks. The authors should describe how they avoided releasing unsafe images.
        \item We recognize that providing effective safeguards is challenging, and many papers do not require this, but we encourage authors to take this into account and make a best faith effort.
    \end{itemize}

\item {\bf Licenses for existing assets}
    \item[] Question: Are the creators or original owners of assets (e.g., code, data, models), used in the paper, properly credited and are the license and terms of use explicitly mentioned and properly respected?
    \item[] Answer: \answerYes{}
    \item[] Justification: All datasets, models, and pre-existing code bases used are cited.
    \item[] Guidelines:
    \begin{itemize}
        \item The answer NA means that the paper does not use existing assets.
        \item The authors should cite the original paper that produced the code package or dataset.
        \item The authors should state which version of the asset is used and, if possible, include a URL.
        \item The name of the license (e.g., CC-BY 4.0) should be included for each asset.
        \item For scraped data from a particular source (e.g., website), the copyright and terms of service of that source should be provided.
        \item If assets are released, the license, copyright information, and terms of use in the package should be provided. For popular datasets, \url{paperswithcode.com/datasets} has curated licenses for some datasets. Their licensing guide can help determine the license of a dataset.
        \item For existing datasets that are re-packaged, both the original license and the license of the derived asset (if it has changed) should be provided.
        \item If this information is not available online, the authors are encouraged to reach out to the asset's creators.
    \end{itemize}

\item {\bf New assets}
    \item[] Question: Are new assets introduced in the paper well documented and is the documentation provided alongside the assets?
    \item[] Answer: \answerNA{}
    \item[] Justification: This paper does not release new assets.
    \item[] Guidelines:
    \begin{itemize}
        \item The answer NA means that the paper does not release new assets.
        \item Researchers should communicate the details of the dataset/code/model as part of their submissions via structured templates. This includes details about training, license, limitations, etc. 
        \item The paper should discuss whether and how consent was obtained from people whose asset is used.
        \item At submission time, remember to anonymize your assets (if applicable). You can either create an anonymized URL or include an anonymized zip file.
    \end{itemize}

\item {\bf Crowdsourcing and research with human subjects}
    \item[] Question: For crowdsourcing experiments and research with human subjects, does the paper include the full text of instructions given to participants and screenshots, if applicable, as well as details about compensation (if any)? 
    \item[] Answer: \answerNA{}
    \item[] Justification: No crowdsourcing or research with human subjects.
    \item[] Guidelines:
    \begin{itemize}
        \item The answer NA means that the paper does not involve crowdsourcing nor research with human subjects.
        \item Including this information in the supplemental material is fine, but if the main contribution of the paper involves human subjects, then as much detail as possible should be included in the main paper. 
        \item According to the NeurIPS Code of Ethics, workers involved in data collection, curation, or other labor should be paid at least the minimum wage in the country of the data collector. 
    \end{itemize}

\item {\bf Institutional review board (IRB) approvals or equivalent for research with human subjects}
    \item[] Question: Does the paper describe potential risks incurred by study participants, whether such risks were disclosed to the subjects, and whether Institutional Review Board (IRB) approvals (or an equivalent approval/review based on the requirements of your country or institution) were obtained?
    \item[] Answer: \answerNA{}
    \item[] Justification: No research with human subjects.
    \item[] Guidelines:
    \begin{itemize}
        \item The answer NA means that the paper does not involve crowdsourcing nor research with human subjects.
        \item Depending on the country in which research is conducted, IRB approval (or equivalent) may be required for any human subjects research. If you obtained IRB approval, you should clearly state this in the paper. 
        \item We recognize that the procedures for this may vary significantly between institutions and locations, and we expect authors to adhere to the NeurIPS Code of Ethics and the guidelines for their institution. 
        \item For initial submissions, do not include any information that would break anonymity (if applicable), such as the institution conducting the review.
    \end{itemize}

\item {\bf Declaration of LLM usage}
    \item[] Question: Does the paper describe the usage of LLMs if it is an important, original, or non-standard component of the core methods in this research? Note that if the LLM is used only for writing, editing, or formatting purposes and does not impact the core methodology, scientific rigorousness, or originality of the research, declaration is not required.
    %this research? 
    \item[] Answer: \answerNA{}
    \item[] Justification: None used
    \item[] Guidelines:
    \begin{itemize}
        \item The answer NA means that the core method development in this research does not involve LLMs as any important, original, or non-standard components.
        \item Please refer to our LLM policy (\url{https://neurips.cc/Conferences/2025/LLM}) for what should or should not be described.
    \end{itemize}

\end{enumerate}

% #########################################

\newpage
\appendix
% \section*{Appendix for ``Composing Linear Layers from Irreducibles''}
\section*{Appendix}
In the Appendix, we provide additional discussions, formal proofs, and further experimental details that support the main paper. Section \ref{apdx:limitations} outlines the limitations of this work. Section \ref{apdx:proofs} presents the proofs of the algorithms and theorems, along with general results for certain special cases. Section \ref{apdx:hyper-parameters} describes the hyperparameters and training setups used in the experiments. Section \ref{apdx:additional_exp} includes additional experimental results.
\section{Limitations}
\label{apdx:limitations}
While our results demonstrate feasibility, they remain theoretical at this stage---though rotor layers achieve competitive performance when replacing attention layers, the benefits are primarily in a radical parameter count reduction rather than immediate compute speedup in practice. Realizing such gains will require dedicated systems-level optimizations beyond the scope of this work.  

Our current implementation relies on dense matrix representations of rotors and does not yet exploit the inherent sparsity of the rotor decomposition. Leveraging this sparsity will require new backpropagation schemes and software libraries, particularly when training from scratch.  We discuss the matrix representation and sparse structure of rotors in detail in Section~\ref{apdx:proofs}.

Overall, this work is intended to highlight the feasibility and promise of decomposing linear layers---key building blocks in modern large models---into smaller (potentially geometric) modules, such as bivectors and rotors, to facilitate more compact and interpretable architectures.
\section{Proofs}
\label{apdx:proofs}

\subsection{Correctness Proofs for Main Algorithms}
We prove the correctness of Alg.~\ref{algo:inv_decomp} and ~\ref{algo:ga_poweriter} in the theorems below.

\begin{theorem}
    Given a bivector $b\in\Cl^2(n)$, a random vector $v\in\Cl^1(n)$ that has a non-zero component in the direction of the dominant simple component, and a threshold $\epsilon\in\rr$, Alg. $\ref{algo:ga_poweriter}$ returns an approximate simple projection of $b$ in a differentiable way.
\end{theorem}

\begin{proof}
    Let $B\in\so(n)$ be the skew-symmetric matrix corresponding to $b$. We begin by showing that $b\llcorner v = Bv$ for any vector $v$. Fix an orthonormal basis $\{e_1,\dots,e_n\}$ of $\rr^n$ and let
    \[
        b = \sum_{1\leq i < j \leq n} b_{ij} (e_i \wedge e_j) \quad \text{and} \quad v = \sum_{k=1}^n v_k e_k,
    \]
    noting that $[B]_{ij} = b_{ij}$ for $1\leq i < j \leq n$ and $[B]_{ij} = - b_{ji}$ for $1\leq j < i \leq n$. It is easy to verify that
    $(e_i \wedge e_j) \llcorner v = v_je_i - v_i e_j.$
    By bilinearity of the right contraction, we have
    \begin{align*}
        b \llcorner v &= \sum_{i<j} b_{ij}(v_je_i - v_i e_j)\\
        &= \sum_{i<j} b_{ij}(v_je_i) - \sum_{i<j} b_{ij}(v_i e_j) \\
        &= \sum_{i<j} b_{ij}(v_je_i) - \sum_{j<i} b_{ji}(v_j e_i) && \text{swapping $i$ and $j$}\\
        &= \sum_{i<j} [B]_{ij}(v_je_i) + \sum_{j<i} [B]_{ij}(v_j e_i)\\
        &= \sum_{i} \sum_{j} [B]_{ij}v_j\;e_i = \sum_{i} (Bv)_i e_i = Bv.
    \end{align*}
    It follows that
    $$b\llcorner (b\llcorner v) = b\llcorner Bv = B^2v = -B^T Bv.$$
    Since $B^2$ is symmetric, we may apply the power iteration method. Observe, by skew-symmetry, the non-zero eigenvalues of $B$ come in conjugate pairs
    $$\{\pm i\sigma_1, \pm i\sigma_2,\ldots,\pm i\sigma_k\},$$
    with $\sigma_1 \ge \sigma_2 \ge \dots \ge \sigma_k$ and  $k \le \lfloor n/2 \rfloor$. An associated orthonormal set of complex eigenvectors can be written as $\{u_j \pm i v_j\}_{j\le k}$. Since $\sigma_1$, the dominant eigenvalue for $B^2$, has multiplicity 2, the power method on $B^2$ will not converge in a single direction. Instead, it will \textit{oscillate} within the two-dimensional span of $\{u_1,v_1\}$ \citet{wilkinson1965}.
    % In particular, the real and imaginary parts $\{u_j\}$ and $\{v_j\}$ are exactly the singular vectors of $B$, with singular values $\{\sigma_j\}$.      
    
    % However, the power method only converges when the eigenvalues are unique, which is not the case for $B^2$. We review this below.
    
    % Note that the eigenvalues and eigenvectors of $B$ come in conjugate pairs, $\{\pm i\sigma_1, \pm i\sigma_2,\ldots,i\sigma_k\}$ and $\{u_1\pm iv_1, u_2\pm iv_2, \ldots, u_k\pm iv_k\}$ where $\sigma_1 =\sigma_2\geq\sigma_3=\sigma_4\geq\ldots\geq\sigma_{k-1}=\sigma_k\geq 0$, $i$ is the imaginary unit, and $k=\lfloor\frac{n}{2}\rfloor$ when $n$ is even. When $n$ is odd, 
    % there exist an unpaired $0$ eigenvalue, however this poses no issue. 
    % The singular vectors of $B$, which are the eigenvectors of $-B^2$ solved for by the power method, are the real and imaginary parts of the eigenvectors, $\{u_1, u_2, \ldots, u_k\}$ and $\{v_1,v_2,\ldots,v_k\}$, while the singular values are $\{\sigma_1, \sigma_2, \ldots, \sigma_2\}$. Running the power method in this case doesn't converge to just one vector, but rather "converges" in oscillation between singular vectors $\citep{wilkinson1965}$. Note that the singular vectors, and therefore eigenvectors, are unique up to rotation.

    Let $v$ be one of the singular vectors on which the algorithm terminates. Since singular vectors are unique up to rotation, WLOG we choose $v$ to be the right singular vector. Then, the left singular vector $u$ and singular value $\sigma$ satisfy $b\llcorner v = Bv = \sigma u$ by definition.

    By the Eckart-Young-Mirsky theorem \citep{golub2013matrix}, the best rank-2 approximation to $B$ is its projection onto its simple two-dimensional subspace spanned by top two singular directions: 
    $$\Projsimple(b)=\sigma_1 u_1v_1^T + \sigma_2 u_2v_2^T.$$ 
    As the singular values come in pairs for $B$, this reduces to $\sigma_1 \left(u_1v_1^T + u_2v_2^T\right)$. But for $B$, the two right singular vectors $v_1, v_2$ and the two left singular vectors $u_1,u_2$ lie in the same two-plane, coming from the dominant eigenvalue conjugate pair, $\pm i\sigma_1$. Thus, we can choose $u_2 = v_1$ and $v_2 = -u_1$ (or vice versa), which ensures $u_1 \perp v_1$ and $u_2 \perp v_2$. It follows that
    $$\Projsimple(b) = \sigma_1\left(u_1v_1^T - v_1u_1^T\right) = \sigma_1 (u_1 \wedge v_1).$$
    
    % Note that $u_2$ and $v_2$ live in the same span as $u_1$ and $v_1$ as dominant eigenvalue conjugate pair, $\pm i\sigma_1$, corresponds to the eigenvectors $u_1\pm iv_1$. We must choose $u_2$ and $v_2$ orthogonal to each other, and noting that $u_1$ and $v_1$ are also orthogonal, we can choose either $u_2 = v_1, v_2 = -u_1$ or $u_2 = -v_1, v_2 = u_1$. 
    
    % Both choices give that $\Projsimple(b) = \sigma_1(u_1v_1^T - v_1u_1^T)$. Therefore, as $u_1\wedge v_1 = u_1v_1^T - v_1u_1^T$, we have that $\Projsimple(b) = \sigma_1 (u_1 \wedge v_1)$.

    Since Alg~\ref{algo:ga_poweriter} solves only approximately for $\sigma(u\wedge v) \approx \sigma_1(u_1 \wedge v_1)$, we have that
    \[
        \Projsimple(b)\approx \sigma_1(u_1\wedge v_1).
    \]
    Note that since the right contraction implementation is differentiable, Alg $\ref{algo:ga_poweriter}$ is differentiable.
\end{proof}

\begin{theorem}
    Given $b\in\Cl^2(n)$ and $k-1$ many vectors $v\in\Cl^1(n)$, Alg. $\ref{algo:inv_decomp}$ returns the invariant decomposition in a differentiable way.
\end{theorem}

\begin{proof}
    The heavy lifting is done by \cite{eelbode2024outereigentangentconcepts} in Lem.~$\ref{lem:invariantDecomp}$, which states that $b$ can be written as the sum of at most $k$ commuting, orthogonal, simple bivectors
    \[
        b = \sum_{j=1}^k \mu_j\frac{v_{\mu_j^+}\wedge v_{\mu_j^-}}{v_{\mu_j^+}\cdot v_{\mu_j^-}}
    \]
    where $\{\pm\mu_1,\pm\mu_2,\ldots,\pm\mu_k\}$ is the spectrum of $b$ and $v_{\mu_j^+}$ and $v_{\mu_j^-}$ are partner eigenvectors. 

    To obtain the full decomposition, it suffices to iteratively extract each term. If we have a differentiable subroutine to find the first term, then we can subtract it from $b$, apply the same routine to the residual, and repeat. Thus, it suffices to prove that the first term in the sum is equal to the projection of $b$ onto the simple bivectors.

    From the proof of Alg. $\ref{algo:ga_poweriter}$, we know that
    \[
        \Projsimple(b) = \sigma(u\wedge v)
    \]
    where $\sigma$ is the largest singular value and $u,v$ are the corresponding left and right singular vectors. We wish to show that
    \[
        \sigma(u\wedge v) = \mu_1\frac{v_{\mu_1^+}\wedge v_{\mu_1^-}}{v_{\mu_1^+}\cdot v_{\mu_1^-}}.
    \]
    The partner eigenvectors take the form $v_{\mu_1^+} = u + iv$ and $v_{\mu_1^-} = u - iv$, as cited in the proof of Alg. $\ref{algo:ga_poweriter}$. Then, the numerator reduces to
    \begin{equation*}
        \begin{split}
            v_{\mu_1^+}\wedge v_{\mu_1^-} &= (u + iv) \wedge (u - iv)\\
            &= u \wedge u - i u \wedge v + i v \wedge u + v \wedge v\\
            &= - i u \wedge v + i v \wedge u\\
            &= -2i u\wedge v.\\
        \end{split}
    \end{equation*}
    as the wedge product is antisymmetric. The denominator reduces to
    \begin{equation*}
        \begin{split}
            v_{\mu_1^+}\cdot v_{\mu_1^-} &= (u + iv)\cdot(u - iv)\\
            &= u\cdot u - i u\cdot v + iv\cdot u + v\cdot v\\
            &= u\cdot u + v\cdot v\\
            &=2\\
        \end{split}
    \end{equation*}
    as singular vectors are orthonormal. Since $\mu_1 = i\sigma$, we have that
    \[
        \mu_1\frac{v_{\mu_1^+}\wedge v_{\mu_1^-}}{v_{\mu_1^+}\cdot v_{\mu_1^-}}     = i\sigma \frac{-2i u\wedge v}{2}
            = \sigma(u\wedge v)
            = \Projsimple(b),
    \]
    as desired. Subtraction is differentiable and as $\Projsimple(b)$ can be found in a differentiable way, this algorithm is differentiable.
\end{proof}

\subsection{Parameter Count}

We revisit the main theorem on the parameter count for rotor maps and provide a proof of the statement.

\begin{theorem}[$\psi$ Parameter Count]
Let \( \psi: \mathbb{R}^{d_{\text{in}}} \to \mathbb{R}^{d_{\text{out}}} \) be a linear map composed of rotor modules 
\(\psi_{r_{ij}, s_{ij}}\) with $i\in[c_1]$ and $j\in[c_2]$, each acting in $\Cl(n)$. Let $2^n \le d \triangleq \min(d_{\text{in}}, d_{\text{out}})$. Then, the total number of learnable parameters is upper bounded by
\[
    2 c_1 c_2 \binom{n}{2} = \mathcal{O}(\log^2 d).
\]
\end{theorem}

\begin{proof}
Each rotor is the exponential of a bivector. As general bivectors are a linear combination of basis bivectors, of which there are $n \choose 2$, parametrizing a bivector takes $\binom{n}{2}$ scalar parameters. Hence, a single rotor-sandwich map $\psi_{r_{ij}, s_{ij}}$ is parametrized using $2 \binom{n}{2}$ scalar parameters. The result is then immediate, since there are $c_1 c_2$ such modules, corresponding to all input and output pairs.
\end{proof}

Tab.~\ref{tab:appdx-param-cnt} summarizes the number of parameters required for each key, query, and value projection per attention layer using different replacement methods across all LLMs used in our experiments.

\subsection{Technical Analysis of Matrix Representation of Rotor Application}
\label{apdx:techAnalysis}
We implement the sandwich product in $\eqref{eq:rotor-sandwich}$, along with other operations such as grade-restricted wedge and inner products, using the \texttt{torch\_ga} Clifford algebra package for PyTorch, available at \citet{torch_ga}.
In our implementation of $\eqref{eq:rotor-sandwich}$, the rotor action is represented by a matrix $M$. In this section, we describe the construction of $M$ and introduce some of its key properties. These results highlight the special orthogonality and block structure of \(M\). While the two-rotor map \(\psi_{r,s}(x)\) in \eqref{eq:rotor-transformation} does not satisfy these properties exactly, it appears to share similar properties and structure; we do not discuss it here.

We compute the sandwich product using a change-of-basis matrix.
\begin{definition}
    Let $r \in \Spin(n)$, and let $\{e_J\}_{J \subseteq [n]}$ denote the canonical basis of $\Cl(n)$ (where $J$ is an ordered multi-index). Define
    $$N_r = [\tau(\psi_{r}(e_J))]_{J \subseteq [n]},$$
    where $\psi_{r}(x) = rxr^\dagger$ and $\tau$ is the canonical vector-space isomorphism $\tau:\Cl(n)\rightarrow\rr^{2^n}$. That is, each row of $N_r$ is the coefficients of $r e_J r^{\dagger}$ expressed in the basis. We refer to $N_r$ as the change-of-basis matrix for $r$.
    \label{def:changeBasis}
\end{definition}

\begin{lemma}
    Let $x\in\Cl(n)$, $r\in\Spin(n)$, and $N_r$ be defined as in Def.~\ref{def:changeBasis}. Then, the two mappings
    \[x\mapsto xN_r \quad\text{and}\quad x\mapsto rxr^\dagger\]
    are the same linear map up to isomorphism.
\end{lemma}
\begin{proof}
    Write $x$ as 
    $$x = \sum_{J\subseteq[n]} x_J e_J,$$
    where $x_J$ is the real coefficient of the basis element $e_J$. Let $\tau:\Cl(n)\rightarrow \rr^{2^n}$ be the canonical vector-space isomorphism that gives the row vector coordinates of the multivector, $\phi_r(x) = rxr^\dagger$ for $x\in\Cl(n)$, and $\delta(y) = yN_r$ for $y\in\rr^{2^n}$. To prove $\phi$ and $\delta$ are the same up to isomorphism, we must prove that $\tau(\phi_r(x)) = \delta(\tau(x))$ for all $x\in\Cl(n)$.
    
    For each basis element $e_J$, the sandwich product gives some multivector $\psi(e_J) = r e_J r^{\dagger}$, which we express as
     $$r e_J r^{\dagger} = \sum_{I \subseteq [n]} [N_r]_{J,I} e_I,$$
    since $(N_r)_{J,I}$ stores the coefficient of $e_I$ in $r e_J r^{\dagger}$ by definition. Then,

    \begin{equation*}
        \begin{split}
            \tau(\phi_r(x)) &= \tau\left(rxr^\dagger\right)\\
            &= \tau\left(r\left[\sum_{J\subseteq [n]}x_Je_J\right]r^\dagger\right)\\
            &= \tau\left(\sum_{J\subseteq [n]}x_Jre_Jr^\dagger\right)\\
            &= \tau\left(\sum_{J\subseteq [n]}x_J\left[\sum_{I\subseteq[n]}[N_r]_{J,I}e_I\right]\right)\\
            &= \tau\left(\sum_{I\subseteq [n]}\left[\sum_{J\subseteq[n]}x_J[N_r]_{J,I}\right]e_I\right)\\
            &= \sum_{I\subseteq[n]}\left[\sum_{J\subseteq[n]}x_J[N_r]_{J,I}\right]\tau(e_I)\quad\text{by linearity}\\
            &= \tau(x)N_r\\
            &= \delta(\tau(x))\\
        \end{split}
    \end{equation*}
    where the second to last step follows as the inner sum is the dot product of $\tau(x)$ with the $I$th column of $N_r$.
    % Now, we expand
    % \[
    %     rxr^{\dagger} = r \left(\sum_{J} x_J e_J\right) r^{\dagger}
    %     = \sum_{J} x_J \psi_{r}(e_J)
    %     = \sum_{I}\left(\sum_{J} x_J (N_r)_{J,I}\right) e_I
    %     = \bar{x}N_r.
    % \]
    % Since $x$ was arbitrary, we are done. \dy{Need Vikas to look over lemma 3.}
\end{proof}

This allows us to compute the rotor sandwich product via matrix multiplication. However, this gives us very little insight into the structure of $N_r$ itself. For one, $N_r$ must respect the grade-preserving property of rotors, meaning it is a block diagonal matrix. The following alternative view provides more insight into its structure.

\paragraph{A more revealing view of $N_r$.}
Let $b\in\Cl^2(n)$ be the bivector associated with $r\in\Spin(n)$ by $r=\exp(b)$, and let $B\in\so(n)$ be the corresponding skew-symmetric matrix. Let \(R\triangleq\exp(2B)\in\SO(n)\) via the matrix exponential map. For each grade $0\leq k\leq n$, set
\[
    M_{k} \;\triangleq\; C_{k}(R)\in \rr^{\binom{n}{k}\times\binom{n}{k}},
    \label{def:compMat}
\]
where $C_k(R)$ is the $k$th compound matrix of $R$. Algebraically, $C_k(R)$ is the $k$th exterior power of $R$, i.e., the unique linear map $\bigwedge^k(R)$ such that
\[\bigwedge^k(R)(m_1\wedge m_2\wedge\cdots\wedge m_k) = Rm_1 \wedge Rm_2 \wedge \cdots \wedge Rm_k\]
for all $m_i\in \rr^n$ for $i\in[k]$ \citep{conrad_exterior_powers}. 
\begin{definition}\label{def:diagMat}
    Let $M_k$ be the $k$th compound, or exterior power, of $R$. Stack the $M_k$ on the diagonal to form $$M \triangleq \diag(M_{0},M_{1},\dots,M_{n}) \in \rr^{2^n \times 2^n}.$$
\end{definition}
We first show that $N_r$ and $M$ are in fact the same matrix.
\begin{theorem}\label{thm:matEquiv}
     Let $N_r$ and $M$ be defined as in Def. $\ref{def:changeBasis}$ and $\ref{def:diagMat}$. Then, $M=N_r$.
\end{theorem}

\begin{proof}
    The rotor sandwich product is a grade-preserving automorphism, implying $\psi_{r}(\Cl^k(n)) = \Cl^{k}(n)$. Thus, $N_r$ is block-diagonal, one block per grade. First, we show $\psi_r(v) = Rv$ for a vector $v$, where $R = \exp(2B)$. Observe
    \begin{equation*}\label{eq:rotor-to-bivec}
        rvr^{\dagger} = e^b v \left(e^b\right)^{\dagger} = e^b v e^{b^\dagger} = e^b v e^{-b},
    \end{equation*}
    where the second equality follows as reversion is an anti-automorphism and the third equality follows as $b^{\dagger} = -b$. Now, consider the adjoint operator in an associative algebra defined as $\mathrm{ad}_{X}(Y) \triangleq [X,Y] = XY-YX$. The Hadamard lemma and Taylor expansion together imply
    $$e^b v e^{-b} = \sum_{k \ge 0} \frac{1}{k!} (\mathrm{ad}_{b})^k(v) = \exp(\mathrm{ad}_{b})v.$$
    We also note the following identity:
    $$\mathrm{ad}_b(v) = bv-vb = (b\cdot v + b \wedge v)-(v \cdot b + v \wedge b) = 2 b \cdot v = 2Bv,$$
    where $B$ is the skew-symmetric matrix for $b$. Putting these together, it follows that $\psi_r(v) \triangleq rvr^{\dagger} = \exp(2B)v = Rv.$ Now, consider any $k$-vector $v = e_{i_1} \wedge \dots \wedge e_{i_k}$. We have
    \begin{align*}
        \psi_r(v) &= \psi(e_{i_1}) \wedge \dots \wedge \psi(e_{i_k}) && \text{$\psi_r$ is a grade-preserving automorphism}\\
        &= (Re_{i_1}) \wedge \dots \wedge (Re_{i_k}) && \text{from above}\\
        &= \bigwedge^{k}(R)(v).
    \end{align*}
    Since the mapping is unique, we are done.
\end{proof}

Note that this block-diagonal structure highlights the \textit{grade-preserving behavior of rotors}. This motivates our design choice to allow grade-mixing in our gadget for more expressivity. In particular, we insert random permutations between rotor layers so that information can mix across different grade components of the input multivector. Now that we have an alternative characterization of the sandwich product, we analyze its structure and the space in which it lives. Below, we examine several structural properties of the matrix $M$. 

\begin{property}
    The matrix $M$ in Def. $\ref{def:diagMat}$ is a block diagonal matrix with at most $2n \choose n$ non-zero entries.
\end{property}
\begin{proof}
 $M=N_r$ by Thm. $\ref{thm:matEquiv}$ and so $M$ is block diagonal. There are $n+1$ blocks, each of size ${n \choose k} \times {n \choose k}$ for $0 \le k \le n$. Thus, the total number of non-zero entries is at most
    \[
    \sum_{k=0}^n {n \choose k}^2 = {2n \choose n},
    \]
    which is far less than the $2^{2n}$ of a dense $2^n \times 2^n$ matrix.
\end{proof}

\begin{property}
    For every grade $0\leq k\leq n$, the block
    $M_{k}=C_{k}(R)$ lies in
    $\SO\left(\binom{n}{k}\right)$. Consequently, \(M \in \SO(2^{n})\).
    \label{prty:weakSO}
\end{property}

\begin{proof}
    By assumption, $B\in\so(n)$, and so $\exp(2B)=R\in\SO(n)$. Taking the $k$th exterior power preserves orthogonality and a determinant of $1$. For each $k$, the $k$th compound matrix $M_k = C_k(R)$ is the matrix representation of the $k$th exterior power $\bigwedge^k(R)$ with respect to the standard basis of $k$-vectors. Since $R$ is orthogonal, $\bigwedge^k(R)$ preserves the induced inner product on $k$-vectors, making $M_k$ orthogonal. Moreover, since $\det(R) = 1$, we have $\det(M_k) = \det\left(\bigwedge^k(R)\right) = (\det(R))^{\binom{n-1}{k-1}} = 1$. Therefore, $M_k \in SO\left(\binom{n}{k}\right)$ for all $0 \leq k \leq n$.
    To prove $M\in\SO(2^n)$, we note that since $M$ is block diagonal, $M^{T} = \diag\left(M_0^T,\dots,M_n^T\right)$ and so $M^TM=I$ as the $M_k$ are special orthogonal. $\det(M)=1$ as the determinant of block diagonal matrices is the product of the determinant of the blocks, meaning $\det(M) = \prod_{k=0}^n \text{det}(M_k) = \prod_{k=0}^n 1 = 1$.
    
    % In the canonical basis, it is represented by $M_k$, so $M_k \in \SO\!\bigl(\binom{n}{k}\bigr)$. Now, we sequentially stack $M_k$ on the diagonal to construct $M$, whose transpose is $M^{T} = \diag(M_0^T,\dots,M_n^T)$. This gives us $M^TM = I$ and $\text{det}(M) = \prod_{k} \text{det}(M_k) = 1$. \tp{TODO for Travis: Look at this}
\end{proof}

Note that under this view, we obtain a map from $\so(n)$ to $\SO(2^n)$ by first applying the surjective exponential map from $\so(n)$ to $\SO(n)$, and then lifting $\SO(n)$ to $\SO(2^n)$ via the exterior powers. This transformation enables an \textit{exponential reduction in the number of parameters required to represent rotor layers}. We conclude by identifying the class of matrices which $M$ belongs. The following is a simple but useful characterization. The group of invertible elements of $\Cl(n)$ is 
$$\Cl^\times(n) = \left\{a\in\Cl(n) | \exists a^{-1}\in\Cl(n) \text{ st. } aa^{-1} = a^{-1}a = 1\right\}.$$
Let 
$$\phi: \Cl^\times(n) \rightarrow \mathrm{GL}(2^n, \rr)$$ be the mapping which assigns to each $a\in\Cl^\times(n)$ its change-of-basis matrix $N_a$. Then for any $N_r$ corresponding to $r\in\Spin(n)$,
    \[N_r\in\phi(\Spin(n)).\]

We now give a more refined version of Prop.~$\ref{prty:weakSO}$.

\begin{property} \label{prty:strongSO}
    Let $\phi$ be defined as above. Then,
    \[\phi(\Spin(n)) = \SO\left(2^n\right) \cap \phi\left(\Cl^\times(n)\right) \cap \left( \bigoplus_{k=0}^n \mathbb{R}^{\binom{n}{k} \times \binom{n}{k}} \right).\]
\end{property}
\begin{proof}

The first direction, $\phi(\Spin(n)) \subseteq \SO(2^n) \cap \phi\left(\Cl^{\times}(n)\right) \cap \left( \bigoplus_{k=0}^n \mathbb{R}^{\binom{n}{k} \times \binom{n}{k}} \right)$, is immediate. If $r \in \Spin(n)$, then by definition $r$ is an invertible element in $\Cl^+(n)$
which satisfies $rr^\dagger = 1_s$ (where $1_s$ denotes the scalar identity). Thus, $\Spin(n) \subseteq \Cl^{\times}(n)$, which implies $\phi(\Spin(n)) \subseteq \phi\left(\Cl^{\times}(n)\right)$.
Furthermore, Prop.~$\ref{prty:weakSO}$ and Thm~$\ref{thm:matEquiv}$ state $N_r \in \SO(2^n)$ for $r \in \Spin(n)$, which ensures $\phi(\Spin(n)) \subseteq \SO(2^n)$. As rotors are grade preserving, $\phi(\Spin(n)) \subseteq \left( \bigoplus_{k=0}^n \mathbb{R}^{\binom{n}{k} \times \binom{n}{k}} \right)$.

For the other direction $\SO(2^n) \cap \phi\left(\Cl^{\times}(n)\right) \cap \left( \bigoplus_{k=0}^n \mathbb{R}^{\binom{n}{k} \times \binom{n}{k}} \right) \subseteq \phi(\Spin(n))$, we start by taking an arbitrary element from the set on the LHS and
show it must also be in the set on the RHS. Thus, we assume $g \in \Cl^{\times}(n)$ (an invertible multivector) such that its corresponding matrix $N_g = \phi(g)$ is in $\SO(2^n)$.
We wish to show $g \in \Spin(n)$. To do this, it suffices to prove three conditions:
$gg^\dagger = 1_s$ (i.e., $g$ is unit-norm), $g \in \Cl^+(n)$ (i.e., $g$ is an even multivector), and $g$ is grade preserving.

The matrix $N_g$ represents the linear transformation $T_g(X) = gXg^\dagger$. The condition $N_g \in \SO(2^n)$ means that $N_g$ is an orthogonal matrix with $\mathrm{det}(N_g)=1_s$.
We know that due to orthogonality, $T_g$ preserves the canonical inner product on $\Cl(n)$. We define this inner product as $\langle X, Y \rangle = \left\langle X^\dagger Y \right\rangle_0$,
where $\langle \cdot \rangle_0$ denotes the scalar part of a multivector.

We will first show that $g^\dagger g = 1_s$. Since $T_g$ is an orthogonal transformation w.r.t. $\langle X, Y \rangle = \left\langle X^\dagger Y \right\rangle_0$, it must satisfy
$T_g^* T_g = \Id$, where $\Id$ is the identity map and $T_g^*$ is the adjoint of $T_g$ w.r.t. the inner product.
The adjoint $T_g^*$ is given by $T_g^*(Y) = g^\dagger Y g$. This can be verified by checking the main property of an adjoint, $\left\langle Y, T_g(X) \right\rangle = \left\langle T_g^*(Y), X \right\rangle$:
\begin{align*}
\left\langle Y, T_g(X) \right\rangle &= \left\langle Y^\dagger \left(gXg^\dagger\right) \right\rangle_0 \\
&= \left\langle \left(g^\dagger Y^\dagger g\right) X \right\rangle_0 \quad \text{(cyclically permute multivectors inside} \\
& \hspace{8.5em}\text{scalar part operation, } \langle MNP \rangle_0 = \langle PMN \rangle_0 \text{)}\\
&= \left\langle \left(g^\dagger Y g\right)^\dagger X \right\rangle_0 \quad \text{(since } (ABC)^\dagger = C^\dagger B^\dagger A^\dagger \text{ and } \left(A^\dagger\right)^\dagger=A \text{)} \\
&= \left\langle T_g^*(Y), X \right\rangle.
\end{align*}
Now, applying the condition $T_g^* T_g = \Id$ means $T_g^*\left(T_g(X)\right) = X$ for all $X \in \Cl(n)$.
Substituting the explicit forms for $T_g$ and $T_g^*$, we get
\begin{align*}
g^\dagger \left(gXg^\dagger\right) g &= X \\
\left(g^\dagger g\right) X \left(g^\dagger g\right) &= X.
\end{align*}
Let $A = g^\dagger g$. The equation becomes $AXA = X$ for all $X \in \Cl(n)$.
By the Wedderburn-Artin theorem \citet{cohn2003basic}, the Clifford algebra $\Cl(n,0)$
is isomorphic to a matrix algebra (or a direct sum of two such algebras) over $\mathbb{R}$, $\mathbb{C}$, or $\mathbb{H}$.
In such matrix algebras (and hence in $\Cl(n,0)$ or its relevant simple components where $A$ lives, noting $A$ is even), if an element $A$ satisfies $AXA = X$
for all elements $X$ of the algebra, then $A$ must be a scalar multiple of the identity, specifically $\pm 1_s$.
Therefore, we must have $A = g^\dagger g = \pm 1_s$.

Fix an orthonormal basis ${e_J}$ (where $J$ is an ordered multi-index). Note that $e_J^\dagger e_J = 1_s$ for every $e_J$. Write $g = \sum_J c_J e_J$ for scalar coefficients $c_J \in \mathbb{R}$. Then, the scalar part
of $g^\dagger g$ is $\left\langle g^\dagger g \right\rangle_0 = \sum_J c_J^2$. Since $g$ is invertible, we have $g \neq 0$, so at least one $c_J \neq 0$, making this sum strictly positive.
Since $g^\dagger g = \lambda \cdot 1_s$ (where $\lambda = \pm 1_s$), its scalar part is $\lambda$. As this scalar part $\left\langle g^\dagger g \right\rangle_0$ must be positive,
we must conclude that $\lambda=1_s$, so $g^\dagger g = 1_s$.

Next, we will show that $gg^\dagger = 1_s$. From the previous result, $g^\dagger g = 1_s$. This means $g$ is invertible and its unique inverse is $g^{-1} = g^\dagger$.
Then, we can write $gg^\dagger = gg^{-1} = 1_s$. This helps us establish the first condition required for $g$ to be an element of $\Spin(n)$.

The block-diagonal structure of $\bigoplus_{k=0}^n \mathbb{R}^{\binom{n}{k} \times \binom{n}{k}}$ ensures that the matrix representation $\phi(g)$ acts independently on each grade subspace $\text{Cl}^k(n)$. Consequently, the induced inner automorphism $T_g(X) = gXg^\dagger$ is grade-preserving, mapping each $k$-vector space to itself.

Finally, we will need to show that $g \in \Cl^+(n)$ (i.e., $g$ is an even multivector). Since $N_g \in \SO(2^n)$, we have $\mathrm{det}(N_g)=1_s$.
Since $gg^\dagger=1_s$, it follows that $g^\dagger=g^{-1}$. The transformation can now be written as $T_g(X) = gXg^{-1}$, which is an inner automorphism of $\Cl(n)$ induced by $g$; such transformations are algebra automorphisms, meaning that they preserve the algebraic product structure (i.e., $T_g(XY) = T_g(X)T_g(Y)$).
Since $g$ preserves grade (and in particular preserves the vector subspace $\Cl^1(n)$), the element $g$ can be expressed as a product of invertible vectors. Combined with $gg^\dagger = 1_s$, this means $g$ is a unit versor, i.e., $g \in \mathrm{Pin}(n)$. Versors who are products of an odd number of vectors (odd versors) are elements of $\mathrm{Pin}(n) \setminus \Spin(n)$, while
those that are products of an even number of vectors (even versors) are elements of $\Spin(n)$. Note that $\mathrm{Pin}(n)$ is the group of all such versors, and $\Spin(n) = \mathrm{Pin}(n) \cap \Cl^+(n)$.

We know that if $g$ were indeed an odd versor (an element of $\mathrm{Pin}(n) \setminus \Spin(n)$), the induced inner automorphism $X \mapsto gXg^{-1}$ when restricted to the vector subspace $\Cl^1(n)$ is an orthogonal transformation with determinant $-1_s$ (as it includes a reflection, reversing orientation). It is a known result from the representation theory of Clifford algebras that the determinant of the full automorphism $N_g: X \mapsto gXg^{-1}$ acting on the entire algebra $\Cl(n)$ corresponds to the parity of the versor $g$. Specifically, $\mathrm{det}(N_g)=+1_s$ if $g$ is an even versor (i.e., $g \in \Spin(n)$), and $\mathrm{det}(N_g)=-1_s$ if $g$ is an odd versor (i.e., $g \in \mathrm{Pin}(n) \setminus \Spin(n)$). This directly links the parity of the versor $g$ to the determinant of the full transformation matrix $N_g$. Since we know $\mathrm{det}\left(N_g\right)=1_s$, $g$ must be an even versor, i.e., $g \in \Cl^+(n)$.

Since we have shown $gg^\dagger = 1_s$, $g \in \Cl^+(n)$, and that $g$ is grade preserving, we conclude that $g \in \Spin(n)$ by its definition. Thus, $\SO(2^n) \cap \phi\left(\Cl^{\times}(n)\right) \cap \left( \bigoplus_{k=0}^n \mathbb{R}^{\binom{n}{k} \times \binom{n}{k}} \right) \subseteq \phi(\Spin(n))$.
\end{proof}

Prop. \ref{prty:strongSO} provides an alternative characterization of the group of matrices associated with rotor conjugation, i.e., $N_r$ lives in the intersection of the special orthogonal group, the subgroup of matrices defined by the group of units of $\Cl(n)$, and those that are grade preserving.
\section{Hyperparameters and Experiment Details}
\label{apdx:hyper-parameters}

We detail the training configurations for all experiments involving Rotor, Low-Rank (LR), and Block-Hadamard (BH) projections used to approximate attention layers.

\subsection{Training Details}

\paragraph{Learning approximation layers in attention.}
To approximate the linear projections in attention layers of an LLM, we train all replacement modules (Rotor, LR, or BH) by minimizing mean squared error (MSE) loss between the predicted and true projection outputs, based on latent representations extracted from LLMs. 

Formally, let $W \in  \rr^{d_{\text{out}} \times d_{\text{in}}}$ be the dense projection matrix (i.e., query, key, or value) we want to approximate within a transformer block. Given hidden input $x \in \rr^{d_{\text{in}}}$, the projection computes $y = Wx \in \rr^{d_{\text{out}}}$.  To train an approximate layer, we collect a dataset $\mathcal{D} = \{(x_i, y_i)\}_{i=1}^N$ by prompting the LLM with a set of prompts $\{P_j\}_{j \le n}$ (e.g., \texttt{Arc Challenge}) and extracting the relevant hidden states at the target layer. Since hidden states are available for each token position in a sequence for self-attention, we have $N = nT$, where $T$ is the average prompt length. We then learn an approximate layer $H_{\theta}$ (rotors, LR, or BH) by minimizing
\[
    \min_{\theta} \sum_{i=1}^{N} (H_{\theta}x_i - y_i)^2,
\]
where $\theta$ denotes the set of trainable parameters (e.g., bivector coefficients for rotors). Optimization is done via gradient descent using the Adam optimizer \citep{kingma2017adammethodstochasticoptimization}. 

We jointly replace the query, key, and value projections within an attention layer of a transformer block. In our experiments, we replace up to three such attention layers and evaluate the resulting model on downstream tasks of perplexity and accuracy metrics across various prompt datasets. When replacing multiple attention layers, say layers $I < J < K$, we train them sequentially in order: first $I$, then $J$, and finally $K$. For each layer, we first replace all earlier trained layers (e.g., $I$ before $J$), and then extract the input-output data for training the new layer under this modified model. This is to ensure that each replacement layer is trained with respect to the distribution induced by preceding replacements. Also, whenever we replace layer $L$, we retrain the output linear projection $W_o^{L}$ within the same attention block, using the same MSE and Adam optimizer for consistency.

\begin{table}[ht]
\centering
\footnotesize
\setlength{\tabcolsep}{5pt}
\renewcommand{\arraystretch}{1.1}
\begin{tabular}{lccc}
  \toprule
  \textbf{Model} & \textbf{Key} & \textbf{Query} & \textbf{Value} \\
  \midrule
  \texttt{LLaMa-3.2 1B} 
    & $2048 \rightarrow 512$ 
    & $2048 \rightarrow 2048$ 
    & $2048 \rightarrow 512$ \\
  \texttt{LLaMa-3.2 3B} 
    & $3072 \rightarrow 1024$ 
    & $3072 \rightarrow 3072$ 
    & $3072 \rightarrow 1024$ \\
  \texttt{Qwen-2.5 1.5B} 
    & $1536 \rightarrow 256$ 
    & $1536 \rightarrow 1536$ 
    & $1536 \rightarrow 256$ \\
  \texttt{Fox-1.0 1.6B} 
    & $2048 \rightarrow 512$ 
    & $2048 \rightarrow 2048$ 
    & $2048 \rightarrow 512$ \\
  \bottomrule
\end{tabular}
\vspace{6pt}
\caption{Input/output hidden dimensions for key, query, and value projections in a single attention layer of different LLMs.}
\label{tab:qkv-shapes}
\end{table}

In Tab. \ref{tab:qkv-shapes}, we summarize the input and output hidden dimension for different LLMs. 

\paragraph{Rotor networks.}
We define the rotor-based transformation as:
\begin{equation*}
    \psi_{r,s}(x) \triangleq r x s^\dagger,
    \label{eq:rotor-transformation-appdx}
\end{equation*}
where $r,s\in\Spin(n)$. Each rotor map is then composed as
\begin{equation*}
    \psi(x) \triangleq \sigma\left(\left\{\psi_{r_{ij},s_{ij}}\left(x^{I_i}\right) | \; i\in[c_1], j\in[c_2]\right\}\right),
    \label{eq:rotor-appdx}
\end{equation*}
where $\sigma$ is a pooling operator applied over the outputs of individual rotor transformation $\psi_{r_{ij},s_{ij}}$. More details are given in Subsection $\ref{subsc:genRotorGadget}$. As discussed, each rotor is parameterized by a small number of bivector coefficients that encode geometric rotations, leading to significantly fewer learnable parameters compared to dense or baseline layers. A rotor layer is constructed by stacking multiple rotor maps in depth and arranging them in parallel across width. For example, with width 2 and depth 3, the layer contains 6 rotor maps, each parameterized independently. In our experiments, to reduce computational cost, we use at most width 2 and depth 3 for a rotor layer, which along with any up and down projection needed that is done by the generalized version in $\ref{subsc:genRotorGadget}$ sums up to roughly 1000 scalar parameters per projection (i.e., per Q/K/V) layer for \texttt{LLaMa-3.2 1B}, compared to $1-4$M in original dense layers. 

Each rotor map is followed by a sequence of fixed (i.e., parameter-free) permutations, normalizations, and a nonlinearity. Since rotor sandwich products are grade-preserving (see Section~\ref{apdx:proofs}), we apply  fixed permutations to enable interaction across grades and increase expressivity. We found that normalization improves training stability. 

Hyperparameters such as depth, width, learning rate, and weight decay are selected via grid search; the final values along with the values we explored are listed in Tab. $\ref{tab:hyperparam}$. All Clifford algebraic operations, including exponentiation of simple bivectors and sandwich products, are implemented entirely in PyTorch using the \texttt{torch\_ga} library \citet{torch_ga}, which supports differentiation. We modified several methods in this package to reduce memory usage. For example, the original package computes the geometric product using a very sparse three dimensional Cayley table of shape $2^n\times 2^n\times 2^n$ \citep{hitzer2013introduction}. However, since we only require the geometric product between a pure rotor and a multivector when computing the sandwich product, we discard all but $1 + \binom{n}{2}$ parts of the third dimension, rather than keeping the full $2^n$.

\begin{table}[t]
\centering
\footnotesize
\setlength{\tabcolsep}{8pt}
\renewcommand{\arraystretch}{1.2}
\begin{tabular}{llcc}
  \toprule
   Method & Hyperparameter & Values Explored & Final Value \\
  \midrule
  \midrule
  \multirow{10}{*}{\texttt{Rotor}}
    & Chunk Size   & 1024, 2048, 4096 & 2048 \\
    & Depth        & 1,2,3 & 1 or 3 \\
    & Width        & 1,2,3 & 1 or 2 \\
    & Nonlinearity          & \texttt{ReLU}, \texttt{PReLU}, \texttt{GELU} & \texttt{PReLU} \\
    & Normalization           & \texttt{true}, \texttt{false} & \texttt{true} \\
    & Permutations          & \texttt{true}, \texttt{false} & \texttt{true} \\
    & Learning rate & $0.001, 0.005, 0.01, 0.05$ & 0.05\\
    & $\ell^2$ weight decay & 0.001, 0.01, 0.1, 0 & 0\\
    & Batch size & 16, 32, 64, 128, 256 & 64\\
    & Cosine annealing  & \texttt{true}, \texttt{false} & \texttt{true} \\
  \midrule
  \multirow{2}{*}{\texttt{Low-Rank (LR)}} 
    & Learning rate & $0.001, 0.005, 0.01, 0.05$ & 0.01\\
    & $\ell^2$ weight decay & 0.001, 0.01, 0.1, 0 & 0\\
    & Batch size & 16, 32, 64, 128, 256 & 256\\
    & Cosine annealing  & \texttt{true}, \texttt{false} & \texttt{true} \\
  \midrule
  \multirow{3}{*}{\texttt{Block-Hadamard (BH)}} 
    & Block number                 & $64$ & 64\\
    & Learning rate & $0.001, 0.005, 0.01, 0.05$ & 0.01\\
    & $\ell^2$ weight decay & 0.001, 0.01, 0.1, 0 & 0\\
    & Batch size & 16, 32, 64, 128, 256 & 256\\
    & Cosine annealing  & \texttt{true}, \texttt{false} & \texttt{true} \\
  \bottomrule
\end{tabular}
\vspace{6pt}
\caption{Hyperparameter settings used for each method. }
\label{tab:hyperparam}
\end{table}

\paragraph{Low-rank approximations.}
We use low-rank (LR) approximation as one of our baselines, following prior work such as \citet{hu2021loralowrankadaptationlarge}. Given a dense matrix \( W \in \mathbb{R}^{d_{\text{out}} \times d_{\text{in}}} \), we approximate it as the product of two lower-dimensional matrices: \( X \in \mathbb{R}^{d_{\text{out}} \times r} \) and \( Y \in \mathbb{R}^{r \times d_{\text{in}}} \), such that
\[
    W \approx XY,
\]
where $r \ll d_{\text{in}}, d_{\text{out}}$. This decomposition effectively constrains the rank of the approximation to at most \( r \), capturing a low-rank subspace of the original operator. Low-rank approximations have been shown to be effective in downstream tasks, especially when applied as additive fine-tuning modules to frozen large pre-trained weights. They are also computationally efficient, requiring only \( \mathcal{O}(r(d_{\text{out}} + d_{\text{in}})) \) parameters and operations. In our experiments, we choose $r = 1$ and $r=4$, and LR1 requires roughly $3-5 \times$ parameters than rotor layers. Hyperparameters are listed in Tab.~\ref{tab:hyperparam}.

\paragraph{Block-Hadamard approximations.}
We adopt the Block-Hadamard (BH) projection as another baseline. The idea is to alternate Hadamard transforms with learnable block-diagonal matrices, enabling a trade-off between expressivity and efficiency through the block size~\citep{zeng2023lookupffn}. Formally, it is defined as
\[
    W \approx \prod_{i=1}^{m} B_i H,
\]
where each \( B_i \) is a learnable block-diagonal matrix with block size \( b \), and \( H \) is a fixed Hadamard transform. This approximation requires only \( \mathcal{O}(b \cdot \max(d_{\text{in}}, d_{\text{out}})) \) parameters, and can be interpreted as analogous to grouped convolutions followed by channel shuffling~\citep{zhang2017shufflenetextremelyefficientconvolutional}. The parameter \( m \) controls the depth of the transformation. In our experiments, we use \( m = 1 \) (i.e., BH1) to reduce parameter count as much as possible and match the scale of rotor-based models. Even with this minimal configuration, BH1 has 8-40$\times$ more parameters than rotors in the case of \texttt{LLaMa-3.2 1B}. Specifically, we use the approximation \( W \approx BH \), where \( B \) is a block-diagonal matrix composed of \( n \) rectangular blocks, each of height \( h = d_{\text{out}} / n \) and width \( w = d_{\text{in}} / n \) for the number $n$ of blocks. Hyperparameters are listed in Tab.~\ref{tab:hyperparam}.

\paragraph{Hyperparameters and model architecture of \texttt{FMNIST} experiments.}
We trained both dense and rotor-based MLPs under identical conditions. Each model consists of $2$ hidden layers (either rotor or dense), followed by ReLU activations. All dense layers were replaced with rotor layers in our rotor variant, except for the final classification head. The dense layers had hidden dimension of $512$, while the rotor layers used $\Cl(15)$ with width $3$ and depth $1$.  We performed a search over learning rates $\eta \in (0.001, 0.1)$ and selected $\eta = 0.005$ for the rotor-based model and $\eta = 0.002$ for the dense baseline based on validation accuracy.

\subsection{Computational Resources for Experiments}
The experiments for each \texttt{LLaMa-3.2 1B} and \texttt{Qwen-2.5 1.5B} each took around $1500$ GPU hours, \texttt{Fox-1.0 1.6B} around $1000$ GPU hours, and \texttt{LLaMa-3.2 3B} around $500$ GPU hours for a total of around $4500$ GPU hours. This was spread across $8$ NVIDIA A$100$ PCIe GPUs with $40$ GBs of HBM$2$ memory. Total time of execution was around $3$ weeks of run-time across all $8$ GPUs.

{\em A note on the execution time of the rotor gadget}. 
%The focus of this paper is not on speed or memory efficiency, but rather feasibility. 
We have not optimized the current code for speed or memory. At inference time, rotor layers are implemented as dense matrix multiplications. Therefore, its runtime scales with the depth of the rotor layer used in the replacement. In our experiments, we used hyperparameters of at most depth $2$ and width $3$; with an optimized implementation, this will result in at most $2\times$ slowdown, since width is trivially parallelizable. 

It is important to note that this corresponds to a direct implementation---we currently do not leverage the block sparsity structure described in Section~\ref{apdx:techAnalysis}, as current software support is quite limited. Custom kernels 
have recently become available to \textit{batch} matrix-matrix multiplication of different dimensions, such as those available in \texttt{cublasGemmGroupedBatchedEx}, but our setting requires vector-matrix multiplication. We found that this is slower than performing the dense implementation, though performance is expected to improve significantly once specialized vector-matrix versions become available. While custom kernels such as \texttt{cublasSgemmStridedBatched} support batching matrix-matrix multiplication of the {\em same} dimension, their adaptation to rotors remains limited: even when applied to the parallel portions of the rotor gadget, they must operate sequentially on each block of the block sparse matrices. As a result, this gives a modest speedup, and there remains substantial room for further optimization.
\section{Additional Experimental Results}
\label{apdx:additional_exp}

\begin{figure}[t]
  \centering
\begin{tikzpicture}
  \begin{axis}[
    width=\textwidth,
    height=0.5\textwidth,
    xlabel={\# Gradient Update Steps},
    ylabel={\# Iterations to Converge},
    ymin=1, ymax=8,
    xtick={2,4,6,8,10},
    ytick={2,4,6,8},
    mark size=2pt,  % default is around 1.5pt
    grid=both,
    tick label style={font=\normalsize},
    label style={font=\normalsize},
    legend style={
      at={(0.5,1.05)},
      anchor=south,
      font=\large,
      draw=none,
      fill=none,
      legend columns=4,
      column sep=12pt,
    },
    error bars/y dir=both,
    error bars/y explicit,
    line width=0.3pt,
    axis line style={line width=0.3pt, gray},
    % error bar style={opacity=0.5, line width=0.3pt},
  ]

% dim = 64
\addplot+[blue!0!red, mark=*, line width=1.4pt, opacity=0.5, error bars/.cd, y dir=both, y explicit]
coordinates {
  (1,4.340) +- (0,0.501)
  (2,4.170) +- (0,0.362)
  (3,4.020) +- (0,0.279)
  (4,3.790) +- (0,0.232)
  (5,3.820) +- (0,0.264)
  (6,3.760) +- (0,0.325)
  (7,3.580) +- (0,0.365)
  (8,3.220) +- (0,0.339)
  (9,2.690) +- (0,0.259)
  (10,1.660) +- (0,0.138)
};\addlegendentry{\texttt{dim=64}}

% dim = 128
\addplot+[blue!20!red, mark=square*, line width=1.4pt, opacity=0.5,error bars/.cd, y dir=both, y explicit]
coordinates {
  (1,4.370) +- (0,0.224)
  (2,4.210) +- (0,0.206)
  (3,4.150) +- (0,0.195)
  (4,4.130) +- (0,0.189)
  (5,4.000) +- (0,0.196)
  (6,3.890) +- (0,0.194)
  (7,3.600) +- (0,0.180)
  (8,3.240) +- (0,0.155)
  (9,2.610) +- (0,0.108)
  (10,1.590) +- (0,0.059)
};\addlegendentry{\texttt{dim=128}}

% dim = 256
\addplot+[blue!40!red, mark=triangle*, line width=1.4pt, opacity=0.5,error bars/.cd, y dir=both, y explicit]
coordinates {
  (1,4.460) +- (0,0.194)
  (2,4.433) +- (0,0.201)
  (3,4.420) +- (0,0.239)
  (4,4.573) +- (0,0.328)
  (5,4.480) +- (0,0.310)
  (6,4.260) +- (0,0.259)
  (7,3.827) +- (0,0.223)
  (8,3.400) +- (0,0.178)
  (9,2.673) +- (0,0.125)
  (10,1.740) +- (0,0.062)
};\addlegendentry{\texttt{dim=256}}

% dim = 512
\addplot+[blue!50!red, mark=pentagon*, line width=1.4pt, opacity=0.5, error bars/.cd, y dir=both, y explicit]
coordinates {
  (1,5.067) +- (0,0.251)
  (2,5.227) +- (0,0.386)
  (3,5.293) +- (0,0.407)
  (4,5.000) +- (0,0.302)
  (5,4.880) +- (0,0.266)
  (6,5.053) +- (0,0.666)
  (7,4.527) +- (0,0.508)
  (8,3.633) +- (0,0.224)
  (9,2.853) +- (0,0.144)
  (10,1.787) +- (0,0.079)
};\addlegendentry{\texttt{dim=512}}

% dim = 1024
\addplot+[blue!60!red, mark=diamond*, line width=1.4pt, opacity=0.5, error bars/.cd, y dir=both, y explicit]
coordinates {
  (1,5.030) +- (0,0.194)
  (2,4.865) +- (0,0.171)
  (3,4.770) +- (0,0.167)
  (4,4.715) +- (0,0.184)
  (5,4.565) +- (0,0.186)
  (6,4.335) +- (0,0.184)
  (7,3.920) +- (0,0.168)
  (8,3.395) +- (0,0.137)
  (9,2.705) +- (0,0.094)
  (10,1.810) +- (0,0.048)
};\addlegendentry{\texttt{dim=1024}}

% dim = 2048
\addplot+[blue!80!red, mark=o, line width=1.4pt, opacity=0.5, solid, error bars/.cd, y dir=both, y explicit]
coordinates {
  (1,6.575) +- (0,0.360)
  (2,6.980) +- (0,0.603)
  (3,6.430) +- (0,0.340)
  (4,6.570) +- (0,0.502)
  (5,6.235) +- (0,0.422)
  (6,5.650) +- (0,0.332)
  (7,5.010) +- (0,0.277)
  (8,4.200) +- (0,0.225)
  (9,3.140) +- (0,0.155)
  (10,1.955) +- (0,0.067)
};\addlegendentry{\texttt{dim=2048}}

% dim = 4096
\addplot+[blue!100!red, mark=square, line width=1.4pt, opacity=0.5, solid, error bars/.cd, y dir=both, y explicit]
coordinates {
  (1,6.128) +- (0,0.266)
  (2,6.328) +- (0,0.363)
  (3,6.200) +- (0,0.302)
  (4,6.140) +- (0,0.356)
  (5,6.024) +- (0,0.397)
  (6,5.644) +- (0,0.381)
  (7,5.224) +- (0,0.462)
  (8,4.444) +- (0,0.428)
  (9,3.380) +- (0,0.346)
  (10,2.028) +- (0,0.103)
};\addlegendentry{\texttt{dim=4096}}

  \end{axis}
\end{tikzpicture}
\caption{Number of iterations required by Alg~\ref{algo:ga_poweriter} to converge within a tolerance of $\epsilon = 10^{-3}$, plotted against the number of gradient updates applied to the parameters of rotors (i.e., bivector coefficients).  Results are averaged over 50 runs and simple bivectors in the invariant decomposition, with one standard deviation shown as error bars. The results demonstrate that warm-starting with previously learned singular vectors significantly accelerates convergence.}
\label{fig:projection-error}
\end{figure}
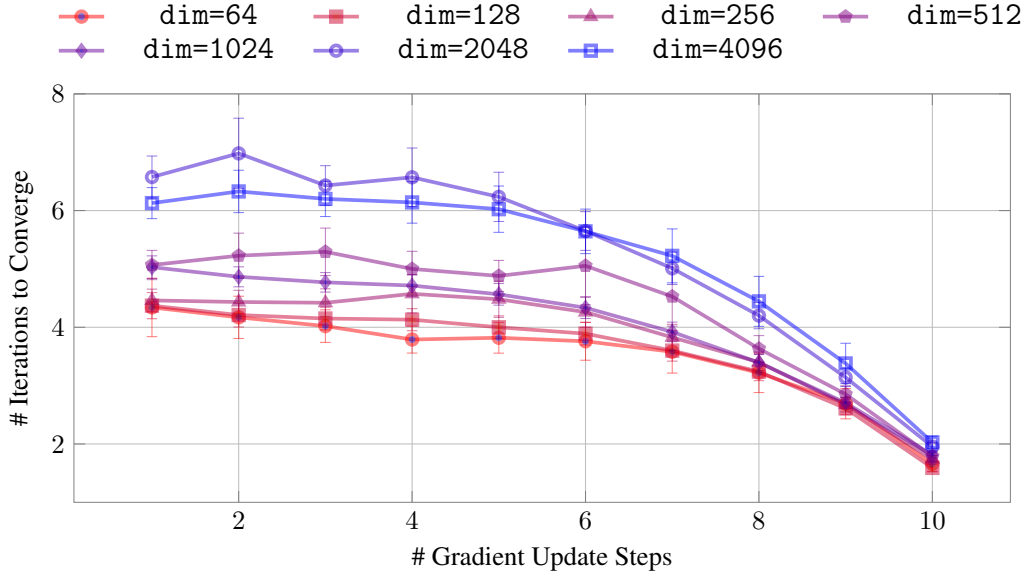
\begin{table*}[t]
\centering
\small
\setlength{\tabcolsep}{10pt}
\begin{tabular}{llcccc}
  \toprule
  Model & Method & \textbf{Key} & \textbf{Query} & \textbf{Value} & Total \\
  \midrule
  \multirow{5}{*}{\texttt{LLaMa-3.2 1B}}
      & Dense & 1,048,576 & 4,194,304 & 1,048,576 & 6,291,456 \\
      & LR1   & 2,560 & 4,096 & 2,560 & 9,216 \\
      & LR4   & 10,240 & 16,384 & 10,240 & 36,864 \\
      & BH1   & 8,192 & 32,768 & 8,192 & 49,152 \\
      & \cellcolor{gray!30}Rotor 
        & \cellcolor{gray!30}$\le$1,080 & \cellcolor{gray!30}$\le$896 & \cellcolor{gray!30}$\le$1,080 & \cellcolor{gray!30}$\le$3,056 \\
  \midrule
  \multirow{5}{*}{\texttt{LLaMa-3.2 3B}}
      & Dense & 3,145,728 & 9,437,184 & 3,145,728 & 15,728,640 \\
      & LR1   & 4,098 & 6,148 & 4,098 & 14,344 \\
      & LR4   & 16,384 & 24,576 & 16,384 & 57,344 \\
      & BH1   & 32,768 & 98,304 & 32,768 & 163,840 \\
      & \cellcolor{gray!30}Rotor 
        & \cellcolor{gray!30}$\le$1,080 & \cellcolor{gray!30}$\le$1,120 & \cellcolor{gray!30}$\le$1,080 & \cellcolor{gray!30}$\le$3,280 \\
  \midrule
  \multirow{5}{*}{\texttt{Qwen-2.5 1.5B}}
      & Dense & 393,216 & 2,359,296 & 393,216 & 3,145,728 \\
      & LR1   & 1,792 & 3,072 & 1,792 & 6,656 \\
      & LR4   & 7,168 & 12,288 & 7,168 & 26,624 \\
      & BH1   & 4,096 & 24,576 & 4,096 & 32,768 \\
      & \cellcolor{gray!30}Rotor 
        & \cellcolor{gray!30}$\le$904 & \cellcolor{gray!30}$\le$896 & \cellcolor{gray!30}$\le$904 & \cellcolor{gray!30}$\le$2,704 \\
  \midrule
  \multirow{5}{*}{\texttt{Fox-1.0 1.6B}}
      & Dense & 1,048,576 & 4,194,304 & 1,048,576 & 6,291,456 \\
      & LR1   & 2,560 & 4,096 & 2,560 & 9,216 \\
      & LR4   & 10,240 & 16,384 & 10,240 & 36,864 \\
      & BH1   & 8,192 & 32,768 & 8,192 & 49,152 \\
      & \cellcolor{gray!30}Rotor 
        & \cellcolor{gray!30}$\le$1,080 & \cellcolor{gray!30}$\le$896 & \cellcolor{gray!30}$\le$1,080 & \cellcolor{gray!30}$\le$3,056 \\
  \bottomrule
\end{tabular}
\vspace{2pt}
\caption{Number of parameters for key, query, and value projections in a single attention layer of each model, with the rightmost column showing their sum.}
\label{tab:appdx-param-cnt}
\end{table*}
\begin{table*}[t]
\scriptsize
\centering
\setlength{\tabcolsep}{1.8pt}
\renewcommand{\arraystretch}{0.95}
\begin{tabular}{
>{\centering\arraybackslash}m{1.5cm}
>{\centering\arraybackslash}m{1.0cm}
*{14}{>{\centering\arraybackslash}m{0.616cm}<{\fontsize{4}{2.5}\selectfont}}
}
\toprule
\textbf{Dataset} & \textbf{Method} & \multicolumn{14}{c}{\textbf{One layer replaced (Layer index)}} \\
& & \textbf{1} & \textbf{2} & \textbf{3} & \textbf{4} & \textbf{5} & \textbf{6} & \textbf{7}
  & \textbf{8} & \textbf{9} & \textbf{10} & \textbf{11} & \textbf{12} & \textbf{13} & \textbf{14}\\
\midrule
\multirow{4}{*}{Wikitext2 $(\downarrow)$}
& LR1   & 2.669 & 2.330  & 2.351 & 2.354 & 2.332 & 2.316 & 2.350  & 2.340  & 2.326 & 2.315 & 2.381 & 2.311 & 2.514 & 2.306  \\
& LR4   & 2.635 & 2.320  & 2.334 & 2.343 & 2.320  & 2.310  & 2.335 & 2.331 & 2.318 & 2.310  & 2.368 & 2.307 & 2.456 & 2.306    \\
& BH1   & 2.341 & 2.316 & 2.329 & 2.323 & 2.323 & 2.308 & 2.338 & 2.321 & 2.314 & 2.307 & 2.354 & 2.308 & 2.417 & 2.307\\
& \cellcolor{gray!30}Rotor 
        & \cellcolor{gray!30}2.312 & \cellcolor{gray!30}2.299 & \cellcolor{gray!30}2.305
        & \cellcolor{gray!30}2.312 & \cellcolor{gray!30}2.308 & \cellcolor{gray!30}2.298 & \cellcolor{gray!30}2.301
        & \cellcolor{gray!30}2.315 & \cellcolor{gray!30}2.303 & \cellcolor{gray!30}2.300 & \cellcolor{gray!30}2.321
        & \cellcolor{gray!30}2.304 & \cellcolor{gray!30}2.320 & \cellcolor{gray!30}2.301  \\
\noalign{\vskip 1.5pt}
\multirow{4}{*}{C4 $(\downarrow)$}
& LR1   & 3.129 & 2.880 & 2.891 & 2.867 & 2.863 & 2.846 & 2.875 & 2.862 & 2.870 & 2.854 & 2.891 & 2.854 & 2.992 & 2.848 \\
& LR4   & 3.007 & 2.880 & 2.890 & 2.869 & 2.862 & 2.846 & 2.873 & 2.862 & 2.867 & 2.854 & 2.888 & 2.853 & 2.939 & 2.848 \\
& BH1   & 2.943 & 2.872 & 2.878 & 2.864 & 2.856 & 2.844 & 2.870 & 2.858 & 2.863 & 2.849 & 2.879 & 2.851 & 2.925 & 2.847 \\
& \cellcolor{gray!30}Rotor 
        & \cellcolor{gray!30}2.892 & \cellcolor{gray!30}2.863 & \cellcolor{gray!30}2.861
        & \cellcolor{gray!30}2.856 & \cellcolor{gray!30}2.852 & \cellcolor{gray!30}2.841 & \cellcolor{gray!30}2.846
        & \cellcolor{gray!30}2.854 & \cellcolor{gray!30}2.852 & \cellcolor{gray!30}2.848 & \cellcolor{gray!30}2.861
        & \cellcolor{gray!30}2.848 & \cellcolor{gray!30}2.858 & \cellcolor{gray!30}2.845 \\
\noalign{\vskip 1.5pt}
\multirow{4}{*}{PTB $(\downarrow)$}
& LR1   & 3.115 & 3.034 & 3.063 & 3.073 & 3.011 & 3.023 & 3.081 & 3.031 & 3.033 & 3.023 & 3.102 & 3.019 & 3.173 & 3.018 \\
& LR4   & 3.110 & 3.021 & 3.063 & 3.055 & 3.006 & 3.015 & 3.076 & 3.030 & 3.028 & 3.012 & 3.078 & 3.018 & 3.150 & 3.017 \\
& BH1   & 3.091 & 3.007 & 3.046 & 3.044 & 3.004 & 3.007 & 3.076 & 3.020 & 3.022 & 3.007 & 3.064 & 3.015 & 3.101 & 3.012 \\
& \cellcolor{gray!30}Rotor 
        & \cellcolor{gray!30}3.045 & \cellcolor{gray!30}2.998 & \cellcolor{gray!30}3.018
        & \cellcolor{gray!30}3.025 & \cellcolor{gray!30}3.001 & \cellcolor{gray!30}2.998 & \cellcolor{gray!30}3.014
        & \cellcolor{gray!30}3.015 & \cellcolor{gray!30}3.008 & \cellcolor{gray!30}3.003 & \cellcolor{gray!30}3.028
        & \cellcolor{gray!30}3.012 & \cellcolor{gray!30}3.034 & \cellcolor{gray!30}3.009 \\
\midrule
\multirow{4}{*}{\makecell{Arc \\ Challenge $(\uparrow)$}}
& LR1   & 62.23 & 63.95 & 59.23 & 47.21 & 60.52 & 66.52 & 52.36 & 58.37 & 57.51 & 50.21 & 40.34 & 58.80 & 22.32 & 57.51 \\
& LR4   & 59.66 & 62.66 & 64.81 & 54.94 & 60.94 & 66.09 & 56.22 & 62.66 & 66.09 & 58.37 & 55.79 & 61.37 & 9.01  & 58.80 \\
& BH1   & 63.52 & 64.38 & 58.80 & 60.04 & 62.23 & 63.09 & 62.66 & 63.52 & 65.24 & 60.94 & 58.88 & 59.66 & 53.22 & 50.21 \\
& \cellcolor{gray!30}Rotor 
        & \cellcolor{gray!30}64.81 & \cellcolor{gray!30}67.38 & \cellcolor{gray!30}65.61
        & \cellcolor{gray!30}62.23 & \cellcolor{gray!30}62.66 & \cellcolor{gray!30}64.38 & \cellcolor{gray!30}61.80
        & \cellcolor{gray!30}63.95 & \cellcolor{gray!30}65.24 & \cellcolor{gray!30}66.09 & \cellcolor{gray!30}62.23
        & \cellcolor{gray!30}64.81 & \cellcolor{gray!30}60.09 & \cellcolor{gray!30}56.22 \\
\multirow{4}{*}{HellaSwag $(\uparrow)$}
& LR1   & 36.00 & 47.00 & 45.00 & 46.67 & 50.00 & 51.00 & 50.33 & 44.00 & 17.33 & 36.00 & 46.00 & 52.67 & 31.33 & 20.00 \\
& LR4   & 38.00 & 54.00 & 46.33 & 48.00 & 51.00 & 55.00 & 56.33 & 51.00 & 26.00 & 43.67 & 49.00 & 54.33 & 31.33 & 20.00 \\
& BH1   & 45.67 & 48.00 & 46.00 & 48.33 & 48.67 & 50.67 & 51.67 & 52.33 & 38.33 & 44.67 & 48.00 & 51.33 & 24.33 & 32.67 \\
& \cellcolor{gray!30}Rotor 
        & \cellcolor{gray!30}50.67 & \cellcolor{gray!30}53.00 & \cellcolor{gray!30}53.67
        & \cellcolor{gray!30}52.33 & \cellcolor{gray!30}52.00 & \cellcolor{gray!30}55.33 & \cellcolor{gray!30}57.00
        & \cellcolor{gray!30}54.67 & \cellcolor{gray!30}48.67 & \cellcolor{gray!30}55.33 & \cellcolor{gray!30}51.00
        & \cellcolor{gray!30}57.00 & \cellcolor{gray!30}51.33 & \cellcolor{gray!30}32.00 \\
\bottomrule
\end{tabular}

\vspace{5pt}

\begin{tabular}{
>{\centering\arraybackslash}m{1.5cm}
>{\centering\arraybackslash}m{1.0cm}
*{13}{>{\centering\arraybackslash}m{0.616cm}<{\fontsize{4}{2.5}\selectfont}}
}
\toprule
\textbf{Dataset} & \textbf{Method} & \multicolumn{13}{c}{\textbf{One layer replaced (Layer index)}} \\
& & \textbf{15} & \textbf{16} & \textbf{17} & \textbf{18} & \textbf{19} & \textbf{20} & \textbf{21}
  & \textbf{22} & \textbf{23} & \textbf{24} & \textbf{25} & \textbf{26} & \textbf{27}\\
\midrule
\multirow{4}{*}{Wikitext2 $(\downarrow)$}
& LR1   & 2.317 & 2.304 & 2.314 & 2.317 & 2.323 & 2.316 & 2.331 & 2.316 & 2.309 & 2.301 & 2.409 & 2.300 & 2.316\\
& LR4  & 2.317 & 2.304 & 2.314 & 2.317 & 2.323 & 2.316 & 2.331 & 2.316 & 2.309 & 2.301 & 2.409 & 2.300 & 2.316 \\
& BH1   & 2.313 & 2.301 & 2.314 & 2.316 & 2.320 & 2.311 & 2.311 & 2.315 & 2.309 & 2.300 & 2.398 & 2.297 & 2.309  \\
& \cellcolor{gray!30}Rotor 
        & \cellcolor{gray!30}2.307 & \cellcolor{gray!30}2.299 & \cellcolor{gray!30}2.302 & \cellcolor{gray!30}2.309
        & \cellcolor{gray!30}2.308 & \cellcolor{gray!30}2.307 & \cellcolor{gray!30}2.306 & \cellcolor{gray!30}2.311
        & \cellcolor{gray!30}2.304 & \cellcolor{gray!30}2.296 & \cellcolor{gray!30}2.358 & \cellcolor{gray!30}2.295
        & \cellcolor{gray!30}2.308  \\
\noalign{\vskip 1.5pt}
\multirow{4}{*}{C4 $(\downarrow)$}
& LR1   & 2.866 & 2.854 & 2.900 & 2.860 & 2.869 & 2.866 & 2.898 & 2.883 & 2.861 & 2.866 & 2.872 & 2.854 & 2.860\\
& LR4   & 2.865 & 2.851 & 2.899 & 2.859 & 2.868 & 2.866 & 2.889 & 2.866 & 2.856 & 2.851 & 2.858 & 2.853 & 2.858\\
& BH1   & 2.862 & 2.847 & 2.856 & 2.858 & 2.867 & 2.862 & 2.872 & 2.857 & 2.854 & 2.849 & 2.852 & 2.850 & 2.853 \\
& \cellcolor{gray!30}Rotor 
        & \cellcolor{gray!30}2.855 & \cellcolor{gray!30}2.844 & \cellcolor{gray!30}2.849
        & \cellcolor{gray!30}2.851 & \cellcolor{gray!30}2.863 & \cellcolor{gray!30}2.858 & \cellcolor{gray!30}2.862
        & \cellcolor{gray!30}2.852 & \cellcolor{gray!30}2.851 & \cellcolor{gray!30}2.846 & \cellcolor{gray!30}2.851
        & \cellcolor{gray!30}2.849 & \cellcolor{gray!30}2.853 \\
\noalign{\vskip 1.5pt}
\multirow{4}{*}{PTB $(\downarrow)$}
& LR1   & 3.025 & 3.003 & 3.020 & 3.021 & 3.036 & 3.019 & 3.086 & 3.010 & 3.014 & 3.012 & 3.019 & 3.061 & 3.098 \\
& LR4   & 3.024 & 3.003 & 3.026 & 3.018 & 3.034 & 3.017 & 3.017 & 3.010 & 3.015 & 2.997 & 3.017 & 3.012 & 3.024 \\
& BH1   & 3.017 & 3.000 & 3.008 & 3.014 & 3.029 & 3.012 & 3.010 & 3.009 & 3.011 & 2.995 & 3.015 & 3.005 & 3.016 \\
& \cellcolor{gray!30}Rotor 
        & \cellcolor{gray!30}3.009 & \cellcolor{gray!30}2.998 & \cellcolor{gray!30}3.003
        & \cellcolor{gray!30}3.006 & \cellcolor{gray!30}3.019 & \cellcolor{gray!30}3.008 & \cellcolor{gray!30}3.008
        & \cellcolor{gray!30}3.001 & \cellcolor{gray!30}3.007 & \cellcolor{gray!30}2.993 & \cellcolor{gray!30}3.012
        & \cellcolor{gray!30}3.001 & \cellcolor{gray!30}3.011 \\
\midrule
\multirow{4}{*}{\makecell{Arc \\ Challenge $(\uparrow)$}}
& LR1   & 41.20 & 65.24 & 27.90 & 65.24 & 68.24 & 16.31 & 47.64 & 67.38 & 66.95 & 65.67 & 65.24 & 66.09 & 66.52 \\
& LR4   & 31.76 & 69.1 & 58.37 & 65.24 & 68.24 & 16.31 & 47.64 & 67.38 & 66.95 & 65.67 & 65.24 & 66.09 & 66.52 \\
& BH1   & 49.79 & 68.67 & 48.50 & 62.23 & 62.24 & 42.49 & 48.93 & 65.24 & 67.81 & 64.39 & 65.31 & 65.79 & 65.24 \\
& \cellcolor{gray!30}Rotor 
        & \cellcolor{gray!30}56.65 & \cellcolor{gray!30}70.82 & \cellcolor{gray!30}61.37
        & \cellcolor{gray!30}60.09 & \cellcolor{gray!30}69.53 & \cellcolor{gray!30}51.07 & \cellcolor{gray!30}47.21
        & \cellcolor{gray!30}67.81 & \cellcolor{gray!30}66.52 & \cellcolor{gray!30}66.09 & \cellcolor{gray!30}65.67
        & \cellcolor{gray!30}66.09 & \cellcolor{gray!30}65.24\\
\noalign{\vskip 1.5pt}
\multirow{4}{*}{HellaSwag $(\uparrow)$}
& LR1   & 28.67 & 49.67 & 32.67 & 45.67 & 51.00 & 5.67 & 36.33 & 52.67 & 57.33 & 53.33 & 53.67 & 53.67 & 54.67 \\
& LR4   & 28.67 & 49.67 & 32.67 & 45.67 & 51.00 & 5.67 & 36.33 & 52.67 & 57.33 & 53.33 & 53.67 & 53.67 & 54.67 \\
& BH1   & 33.67 & 56.67 & 45.67 & 45.67 & 53.33 & 22.67 & 37.67 & 53.67 & 53.67 & 47.67 & 52.33 & 56.00 & 51.67 \\
& \cellcolor{gray!30}Rotor 
        & \cellcolor{gray!30}42.67 & \cellcolor{gray!30}55.00 & \cellcolor{gray!30}49.33
        & \cellcolor{gray!30}51.33 & \cellcolor{gray!30}52.67 & \cellcolor{gray!30}18.67 & \cellcolor{gray!30}35.00
        & \cellcolor{gray!30}51.00 & \cellcolor{gray!30}56.67 & \cellcolor{gray!30}54.00 & \cellcolor{gray!30}54.33
        & \cellcolor{gray!30}54.00 & \cellcolor{gray!30}56.67 \\
\bottomrule
\end{tabular}
\caption{\small{Performance on log-PPL (↓) and accuracy (↑) when replacing \textbf{one attention layer} for layer indices $1$–$27$ of \texttt{Qwen-2.5 1.5B}. Methods are Low-Rank ($r=1$ and $4$), BH1, and Rotor. Original log-PPL and accuracy are: $\texttt{Wikitext2 } 2.287, \texttt{C4 }2.834, \texttt{PTB } 2.985, \texttt{Arc Challenge }66.09, \texttt{Hellaswag }55.00$.}}
\label{tab:qwen-1.5-one-layer}
\end{table*}
\begin{table*}[t]
\scriptsize
\centering
\setlength{\tabcolsep}{1.8pt}
\renewcommand{\arraystretch}{0.95}
\begin{tabular}{
>{\centering\arraybackslash}m{1.5cm}
>{\centering\arraybackslash}m{1.0cm}
*{14}{>{\centering\arraybackslash}m{0.616cm}<{\fontsize{4}{2.5}\selectfont}}
}
\toprule
\textbf{Dataset} & \textbf{Method} & \multicolumn{14}{c}{\textbf{One layer replaced (Layer index)}} \\
& & \textbf{1} & \textbf{2} & \textbf{3} & \textbf{4} & \textbf{5} & \textbf{6} & \textbf{7}
  & \textbf{8} & \textbf{9} & \textbf{10} & \textbf{11} & \textbf{12} & \textbf{13} & \textbf{14}\\
\midrule
\multirow{4}{*}{Wikitext2 $(\downarrow)$}
& LR1   & 3.867 & 2.539 & 2.525 & 2.521 & 2.517 & 2.509 & 2.557 & 2.499 & 2.507 & 2.519 & 2.510 & 2.505 & 2.550 & 2.519 \\
& LR4   & 3.389 & 2.514 & 2.512 & 2.501 & 2.510 & 2.492 & 2.531 & 2.488 & 2.494 & 2.497 & 2.501 & 2.500 & 2.510 & 2.503 \\
& BH1   & 2.552 & 2.505 & 2.492 & 2.502 & 2.502 & 2.492 & 2.516 & 2.487 & 2.486 & 2.497 & 2.493 & 2.488 & 2.505 & 2.494 \\
& \cellcolor{gray!30}Rotor 
        & \cellcolor{gray!30}2.594 & \cellcolor{gray!30}2.520 & \cellcolor{gray!30}2.497 & \cellcolor{gray!30}2.499
        & \cellcolor{gray!30}2.517 & \cellcolor{gray!30}2.498 & \cellcolor{gray!30}2.521 & \cellcolor{gray!30}2.488
        & \cellcolor{gray!30}2.492 & \cellcolor{gray!30}2.511 & \cellcolor{gray!30}2.499 & \cellcolor{gray!30}2.492
        & \cellcolor{gray!30}2.506 & \cellcolor{gray!30}2.510 \\
\bottomrule
\end{tabular}

\vspace{5pt}

\begin{tabular}{
>{\centering\arraybackslash}m{1.5cm}
>{\centering\arraybackslash}m{1.0cm}
*{13}{>{\centering\arraybackslash}m{0.616cm}<{\fontsize{4}{2.5}\selectfont}}
}
\toprule
\textbf{Dataset} & \textbf{Method} & \multicolumn{13}{c}{\textbf{One layer replaced (Layer index)}} \\
& & \textbf{15} & \textbf{16} & \textbf{17} & \textbf{18} & \textbf{19} & \textbf{20} & \textbf{21}
  & \textbf{22} & \textbf{23} & \textbf{24} & \textbf{25} & \textbf{26} & \textbf{27}\\
\midrule
\multirow{4}{*}{Wikitext2 $(\downarrow)$}
& LR1   & 2.501 & 2.487 & 2.481 & 2.473 & 2.487 & 2.484 & 2.477 & 2.479 & 2.487 & 2.489 & 2.490 & 2.479 & 2.519  \\
& LR4   & 2.497 & 2.485 & 2.477 & 2.471 & 2.480 & 2.479 & 2.471 & 2.467 & 2.480 & 2.474 & 2.471 & 2.471 & 2.503 \\
& BH1   & 2.484 & 2.479 & 2.476 & 2.471 & 2.472 & 2.475 & 2.466 & 2.469 & 2.477 & 2.476 & 2.475 & 2.469 & 2.482  \\
& \cellcolor{gray!30}Rotor 
        & \cellcolor{gray!30}2.488 & \cellcolor{gray!30}2.481 & \cellcolor{gray!30}2.477 & \cellcolor{gray!30}2.471
        & \cellcolor{gray!30}2.476 & \cellcolor{gray!30}2.478 & \cellcolor{gray!30}2.466 & \cellcolor{gray!30}2.471
        & \cellcolor{gray!30}2.482 & \cellcolor{gray!30}2.484 & \cellcolor{gray!30}2.482 & \cellcolor{gray!30}2.477
        & \cellcolor{gray!30}2.496  \\
\bottomrule
\end{tabular}
\caption{\small{Performance on log-PPL (↓) and accuracy (↑) when replacing \textbf{one attention layer} for layer indices $1$–$27$ of \texttt{LLaMa-3.2 3B}. Methods are Low-Rank ($r=1$ and $4$), BH1, and Rotor. Original log-PPL is $2.460$.}}
\label{tab:llama3b-one-layer}
\end{table*}
\begin{table*}[t]
\scriptsize
\centering
\setlength{\tabcolsep}{1.8pt}
\renewcommand{\arraystretch}{0.95}
\begin{tabular}{
>{\centering\arraybackslash}m{1.5cm}
>{\centering\arraybackslash}m{1.0cm}
*{16}{>{\centering\arraybackslash}m{0.58cm}<{\fontsize{4}{2.5}\selectfont}}
}
\toprule
\textbf{Dataset} & \textbf{Method} & \multicolumn{14}{c}{\textbf{One layer replaced (Layer index)}} \\
& & \textbf{1} & \textbf{2} & \textbf{3} & \textbf{4} & \textbf{5} & \textbf{6} & \textbf{7}
  & \textbf{8} & \textbf{9} & \textbf{10} & \textbf{11} & \textbf{12} & \textbf{13} & \textbf{14}  & \textbf{15} & \textbf{16}\\
\midrule
\multirow{4}{*}{Wikitext2 $(\downarrow)$}
& LR1   & 2.645 & 2.573 & 2.549 & 2.535 & 2.531 & 2.546 & 2.550 & 2.533 & 2.581 & 2.540 & 2.557 & 2.546 & 2.538 & 2.538 & 2.536 & 2.546 \\
& LR4   & 2.614 & 2.568 & 2.538 & 2.535 & 2.530 & 2.544 & 2.545 & 2.531 & 2.570 & 2.538 & 2.552 & 2.535 & 2.538 & 2.538 & 2.536 & 2.534 \\
& BH1   & 2.633 & 2.548 & 2.545 & 2.534 & 2.530 & 2.545 & 2.544 & 2.531 & 2.577 & 2.536 & 2.550 & 2.538 & 2.537 & 2.537 & 2.534 & 2.532 \\
& \cellcolor{gray!30}Rotor 
        & \cellcolor{gray!30}2.544 & \cellcolor{gray!30}2.542 & \cellcolor{gray!30}2.541 & \cellcolor{gray!30}2.532
        & \cellcolor{gray!30}2.529 & \cellcolor{gray!30}2.538 & \cellcolor{gray!30}2.540 & \cellcolor{gray!30}2.530
        & \cellcolor{gray!30}2.537 & \cellcolor{gray!30}2.533 & \cellcolor{gray!30}2.549 & \cellcolor{gray!30}2.536
        & \cellcolor{gray!30}2.534 & \cellcolor{gray!30}2.531 & \cellcolor{gray!30}2.531 & \cellcolor{gray!30}2.531 \\
\noalign{\vskip 1.5pt}
\multirow{4}{*}{C4 $(\downarrow)$}
& LR1   & 2.997 & 2.915 & 2.879 & 2.877 & 2.877 & 2.891 & 2.884 & 2.878 & 2.889 & 2.886 & 2.876 & 2.884 & 2.874 & 2.881 & 2.878 & 2.875 \\
& LR4   & 2.960 & 2.914 & 2.880 & 2.877 & 2.877 & 2.891 & 2.884 & 2.878 & 2.888 & 2.885 & 2.877 & 2.883 & 2.874 & 2.881 & 2.878 & 2.874 \\
& BH1   & 2.991 & 2.901 & 2.875 & 2.878 & 2.876 & 2.891 & 2.881 & 2.878 & 2.889 & 2.879 & 2.874 & 2.876 & 2.872 & 2.879 & 2.877 & 2.872 \\
& \cellcolor{gray!30}Rotor 
        & \cellcolor{gray!30}2.924 & \cellcolor{gray!30}2.901 & \cellcolor{gray!30}2.875 & \cellcolor{gray!30}2.876
        & \cellcolor{gray!30}2.876 & \cellcolor{gray!30}2.890 & \cellcolor{gray!30}2.881 & \cellcolor{gray!30}2.878
        & \cellcolor{gray!30}2.882 & \cellcolor{gray!30}2.878 & \cellcolor{gray!30}2.873 & \cellcolor{gray!30}2.876
        & \cellcolor{gray!30}2.872 & \cellcolor{gray!30}2.878 & \cellcolor{gray!30}2.876 & \cellcolor{gray!30}2.871 \\
\noalign{\vskip 1.5pt}
\multirow{4}{*}{PTB $(\downarrow)$}
& LR1   & 3.359 & 3.270 & 3.222 & 3.231 & 3.226 & 3.255 & 3.222 & 3.231 & 3.279 & 3.233 & 3.229 & 3.268 & 3.232 & 3.237 & 3.231 & 3.279 \\
& LR4   & 3.317 & 3.259 & 3.215 & 3.231 & 3.225 & 3.260 & 3.224 & 3.230 & 3.276 & 3.232 & 3.229 & 3.235 & 3.230 & 3.242 & 3.229 & 3.226 \\
& BH1   & 3.318 & 3.230 & 3.214 & 3.230 & 3.224 & 3.259 & 3.224 & 3.229 & 3.277 & 3.231 & 3.228 & 3.234 & 3.230 & 3.241 & 3.227 & 3.227 \\
& \cellcolor{gray!30}Rotor 
        & \cellcolor{gray!30}3.283 & \cellcolor{gray!30}3.229 & \cellcolor{gray!30}3.212 & \cellcolor{gray!30}3.229
        & \cellcolor{gray!30}3.224 & \cellcolor{gray!30}3.257 & \cellcolor{gray!30}3.223 & \cellcolor{gray!30}3.228
        & \cellcolor{gray!30}3.255 & \cellcolor{gray!30}3.229 & \cellcolor{gray!30}3.228 & \cellcolor{gray!30}3.232
        & \cellcolor{gray!30}3.228 & \cellcolor{gray!30}3.234 & \cellcolor{gray!30}3.225 & \cellcolor{gray!30}3.225 \\
\midrule
\multirow{4}{*}{\makecell{Arc \\ Challenge $(\uparrow)$}}
& LR1   & 31.76 & 31.33 & 33.05 & 42.06 & 42.49 & 47.64 & 36.48 & 44.64 & 22.75 & 41.63 & 24.03 & 39.48 & 39.91 & 20.17 & 21.03 & 29.18\\
& LR4   & 27.47 & 26.61 & 27.90 & 36.91 & 39.06 & 48.50 & 36.05 & 41.20 & 15.45 & 35.19 & 18.88 & 31.33 & 42.06 & 16.74 & 21.03 & 23.61\\
& BH1   & 34.33 & 21.46 & 28.33 & 32.19 & 38.20 & 45.06 & 26.61 & 42.49 & 11.59 & 36.48 & 14.16 & 21.89 & 46.78 & 24.89 & 25.75 & 9.44\\
& \cellcolor{gray!30}Rotor 
        & \cellcolor{gray!30}27.04 & \cellcolor{gray!30}21.03 & \cellcolor{gray!30}26.18 & \cellcolor{gray!30}23.61
        & \cellcolor{gray!30}30.47 & \cellcolor{gray!30}36.48 & \cellcolor{gray!30}29.18 & \cellcolor{gray!30}39.91
        & \cellcolor{gray!30}22.32 & \cellcolor{gray!30}31.33 & \cellcolor{gray!30}13.30 & \cellcolor{gray!30}22.32
        & \cellcolor{gray!30}31.76 & \cellcolor{gray!30}28.33 & \cellcolor{gray!30}24.46 & \cellcolor{gray!30}18.45 \\
\multirow{4}{*}{HellaSwag $(\uparrow)$}
& LR1   & 31.33 & 37.00 & 35.33 & 34.33 & 44.00 & 44.67 & 42.67 & 45.33 & 8.33 & 17.33 & 5.00 & 28.67 & 31.33 & 42.67 & 23.33 & 39.00 \\
& LR4   & 35.67 & 36.67 & 25.00 & 33.00 & 42.67 & 43.33 & 43.00 & 44.67 & 6.33 & 19.00 & 4.33 & 26.67 & 31.33 & 41.33 & 18.67 & 39.33 \\
& BH1   & 40.00 & 34.00 & 24.67 & 35.33 & 44.33 & 45.00 & 42.67 & 45.33 & 9.67 & 19.67 & 5.33 & 22.67 & 36.00 & 40.00 & 29.67 & 33.33 \\
& \cellcolor{gray!30}Rotor 
        & \cellcolor{gray!30}40.00 & \cellcolor{gray!30}37.33 & \cellcolor{gray!30}28.33 & \cellcolor{gray!30}36.67
        & \cellcolor{gray!30}44.67 & \cellcolor{gray!30}44.33 & \cellcolor{gray!30}41.00 & \cellcolor{gray!30}44.33
        & \cellcolor{gray!30}17.33 & \cellcolor{gray!30}32.33 & \cellcolor{gray!30}18.33 & \cellcolor{gray!30}35.67
        & \cellcolor{gray!30}34.00 & \cellcolor{gray!30}42.67 & \cellcolor{gray!30}26.33 & \cellcolor{gray!30}38.33 \\
\bottomrule
\end{tabular}

\vspace{5pt}

\begin{tabular}{
>{\centering\arraybackslash}m{1.5cm}
>{\centering\arraybackslash}m{1.0cm}
*{15}{>{\centering\arraybackslash}m{0.616cm}<{\fontsize{4}{2.5}\selectfont}}
}
\toprule
\textbf{Dataset} & \textbf{Method} & \multicolumn{13}{c}{\textbf{One layer replaced (Layer index)}} \\
& & \textbf{17} & \textbf{18} & \textbf{19} & \textbf{20} & \textbf{21}
  & \textbf{22} & \textbf{23} & \textbf{24} & \textbf{25} & \textbf{26} & \textbf{27} & \textbf{28} & \textbf{29} & \textbf{30} & \textbf{31}\\
\midrule
\multirow{4}{*}{Wikitext2 $(\downarrow)$}
& LR1   & 2.536 & 2.530 & 2.564 & 2.553 & 2.534 & 2.544 & 2.546 & 2.523 & 2.537 & 2.522 & 2.535 & 2.617 & 2.540 & 2.544 & 2.571 \\
& LR4   & 2.534 & 2.528 & 2.563 & 2.553 & 2.534 & 2.543 & 2.541 & 2.523 & 2.536 & 2.523 & 2.534 & 2.558 & 2.531 & 2.543 & 2.541 \\
& BH1   & 2.532 & 2.528 & 2.558 & 2.550 & 2.530 & 2.538 & 2.537 & 2.521 & 2.530 & 2.521 & 2.532 & 2.556 & 2.529 & 2.539 & 2.532 \\
& \cellcolor{gray!30}Rotor 
        & \cellcolor{gray!30}2.531 & \cellcolor{gray!30}2.527 & \cellcolor{gray!30}2.531 & \cellcolor{gray!30}2.548
        & \cellcolor{gray!30}2.529 & \cellcolor{gray!30}2.538 & \cellcolor{gray!30}2.534 & \cellcolor{gray!30}2.521
        & \cellcolor{gray!30}2.529 & \cellcolor{gray!30}2.521 & \cellcolor{gray!30}2.531 & \cellcolor{gray!30}2.553
        & \cellcolor{gray!30}2.528 & \cellcolor{gray!30}2.537 & \cellcolor{gray!30}2.531 \\
\noalign{\vskip 1.5pt}
\multirow{4}{*}{C4 $(\downarrow)$}
& LR1   & 2.879 & 2.878 & 2.891 & 2.884 & 2.879 & 2.888 & 2.879 & 2.872 & 2.885 & 2.872 & 2.875 & 2.875 & 2.870 & 2.876 & 2.904 \\
& LR4   & 2.879 & 2.877 & 2.891 & 2.884 & 2.880 & 2.889 & 2.879 & 2.872 & 2.884 & 2.872 & 2.874 & 2.875 & 2.870 & 2.875 & 2.890 \\
& BH1   & 2.878 & 2.876 & 2.883 & 2.883 & 2.878 & 2.886 & 2.878 & 2.871 & 2.879 & 2.870 & 2.872 & 2.873 & 2.868 & 2.874 & 2.886 \\
& \cellcolor{gray!30}Rotor 
        & \cellcolor{gray!30}2.877 & \cellcolor{gray!30}2.875 & \cellcolor{gray!30}2.881 & \cellcolor{gray!30}2.883
        & \cellcolor{gray!30}2.878 & \cellcolor{gray!30}2.885 & \cellcolor{gray!30}2.877 & \cellcolor{gray!30}2.871
        & \cellcolor{gray!30}2.879 & \cellcolor{gray!30}2.870 & \cellcolor{gray!30}2.872 & \cellcolor{gray!30}2.874
        & \cellcolor{gray!30}2.868 & \cellcolor{gray!30}2.874 & \cellcolor{gray!30}2.883 \\
\noalign{\vskip 1.5pt}
\multirow{4}{*}{PTB $(\downarrow)$}
& LR1   & 3.229 & 3.233 & 3.253 & 3.246 & 3.220 & 3.237 & 3.227 & 3.218 & 3.248 & 3.221 & 3.219 & 3.249 & 3.220 & 3.242 & 3.264 \\
& LR4   & 3.224 & 3.221 & 3.239 & 3.240 & 3.214 & 3.226 & 3.222 & 3.213 & 3.230 & 3.219 & 3.218 & 3.222 & 3.212 & 3.233 & 3.235 \\
& BH1   & 3.222 & 3.222 & 3.238 & 3.239 & 3.214 & 3.226 & 3.222 & 3.210 & 3.230 & 3.217 & 3.217 & 3.222 & 3.212 & 3.221 & 3.226 \\
& \cellcolor{gray!30}Rotor 
        & \cellcolor{gray!30}3.221 & \cellcolor{gray!30}3.220 & \cellcolor{gray!30}3.232 & \cellcolor{gray!30}3.238
        & \cellcolor{gray!30}3.213 & \cellcolor{gray!30}3.225 & \cellcolor{gray!30}3.220 & \cellcolor{gray!30}3.210
        & \cellcolor{gray!30}3.228 & \cellcolor{gray!30}3.216 & \cellcolor{gray!30}3.217 & \cellcolor{gray!30}3.220
        & \cellcolor{gray!30}3.211 & \cellcolor{gray!30}3.220 & \cellcolor{gray!30}3.223 \\
\midrule
\multirow{4}{*}{\makecell{Arc \\ Challenge $(\uparrow)$}}
& LR1   & 4.72 & 6.87 & 39.91 & 44.21 & 18.45 & 33.91 & 34.33 & 21.46 & 7.30 & 27.47 & 18.88 & 29.18 & 47.64 & 10.30 & 7.73 \\
& LR4   & 6.87 & 8.15 & 38.63 & 41.63 & 18.88 & 30.90 & 34.33 & 21.46 & 8.58 & 27.04 & 18.88 & 31.76 & 48.07 & 12.45 & 13.73 \\
& BH1   & 10.30 & 5.58 & 35.62 & 42.49 & 20.17 & 30.04 & 32.19 & 24.03 & 9.87 & 25.32 & 19.31 & 25.32 & 46.78 & 18.88 & 24.46 \\
& \cellcolor{gray!30}Rotor 
        & \cellcolor{gray!30}20.17 & \cellcolor{gray!30}18.88 & \cellcolor{gray!30}25.32 & \cellcolor{gray!30}33.91
        & \cellcolor{gray!30}21.03 & \cellcolor{gray!30}30.04 & \cellcolor{gray!30}26.61 & \cellcolor{gray!30}24.46
        & \cellcolor{gray!30}11.59 & \cellcolor{gray!30}26.18 & \cellcolor{gray!30}21.46 & \cellcolor{gray!30}24.89
        & \cellcolor{gray!30}45.49 & \cellcolor{gray!30}19.74 & \cellcolor{gray!30}24.46 \\
\noalign{\vskip 1.5pt}
\multirow{4}{*}{HellaSwag $(\uparrow)$}
& LR1   & 27.33 & 26.33 & 42.33 & 47.33 & 41.33 & 37.00 & 37.67 & 32.00 & 23.67 & 40.00 & 34.67 & 46.67 & 47.33 & 28.67 & 12.00 \\
& LR4   & 25.00 & 28.33 & 41.00 & 46.33 & 40.00 & 37.00 & 38.67 & 32.33 & 26.67 & 40.67 & 36.33 & 44.33 & 46.00 & 31.33 & 28.00 \\
& BH1   & 28.67 & 27.67 & 42.33 & 47.67 & 44.00 & 40.67 & 38.33 & 34.67 & 27.33 & 39.67 & 36.67 & 42.33 & 47.00 & 32.33 & 36.67 \\
& \cellcolor{gray!30}Rotor 
        & \cellcolor{gray!30}33.67 & \cellcolor{gray!30}34.00 & \cellcolor{gray!30}33.00 & \cellcolor{gray!30}34.67
        & \cellcolor{gray!30}37.67 & \cellcolor{gray!30}41.00 & \cellcolor{gray!30}38.33 & \cellcolor{gray!30}34.33
        & \cellcolor{gray!30}29.67 & \cellcolor{gray!30}39.67 & \cellcolor{gray!30}33.67 & \cellcolor{gray!30}40.67
        & \cellcolor{gray!30}46.33 & \cellcolor{gray!30}31.67 & \cellcolor{gray!30}38.67 \\
\bottomrule
\end{tabular}
\caption{\small{Performance on log-PPL (↓) and accuracy (↑) when replacing \textbf{one attention layer} for layer indices $1$–$31$ results of \texttt{Fox-1.0 1.6B}. Methods are Low-Rank ($r=1$ and 4), BH1, and Rotor. Original log-PPL and accuracy are: $\texttt{Wikitext2 } 2.517, \texttt{C4 }2.862, \texttt{PTB } 3.205, \texttt{Arc Challenge }24.89, \texttt{Hellaswag }38.33$.}}
\label{tab:fox=1.6-one-layer}
\end{table*}

In Fig.~\ref{fig:projection-error}, we report the number of iterations required by Alg.~\ref{algo:ga_poweriter} to converge within a tolerance of $\epsilon = 10^{-3}$, as a function of the number of gradient updates applied to the rotors (i.e., bivector coefficients). Synthetic data was generated with random input and having output come from the rotation corresponding to a random bivector. Each gradient update step is towards learning that random bivector with MSE loss. As expected, higher-dimensional projections require more iterations to converge. Notably, warm-starting from previously learned singular values significantly reduces the convergence speed across all dimensions, supporting our claim in Section~\ref{sec:compute-rotors}. 

In Tab.~\ref{tab:qwen-1.5-one-layer}, ~\ref{tab:llama3b-one-layer}, and ~\ref{tab:fox=1.6-one-layer}, we provide additional experimental results on \texttt{Qwen-2.5 1.5B}, \texttt{LLaMa-3.2 3B}, and \texttt{Fox-1.0 1.6B} \citep{hu2025fox1opensmalllanguage}, where a single attention layer is replaced by Rotor, LR1, LR4, or BH1 approximations. In Tab.~\ref{tab:fox_avg}, we provide averages for single and two-layer replacements for \texttt{Fox-1-1.6B}. The trends observed in the main paper persist: our rotor-based method consistently matches the baselines with significantly fewer parameters (see Tab.~\ref{tab:appdx-param-cnt}). These results further support our central claim: \textit{linear layers can be synthesized from a small number of geometric primitives encoding rotations}.

\begin{table}
\scriptsize
\centering
\setlength{\tabcolsep}{1.3pt}
\renewcommand{\arraystretch}{1}
\begin{tabular}{
>{\centering\arraybackslash}m{0.3cm}
>{\centering\arraybackslash}m{1.1cm}
>{\centering\arraybackslash}m{1.0cm}
>{\centering\arraybackslash}m{0.85cm}
>{\centering\arraybackslash}m{0.85cm} 
}
\toprule
& \textbf{Dataset} & \textbf{Method} 
& \multicolumn{2}{c}{\textbf{Fox-1.0 1.6B}} \\
& & 
& one & two \\
\midrule
\multirow{15}{*}{\makecell{\rotatebox{90}{Log-PPL}}}
& \multirow{5}{*}{\makecell{\rotatebox{90}{Wikitext2}}}   
& Original & \multicolumn{2}{c}{---- 2.517 ----} \\
& & LR1    & 2.550 & 2.589 \\
& & LR4    & 2.540 & 2.578 \\
& & BH1    & \textcolor{second_best}{2.538} & \textcolor{best}{2.573} \\
& & \cellcolor{gray!30}Rotor & \cellcolor{gray!30} \textcolor{best}{2.534} & \cellcolor{gray!30} \textcolor{second_best}{2.576}
\\
& \multirow{5}{*}{\makecell{\rotatebox{90}{C4}}}
& Original & \multicolumn{2}{c}{---- 2.862 ----} \\
& & LR1    & 2.881 & 2.907 \\
& & LR4    & 2.880 & \textcolor{second_best}{2.901} \\
& & BH1    & \textcolor{second_best}{2.878} & \textcolor{best}{2.898} \\
& & \cellcolor{gray!30}Rotor & \cellcolor{gray!30} \textcolor{best}{2.877} & \cellcolor{gray!30} \textcolor{second_best}{2.901}
\\
& \multirow{5}{*}{\makecell{\rotatebox{90}{PTB}}}
& Original & \multicolumn{2}{c}{---- 3.205 ----} \\
& & LR1    & 3.238 & 3.299 \\
& & LR4    & 3.230 & 3.276 \\
& & BH1    & \textcolor{second_best}{3.228} & \textcolor{second_best}{3.275} \\
& & \cellcolor{gray!30}Rotor & \cellcolor{gray!30} \textcolor{best}{3.226} & \cellcolor{gray!30} \textcolor{best}{3.259}
\\
\midrule
\multirow{10}{*}{\makecell{\rotatebox{90}{Accuracy (\%)}}}
& \multirow{5}{*}{\makecell{\rotatebox{90}{\makecell{Arc \\ Challenge}}}} 
& Original & \multicolumn{2}{c}{---- 24.89 ----} \\
& & LR1    & \textcolor{best}{29.03} & 26.40 \\
& & LR4    & \textcolor{second_best}{27.40} & 27.12 \\
& & BH1    & 26.77 & \textcolor{best}{31.42} \\
& & \cellcolor{gray!30}Rotor & \cellcolor{gray!30} 25.82 & \cellcolor{gray!30} \textcolor{second_best}{30.22}
\\
& \multirow{5}{*}{\makecell{\rotatebox{90}{Hellaswag}}} 
& Original & \multicolumn{2}{c}{---- 38.33 ----} \\
& & LR1    & 32.33 & 25.49 \\
& & LR4    & 32.41 & 24.40 \\
& & BH1    & \textcolor{second_best}{33.92} & \textcolor{second_best}{27.73} \\
& & \cellcolor{gray!30}Rotor & \cellcolor{gray!30} \textcolor{best}{35.14} & \cellcolor{gray!30} \textcolor{best}{32.87}
\\
\bottomrule
\end{tabular}
\vspace{10pt}
\caption{\footnotesize{Log-PPL $(\downarrow)$ and accuracy $(\uparrow)$ using original, Low-Rank ($r=1$ or $4$), BH1, and Rotor for $1$–$2$ layer replacements on \texttt{Fox-1.0 1.6B}. One-layer results are averaged over all layers; two-layer results are averaged over five random selections. Red indicates best, blue second-best per setting.}}
\label{tab:fox_avg}
\end{table}

\end{document}